\documentclass{article}

 \usepackage[preprint]{neurips_2026}

\usepackage[utf8]{inputenc} 
\usepackage[T1]{fontenc}    
\usepackage{hyperref}       
\usepackage{url}            
\usepackage{booktabs}       
\usepackage{amsfonts}       
\usepackage{nicefrac}       
\usepackage{microtype}      
\usepackage{xcolor}         

\usepackage{amsmath}
\usepackage{graphicx}
\usepackage{subcaption}
\usepackage{url}
\usepackage{amsthm}
\theoremstyle{definition}

\usepackage{enumitem}
\usepackage{algorithm,algpseudocode}
\theoremstyle{plain}

\newtheorem{theorem}{Theorem}
\newtheorem{proposition}{Proposition}
\newtheorem{corollary}{Corollary}
\usepackage{booktabs} 
\usepackage{float}
\theoremstyle{remark}

\usepackage{amssymb}
\usepackage{tikz}
\usepackage{wrapfig}
\usepackage{multirow}
\usepackage{xspace}
\usepackage{titletoc}
\title{Kurtosis-Guided Denoising Score Matching for Tabular Anomaly Detection}

\newcommand{\method}{K-DSM\xspace}
\author{%
  Victor Livernoche \quad Jie Zan \quad Reihaneh Rabbany \\
  McGill University, Mila \\
  Montréal, Canada
}

\begin{document}

\maketitle

\vspace{-1em}
\begin{abstract}
Denoising score matching (DSM) provides a way to learn data distributions by training a neural network to recover the score function, defined as the gradient of the log density, from noise-corrupted samples. Once trained, the score magnitude at a test point reflects how consistent that point is with the learned distribution, making it a natural anomaly signal. The key practical challenge is selecting the perturbation scale: too little noise yields unstable score estimates in sparse regions, while too much erases local structure and weakens anomaly sensitivity. This is compounded by the difficulty of hyperparameter tuning when anomalies are unknown and no validation set is available. We introduce kurtosis-based noise scaling (K-DSM)\footnote{We release the codebase here \space \url{https://github.com/vicliv/K-DSM}}, a per-feature scheme that sets noise levels from the shape of each marginal distribution, improving coverage of low-density regions and precision in high-density regions without extra model complexity. Contrary to prior claims that multi-scale or noise-conditioned training is necessary, we find that a carefully trained single-scale model is already a strong anomaly detector. On standard tabular anomaly detection benchmarks, K-DSM achieves state-of-the-art performance in the semi-supervised setting. When combined with a lightweight EMA-teacher filtering rule that removes low-density training points before each gradient step, it also achieves strong performance in the fully unsupervised (contaminated) setting, suggesting that simple, data-adaptive noise scaling enables robust anomaly detection while reducing reliance on hyperparameter tuning.
\end{abstract}

\section{Introduction}

Anomaly detection involves identifying patterns or observations that deviate from the expected or typical data distribution \citep{AD_def}. Effective detection of anomalies is critical in many fields, including fraud detection \citep{financial_example}, disease identification in healthcare \citep{medical_example}, and fault detection in manufacturing processes \citep{manufacture}. Traditional machine learning algorithms, such as statistical methods \citep{stats_method}, distance-based approaches \citep{kNN}, density estimation \citep{density_based}, and tree-based isolation methods \citep{isolation_forest}, have been applied to this problem. However, these methods exhibit significant limitations when applied to high-dimensional and structurally complex data distributions. 
These difficulties have prompted a shift toward more scalable and expressive deep learning approaches.

Among deep learning approaches, Denoising Score Matching (DSM) has emerged as a particularly promising technique due to its ability to model high-dimensional or otherwise intractable data distributions. Building on the theoretical connection between score matching and denoising autoencoders \citep{DSM}, DSM learns the gradient of the log-likelihood of the noise-perturbed distribution ($ \nabla_{\tilde{x}} \log q_\sigma(\tilde{x} \mid x)$) by perturbing data points with a predefined noise distribution. In the context of anomaly detection, the trained DSM model can be used to evaluate the score magnitude of input samples. A large score norm suggests that the sample resides in a region of low data density, thereby indicating a high likelihood of being anomalous. Conversely, a small score norm suggests the point is near the manifold and likely normal.

Despite these advantages, a key limitation of standard DSM lies in its reliance on a single, fixed noise scale $\sigma$. This parameter controls both the locality of the learned estimates and the effective coverage of the data space. If $\sigma$ is set too large, the resulting score norms become indistinguishable across samples, while an overly small $\sigma$ yields excessively local estimates that fail to capture global structure. 
To overcome this challenge, score-based anomaly detection methods have adopted multi-scale strategies. Noise Conditional Score Networks (NCSNs) \citep{song2020generativemodelingestimatinggradients}, originally proposed for generative modeling, condition the score network on multiple noise scales, thereby removing the need to commit to a single choice. Building on this idea, Multi-Scale Score Matching (MSM) \citep{mahmood2021multiscale} applies NCSN to anomaly detection, computing score norms across different perturbation strengths and subsequently separating anomalous samples using an unsupervised classifier.

Although multi-scale approaches alleviate the limitations of single-noise DSM, they introduce practical drawbacks: conditioning across multiple noise levels substantially increases computational and inference cost, and performance remains sensitive to the choice of the inference-time noise scale. These challenges highlight the need for a noise-scale selection strategy that is both label-independent and avoids reliance on multi-scale conditioning. In this paper, we propose \method, a kurtosis-guided denoising score matching method that addresses this challenge. Our approach adaptively assigns feature-specific noise levels based on the statistical shape of each feature distribution. We compute kurtosis, the fourth standardized moment, to quantify tailedness: light-tailed or flat distributions indicate broad coverage and thus require smaller perturbations, whereas heavy-tailed or sharply concentrated distributions require larger perturbations to ensure adequate coverage. This kurtosis-guided scheme offers a principled, data-driven alternative to heuristic tuning or exhaustive grid search, aligning perturbation scales with the intrinsic geometry of the feature space and enabling more effective and efficient training. Performance-wise, our method achieves state-of-the-art performance on ADBench \citep{han2022adbench} while maintaining significantly faster speed during inference time.

Our primary focus is the \emph{semi-supervised} setting, where training data consist exclusively of normal samples. DSM's assumption of clean training data is naturally violated in the fully unsupervised setting (contaminated training), but we show how a simple training mechanism recovers most of the lost performance in that regime (Section~\ref{sec:unsup_main}). In summary, our contributions are as follows:

\begin{itemize}[leftmargin=1.05em, topsep=0pt, itemsep=0pt]
\item \textbf{Single-scale DSM equipped with the right architecture outperform multi-scale conditioning.} With a modern MLP, the same backbone used by \method, single-scale DSM alone reaches \textbf{0.584} mean AUC-PR on ADBench, outperforming MSM (0.550), DTE (0.526), and matching DDAE and KNN, which challenges the claim that multi-scale conditioning is necessary.
\item \textbf{\method: Kurtosis-guided per-feature noise scaling with theoretical guarantees.}  We propose a diagonal perturbation rule that maps marginal kurtosis to $\sigma_j$ via the Cornish--Fisher expansion, preceded by a robust histogram-rearrangement step that isolates tail weight from shape artifacts (Section~\ref{sec:kurtosis}).
We prove the affine rule is the unique first-order optimal noise scale for tail coverage.
\item \textbf{Empirical state-of-the-art on ADBench (57 datasets)}. \method achieves \textbf{0.627 mean AUC-PR} in the semi-supervised setting, outperforming the runner-up baseline (\textbf{0.594}) and single-scale DSM (\textbf{0.584}). Per-feature kurtosis scaling also improves robustness to noise scaling hyperparameters. 
\item \textbf{EMA-teacher extension for unsupervised anomaly detection.} We repurpose the EMA-teacher construction as a per-sample density proxy, a lightweight filtering rule that drops low-density batch points, identified by the score norm of an exponential moving average of the online model. Applied on top of K-DSM, it lifts mean AUC-PR from 0.188 to \textbf{0.343}, the best across 14 methods.
\end{itemize}


\section{Preliminaries: Score Matching and Denoising Score Matching}
\label{sec:preliminaries}




\paragraph{The score function.}
For a continuous data distribution with density $p_{\text{data}}(x)$, the \emph{score function} is defined as the gradient of the log-density,
\begin{equation}
    s(x) \;=\; \nabla_x \log p_{\text{data}}(x).
    \label{eq:score_def}
\end{equation}
The score captures the local geometry of the distribution: it points toward regions of increasing density and its magnitude reflects the steepness of that increase. Crucially, $s(x)$ depends only on $\nabla_x p_{\text{data}}(x) / p_{\text{data}}(x)$ and thus does not require knowledge of the normalizing constant of $p_{\text{data}}$, making it attractive for modeling complex, high-dimensional distributions where the partition function is intractable.

\paragraph{Score matching.}
\citet{hyvarinen2005score} showed that a parametric model $s_\theta(x)$ can be fit to $s(x)$ by minimizing the expected squared error $\frac{1}{2}\mathbb{E}_{p_{\text{data}}}\!\bigl[\|s_\theta(x) - \nabla_x \log p_{\text{data}}(x)\|_2^2\bigr]$. Through integration by parts, this objective can be rewritten in a form that depends only on $s_\theta$ and $p_{\text{data}}$, eliminating the need to evaluate $\nabla_x \log p_{\text{data}}$ directly. However, the resulting \emph{implicit score matching} (ISM) objective involves the trace of the Jacobian $\mathrm{tr}(\nabla_x s_\theta(x))$, whose computation scales as $\mathcal{O}(d)$ forward passes for $d$-dimensional data, making it prohibitively expensive in high dimensions.

\paragraph{Denoising score matching.}
\citet{DSM} introduced denoising score matching (DSM) as a computationally efficient alternative. The key insight is that instead of matching the score of the unknown data distribution, one can match the score of a \emph{known} noise-perturbed distribution. Given clean data $x \sim p_{\text{data}}(x)$ and a corruption kernel $q_\sigma(\tilde{x} \mid x)$, the DSM objective trains $s_\theta$ to approximate the score of the conditional:
\begin{equation}
\mathcal{L}_{\text{DSM}}(\theta) \;:=\;
\frac{1}{2} \, \mathbb{E}_{x \sim p_{\text{data}}} \, \mathbb{E}_{\tilde{x} \sim q_\sigma(\tilde{x} \mid x)}
\left[ \left\| s_\theta(\tilde{x}) - \nabla_{\tilde{x}} \log q_\sigma(\tilde{x} \mid x) \right\|_2^2 \right].
\label{eq:dsm}
\end{equation}
\citet{DSM} proved that the minimizer of Eq.~\eqref{eq:dsm} also minimizes the ISM objective with respect to the \emph{perturbed} density $q_\sigma(\tilde{x}) = \int q_\sigma(\tilde{x} \mid x)\, p_{\text{data}}(x)\, dx$, establishing an equivalence between denoising and score estimation. Because $\nabla_{\tilde{x}} \log q_\sigma(\tilde{x} \mid x)$ is available in closed form for standard noise families, DSM requires no second-order derivatives and reduces training to a simple regression problem. For Gaussian corruption $q_\sigma(\tilde{x} \mid x) = \mathcal{N}(\tilde{x};\, x,\, \sigma^2 I)$, the target simplifies to $-(\tilde{x} - x)/\sigma^2$, and the model learns to predict the direction and magnitude of the displacement from the corrupted point back to the clean data.

\section{Kurtosis-Guided Denoising Score Matching}
\label{sec:method}

\subsection{The Noise Scale Problem}
\label{sec:noise_problem}

\begin{figure}[t]
    \centering
    \includegraphics[width=0.99\linewidth]{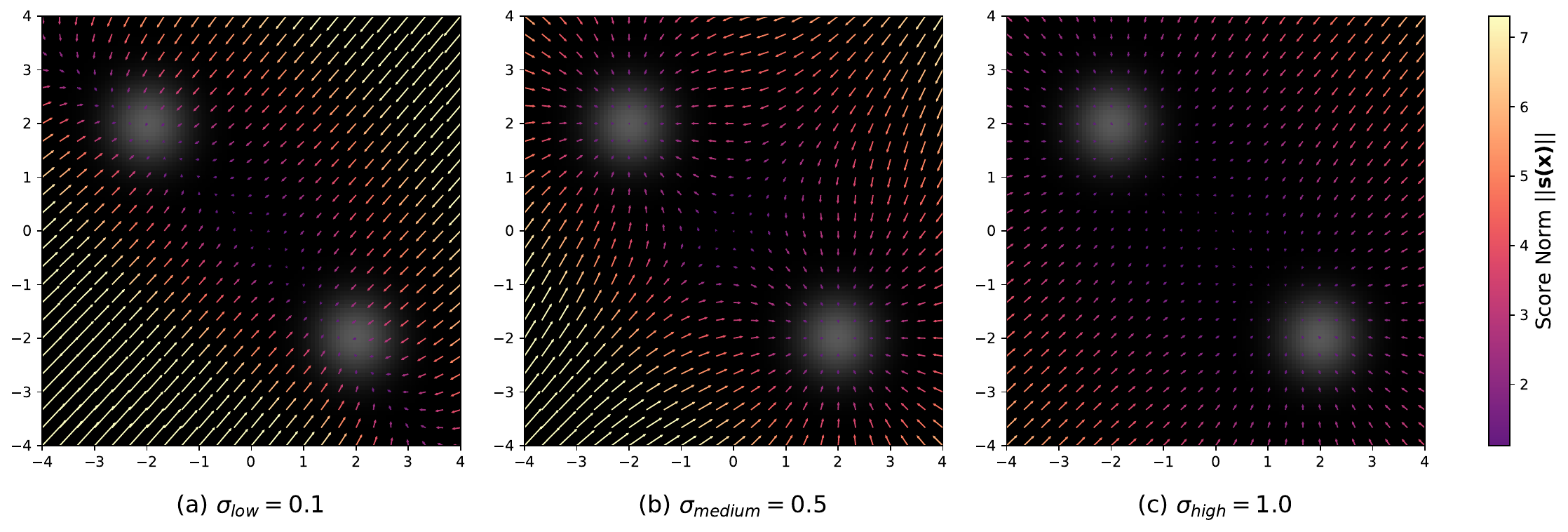}
    \caption{Score fields learned by DSM models with different Gaussian noise levels. Each panel shows estimated score vectors over a 2D grid, color-coded by norm $\|\mathbf{s}(\mathbf{x})\|$, for samples from a two-Gaussian mixture $\mathcal{N}((-2, 2), 0.3 \mathbf{I})$ and $\mathcal{N}((2, -2), 0.3 \mathbf{I})$. (a) fails to capture global structure, (b) yields a smooth field with clear norm contrast aligned to the data density, (c) preserves structure but with weakened norm contrast.}
    \label{fig:score field}
\end{figure}

The perturbation scale $\sigma$ in DSM controls a trade-off between fidelity to the data distribution and coverage of the surrounding space. Figure~\ref{fig:score field} illustrates this on a two-Gaussian mixture. At low noise, the perturbed distribution stays close to $p_{\text{data}}$ and the learned score field is informative only in high-density regions, offering no discriminative signal where anomalies actually fall. At high noise, the field retains directional structure but the perturbation smooths over the fine structure of $p_{\text{data}}$, collapsing score norm differences between normal and anomalous inputs. An intermediate scale achieves the best compromise: $\sigma$ is large enough to push probability mass into low-density regions during training, yet small enough to preserve the local structure that separates them from the normal support. Prior work addressed this sensitivity through multi-scale conditioning \citep{mahmood2021multiscale,song2020generativemodelingestimatinggradients}, training a single network across many noise levels simultaneously. However, multi-scale methods introduce substantial computational and memory overhead and require selecting a noise level at inference time, which itself becomes a tuning problem. We instead pursue a single-scale strategy with a principled, data-driven rule for choosing $\sigma$.

\subsection{Per-Feature Perturbation}
\label{sec:per_feature}
The standard DSM objective applies a single global noise scale to all
features. We argue that this is suboptimal for tabular data, where
features often have fundamentally different distributional
characteristics. We therefore propose a per-feature DSM objective
with a diagonal noise covariance:
\begin{equation}
\mathcal{L}_{\text{DSM}}^{\text{per-feature}}(\theta) 
= \frac{1}{2} \; 
\mathbb{E}_{x \sim p_{\text{data}}} \;
\mathbb{E}_{\tilde{x} \sim \mathcal{N} \left(x, \Sigma\right)} 
\left[ 
\left\| s_\theta(\tilde{x}) + \Sigma^{-1/2} (\tilde{x} - x) \right\|_2^2
\right]
\label{eq:perfeature_dsm}
\end{equation}
where $\Sigma = \mathrm{diag}(\sigma_1^2, \dots, \sigma_d^2)$. We use the residual $\Sigma^{-1/2}(\tilde{x} - x)$, which equals the unit-variance noise $\varepsilon \sim \mathcal{N}(0, I_d)$, making this the per-feature analogue of the $\varepsilon$-prediction objective used in DDPMs \citet{ho2020denoisingdiffusionprobabilisticmodels}. Without this rescaling, features with small $\sigma_j$ would produce high-magnitude score signals that dominate the loss. Predicting $\varepsilon$ decouples the gradient contribution from the noise scale, ensuring each feature has an equal impact on training regardless of its specific $\sigma_j$.

The use of a diagonal covariance is a choice regarding the perturbation kernel and \emph{this does not assume feature independence in the data}. This kernel determines the volume of the input space where the network receives a gradient signal. While each $\sigma_j$ independently sets the perturbation range for its respective coordinate, the network still observes joint samples and fits the full joint score $\nabla \log q_\sigma(\tilde{x})$. As shown in Figure~\ref{fig:perturbation}, this approach allows the model to extend coverage into low-density regions for heavy-tailed features without sacrificing the fidelity of light-tailed ones. This raises the question: which marginal statistic should drive the choice of each $\sigma_j$?

\begin{figure}[t]
  \centering
  \includegraphics[width=0.99\textwidth]{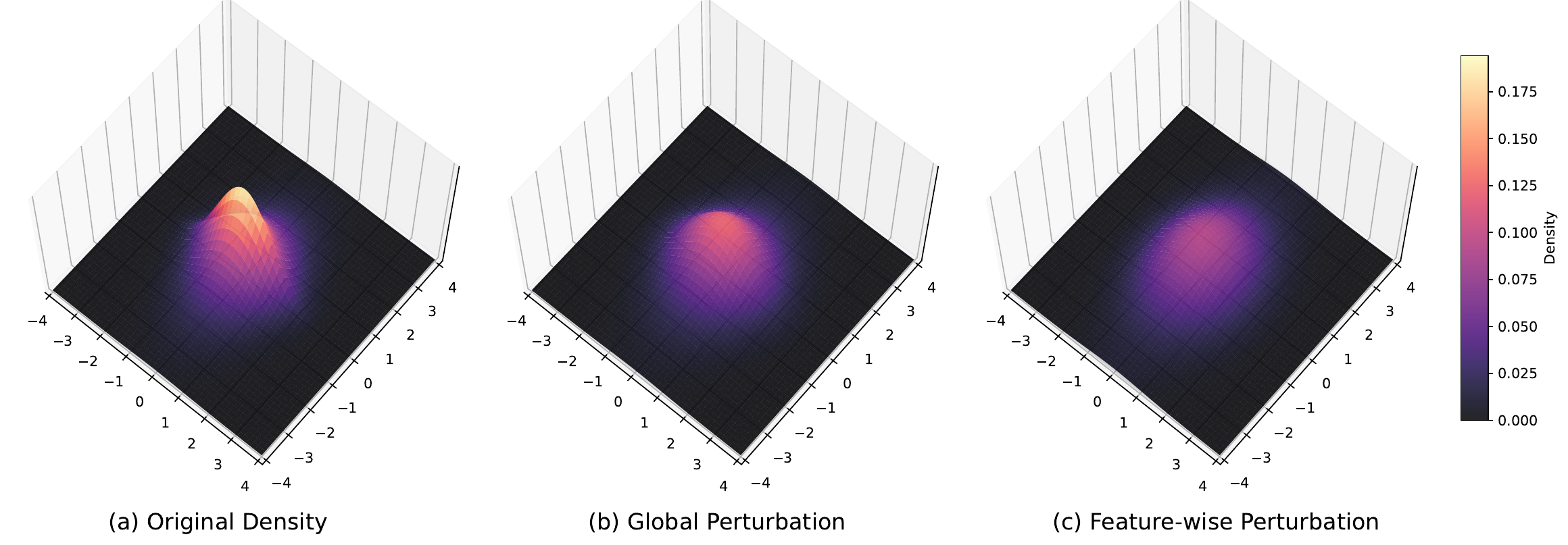}
  \caption{Original density (left) drawn from $x \sim \mathcal{N}(0,
  1)$ and $y \sim \mathrm{Laplace}\!\left(0,
  \tfrac{1}{\sqrt{2}}\right)$. Density under $\sigma_{global} =0.5$
  (center). Density under feature-wise perturbation guided by
  kurtosis (right), showing improved coverage of the surrounding
  space.}
  \label{fig:perturbation}
  \vspace{-1em}
\end{figure}

\subsection{Kurtosis as a Noise-Scale Proxy}
\label{sec:kurtosis}

The appropriate noise scale for a given feature depends on the shape of its marginal density. When the distribution is spread out (e.g., uniform-like or broad-variance Gaussian), the data already covers its effective support and only small perturbations are needed; excessive noise would over-smooth the score field. When the distribution is sharply peaked, the data concentrates in a narrow region, leaving the surrounding space underrepresented; larger perturbations are needed to force the model to learn gradients in these sparse areas.

To quantify this distinction, we use \emph{kurtosis}, the fourth standardised moment:
\begin{equation}
    \kappa = \frac{\mathbb{E}\left[(X - \mu)^4\right]}{s^4},
\end{equation}
where $\mu$ and $s^2 = \mathbb{E}[(X-\mu)^2]$ denote the mean and variance of the feature distribution, respectively (distinct from the noise scale $\sigma$). Despite a common misconception that kurtosis measures ``peakedness'' \citep{kurtosis_peak}, it more accurately reflects tail extremity: distributions with heavier tails produce larger kurtosis values. However, kurtosis is sensitive to distributional shape in ways that can be misleading: bimodal, skewed, or irregular distributions may produce kurtosis values that reflect structural features rather than tail behavior.

\paragraph{Histogram rearrangement.}
\label{sec:histo}

\begin{wrapfigure}{r}{0.5\textwidth}
\vspace{-0.6em}
    \centering
    \includegraphics[width=0.98\linewidth]{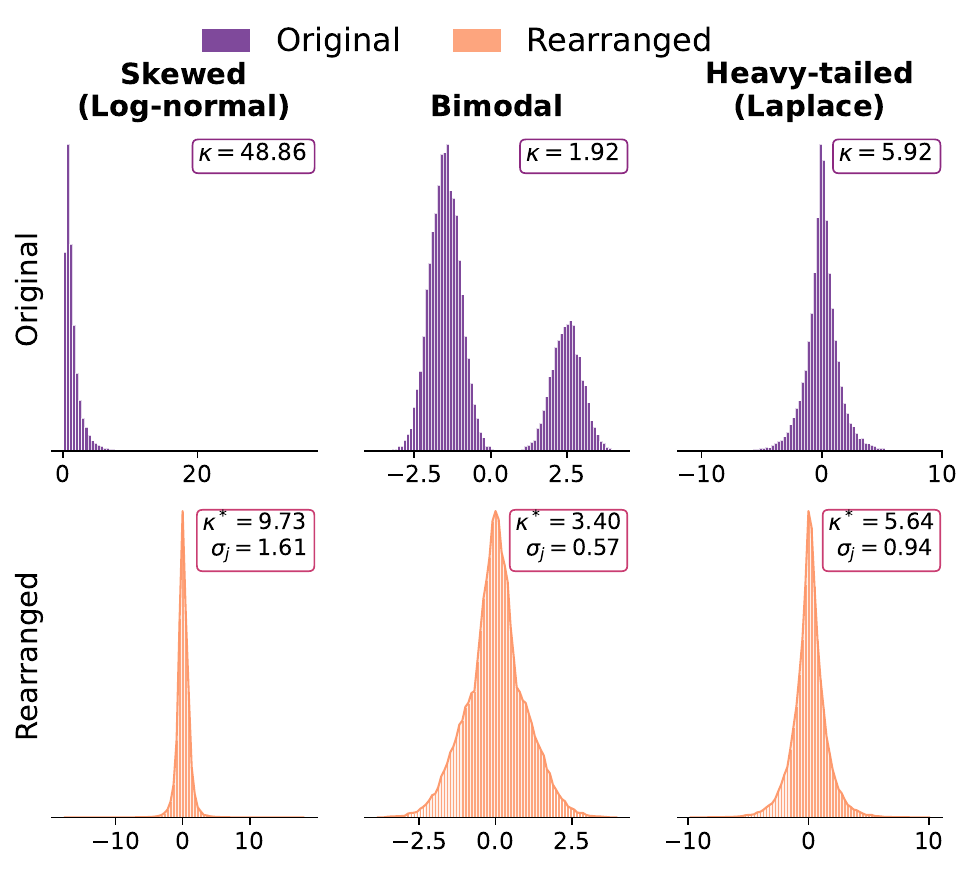}
    \caption{Histogram rearrangement for kurtosis-based noise selection. Top: original marginals. Bottom: symmetric decreasing rearrangements. Rearranged kurtosis $\kappa^*$ better reflects tail extent and determines the per-feature noise scale $\sigma_j$; it is not affected by skewness and is robust to multimodality.}
    \label{fig:rearrangement}
    \vspace{-2em}
\end{wrapfigure}

To ensure that kurtosis captures tail weight rather than shape artifacts, we apply a histogram rearrangement step before computing kurtosis.
The procedure symmetrizes the empirical distribution about its dominant mode by sorting histogram bins in decreasing order of frequency and placing them symmetrically around the origin. This is the discrete analogue of the symmetric decreasing rearrangement from measure theory \citep{LiebLoss2001}: it preserves the distribution of density values (equimeasurability) while producing a symmetric, unimodal profile whose kurtosis reflects only the rate of decay from peak to tails. For a symmetric unimodal input, the rearrangement is nearly a no-op. For skewed distributions, it captures tail extremity rather than asymmetry. For bimodal distributions, it treats the smaller mode as a tail. The full algorithm and convergence properties are given in Appendix~\ref{app:rearrangement}. Importantly, rearrangement is applied \emph{only} to compute $\sigma_j$; training, perturbation, and scoring all operate on the unrearranged features.

\paragraph{Noise-scale formula.}
We assign per-feature noise levels using the following formula:
\begin{equation}
\sigma_j = \sigma_{\mathrm{base}} \cdot \bigl(1 + c \cdot (\kappa_j - 3)\bigr), \qquad \text{clipped to } [\sigma_{\min},\, \sigma_{\max}],
\label{eq:sigma_formula}
\end{equation}
where $\sigma_{\mathrm{base}}$ is a scalar hyperparameter, $c > 0$ controls the sensitivity to kurtosis, and $\kappa_j$ is the kurtosis of the $j$-th rearranged feature marginal. The Gaussian distribution ($\kappa = 3$) serves as the natural reference point: at this value, $\sigma_j = \sigma_{\mathrm{base}}$, and deviations from mesokurtic behavior scale the noise proportionally. The functional form $1 + c(\kappa - 3)$ is directly motivated by the Cornish--Fisher tail-radius expansion (Section~\ref{sec:theory}), which shows that the tail extent of a symmetric unimodal distribution scales linearly in $(\kappa - 3)$ to first order. We set $c = 0.33$, $\sigma_{\mathrm{base}} = 0.5$, and $[\sigma_{\min}, \sigma_{\max}] = [0.1, 2.0]$.

\subsection{Theoretical Justification}
\label{sec:theory}

We first show that kurtosis captures tail extent for symmetric unimodal distributions, then that the affine rule in Eq.~\eqref{eq:sigma_formula} is the first-order optimal noise scale for tail coverage. Because heavier-tailed features spread probability mass further from the mode, the perturbation must be larger to expose the score network to these low-density regions; a statistic that orders tail extent therefore directly prescribes a noise-scale ordering across features. The histogram rearrangement maps every marginal to a symmetric unimodal proxy, so the guarantees below apply to arbitrary feature distributions.

\subsubsection{Kurtosis as a Tail-Extent Index}
\label{sec:kurtosis-tail}

After standardisation ($\mu=0$, $\sigma^2=1$), the kurtosis
$\kappa = \mathbb{E}[X^4]$ of a symmetric unimodal density is
determined entirely by its tail decay profile. Define the
\emph{$\tau$-tail radius} $R_\tau = Q_{1-\tau/2}(X)$, the quantile
above which a fraction $\tau/2$ of probability mass lies. We will use
$R_\tau$ as the geometric target the perturbation must reach in order
to expose the score network to the tail of the data marginal.

\begin{proposition}[Kurtosis orders tail extent in the tail regime]
\label{prop:kurtosis-exact}
Let $X$ be drawn from a unit-variance symmetric unimodal density
parameterised by a shape parameter controlling tailedness.
\begin{enumerate}[label=(\alph*),nosep]
\item \textbf{Generalised Gaussian.} For $p(x) \propto \exp(-|x/a|^\beta)$
with $a(\beta)$ chosen to enforce unit variance, there exists a
threshold $\tau^*_{\mathrm{GGD}} > 0$ such that, for every
$\tau \in (0, \tau^*_{\mathrm{GGD}})$, $R_\tau$ is strictly increasing
in $\kappa$ across the family.

\item \textbf{Student-$t$.} For the Student-$t$ with $\nu > 4$
degrees of freedom and unit-variance standardisation, there exists a
threshold $\tau^*_{t} > 0$ such that, for every
$\tau \in (0, \tau^*_{t})$, $R_\tau$ is strictly increasing in $\kappa$.
\end{enumerate}
Numerical values of $\tau^*$ and full proofs appear in
Appendix~\ref{app:kurtosis-proofs}; both thresholds exceed $0.02$,
well above the regime used in this paper.
\end{proposition}

The restriction to small $\tau$ is essential because unit-variance heavy-tailed densities are simultaneously more peaked near the mode (see Appendix~\ref{app:kurtosis-proofs}). The next result extends this monotonicity, to first order, to all symmetric unimodal densities via the Cornish--Fisher expansion.


\begin{proposition}[Cornish--Fisher tail-radius expansion]
\label{prop:cornish-fisher}
Let $p$ be a symmetric, standardised, unimodal density with
kurtosis~$\kappa$. Then the $\tau$-tail radius admits the expansion
\begin{equation}
  R_\tau
  \;\approx\;
  z_\tau
  \;+\;
  \frac{\kappa - 3}{24}\,\bigl(z_\tau^{3} - 3\,z_\tau\bigr)
  \;+\;
  O\!\bigl((\kappa-3)^{2}\bigr),
  \label{eq:cornishfisher}
\end{equation}
where $z_\tau = \Phi^{-1}(1-\tau/2)$ and skewness terms vanish by
symmetry.
\end{proposition}
 
\begin{proof}[Proof sketch]
Starting from the Edgeworth expansion of the CDF, which for a
symmetric distribution reduces to
$F(x) = \Phi(x) - \tfrac{\kappa-3}{24}\,\mathrm{He}_3(x)\,\phi(x)
+ O((\kappa-3)^2)$,
we invert at level $p = 1 - \tau/2$ by writing
$R_\tau = z_\tau + \delta$ and solving for $\delta$ to first order.
The full derivation is given in Appendix~\ref{app:cornish-fisher}.
\end{proof}

\subsubsection{From Tail Extent to Perturbation Scale: an Optimality Statement}
\label{sec:sigma-theory}

We now state the optimization problem that the kurtosis rule solves
exactly. The goal of the perturbation is to expose the score network
to the tail of the data marginal. We formalize this as a coverage
requirement: a training point $x$ near the mode should, with
probability at least $\alpha$, be displaced by noise far enough to
reach the $\tau$-tail region $\{|y| \geq R_\tau\}$ of the marginal.

\begin{theorem}[Minimum-noise coverage rule]
\label{thm:min-noise-coverage}
Let $X$ be a symmetric unimodal unit-variance marginal with kurtosis
$\kappa$ and $\tau$-tail radius $R_\tau(\kappa)$. Fix a tail level
$\tau$ and a coverage level $\alpha\in(0,1)$, and consider the
constrained optimization
\begin{equation}
\sigma^{*}(\kappa) \;=\;
\arg\min_{\sigma > 0}\; \sigma
\quad \text{s.t.} \quad
\Pr_{\varepsilon \sim \mathcal{N}(0,\sigma^{2})}\!\bigl[|\varepsilon| \geq R_\tau(\kappa)\bigr] \;\geq\; \alpha.
\label{eq:min-noise-problem}
\end{equation}
Its unique solution is
\begin{equation}
\sigma^{*}(\kappa) \;=\; \frac{R_\tau(\kappa)}{\Phi^{-1}(1 - \alpha/2)},
\label{eq:min-noise-sol}
\end{equation}
and, under the conditions of
Proposition~\ref{prop:kurtosis-exact}, $\sigma^{*}(\kappa)$ is
strictly increasing in $\kappa$.
\end{theorem}

\begin{proof}[Proof sketch]
For $\varepsilon\sim\mathcal{N}(0,\sigma^{2})$,
$\Pr(|\varepsilon|\ge R_\tau)=2(1-\Phi(R_\tau/\sigma))\ge\alpha$
iff $\sigma\ge R_\tau/\Phi^{-1}(1-\alpha/2)$. The constraint is
monotone in $\sigma$, so the minimum is attained with equality. Strict monotonicity in $\kappa$
follows from Proposition~\ref{prop:kurtosis-exact} since
$\Phi^{-1}(1-\alpha/2)$ is independent of $\kappa$. Full proof in 
Appendix~\ref{app:min-noise-proofs}.
\end{proof}

\begin{corollary}[First-order optimal affine rule]
\label{cor:sigma-monotone}
Under the assumptions of Theorem~\ref{thm:min-noise-coverage}, the
minimum-noise scale admits the first-order expansion around the
Gaussian reference $\kappa=3$
\begin{equation}
\frac{\sigma^{*}(\kappa)}{\sigma^{*}(3)}
\;=\;
1 \;+\; \frac{z_\tau^{2} - 3}{24}\,(\kappa - 3)
\;+\; O\!\bigl((\kappa - 3)^{2}\bigr),
\qquad z_\tau = \Phi^{-1}(1 - \tau/2),
\label{eq:optimal-affine}
\end{equation}
which is affine in $\kappa$ with slope
$c_{\mathrm{CF}}(\tau)=(z_\tau^{2}-3)/24$, strictly positive whenever
$\tau<2(1-\Phi(\sqrt{3}))\approx 0.083$. In particular, the affine rule
$\sigma(\kappa)=\sigma_{\text{base}}[1+c(\kappa-3)]$ of
Eq.~\eqref{eq:sigma_formula} is the unique first-order Taylor
approximation of $\sigma^{*}(\kappa)$ in the excess-kurtosis
direction.
\end{corollary}

\begin{proof}[Proof sketch]
Substitute the Cornish--Fisher expansion~\eqref{eq:cornishfisher}
into~\eqref{eq:min-noise-sol} and divide by $\sigma^{*}(3)$; the
$\Phi^{-1}(1-\alpha/2)$ factor cancels and the ratio simplifies to
$1+\tfrac{z_\tau^{2}-3}{24}(\kappa-3)+O((\kappa-3)^{2})$. Uniqueness
of the first-order Taylor coefficient gives the uniqueness claim.
Full proof in Appendix~\ref{app:min-noise-proofs}.
\end{proof}

\paragraph{Why kurtosis specifically.}
For a symmetric standardised marginal, all odd cumulants vanish and the first two (mean and variance) are removed by standardisation, 
so the excess kurtosis $\kappa-3$ is the leading non-Gaussian cumulant. Corollary~\ref{cor:sigma-monotone} is therefore the unique first-order correction to a Gaussian-reference noise scale. In the practical formula, $\sigma_{\text{base}}=\sigma^{*}(3)$ fixes the coverage level at the Gaussian reference and $c$ plays the role of $c_{\mathrm{CF}}(\tau)$; the empirical optimum $c=0.33$ corresponds to $\tau\approx10^{-3}$, larger than the nominal $c_{\mathrm{CF}}(0.05)\approx 0.035$ because the score network benefits from dense coverage of the tail region rather than minimal coverage at the boundary. Ablations (in Section~\ref{sec:results}) confirm that detection quality is robust to the precise value of $c$: the ablation over $\tau$ induces implied slopes $c_{\mathrm{CF}}(\tau)\in\{0.012,\,0.035,\,0.15,\,0.20,\,0.33,\,0.51,\,0.56,\,0.88\}$, all of which perform within a narrow band of AUC-PR, and we adopt the round value $c=0.33\approx 1/3$ as a simple default that sits at the empirical plateau and matches the Gaussian anchor $\kappa_j=3$. More details in Appendix~\ref{app:c-tau}. We investigated other sigma selecting strategies, finding that kurtosis performs best (See Appendix ~\ref{app:alternative_strategies}).


\section{Related Work}

Classical tabular anomaly detectors rely on distance, density, or
projection heuristics and remain standard baselines: PCA
\citep{PCA_AD}, OCSVM \citep{OCSVM}, MCD \citep{MCD}, LOF \citep{LOF},
LODA \citep{LODA}, $k$-NN \citep{kNN}, Isolation Forest
\citep{isolation_forest}, Feature Bagging \citep{FeatureBagging},
HBOS \citep{HBOS}, ECOD \citep{ECOD}, COPOD \citep{COPOD}, and CBLOF
\citep{CBLOF}. Early deep approaches such as Deep SVDD \citep{DSVDD}
and DAGMM \citep{DAGMM} learn compact representations or couple
autoencoders with Gaussian mixtures, and self-supervised frameworks
GOAD \citep{GOAD} and NeuTraL-AD \citep{Neutral_AD} induce
anomaly-discriminative features through auxiliary classification
tasks. Closer to our comparisons, contrastive ICL \citep{ICL}
leverages representation learning, SLAD \citep{SLAD} exploits partial
supervision, and DTE \citep{DTE} scores anomalies via the estimated
distribution of diffusion times. Reconstruction-based methods connect
to our approach through the DSM--DAE (Denoising Autoencoder) equivalence \citep{DSM}; we
include DDAE \citep{DAE_diffusion} as a direct representative.

\paragraph{Score based methods:}

\citet{mahmood2021multiscale} proposed a multi-scale approach (MSM) that leverages a Noise Conditioned Score Network (NCSN) \citep{song2020generativemodelingestimatinggradients} to estimate score norms across different noise levels. These estimates are then used to detect anomalies through unsupervised clustering. Their work also analyzed the inseparability of score norms under the single-noise case in the image setting. \citet{Shin_2023_ICCV} adopt a similar multi-scale framework but replaces score-norm classification with a fixed-noise restoration error criterion. A common limitation of multi-scale methods lies in the choice of noise level at inference time: since the model is noise-conditioned, a noise input must be specified to determine the evaluation level for each sample. MSM addresses this by clustering over all noise scales, but the performance may suffer from unreliable score norms at suboptimal noise levels. 


\section{Experiments}
\label{sec:experiments}

\begin{figure}[h]
\vspace{-1em}
    \centering
        \centering
        \begin{subfigure}[b]{0.45\textwidth}
          \includegraphics[width=\textwidth]{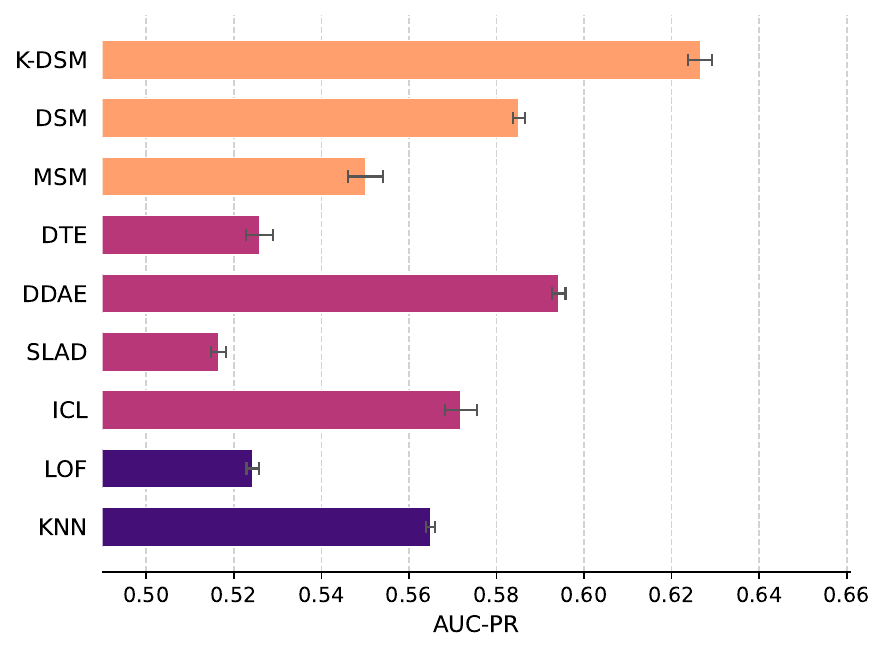}
          \caption{AUC-PR}
        \end{subfigure}
        \begin{subfigure}[b]{0.45\textwidth}
          \includegraphics[width=\textwidth]{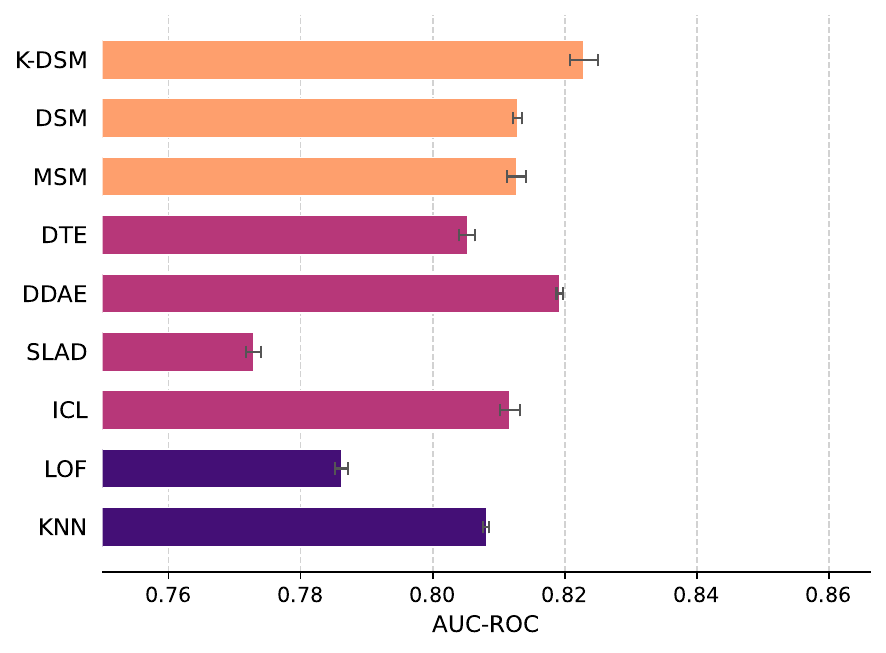}
          \caption{AUC-ROC}
        \end{subfigure}
        \caption{Average and standard deviation of AUC-PR and AUC-ROC on the 57 datasets from ADBench over five different seeds under semi-supervised setting. Colour scheme: orange (score-based), magenta (deep learning methods), purple (classical methods).}
        \label{fig:pr and aur roc}
    \vspace{-1em}
\end{figure}
\subsection{Protocol}
\label{sec:protocol}

We follow the  proceedure from \citep{DTE}. We evaluate on ADBench \citep{han2022adbench} (57 tabular datasets; see Table~\ref{table:datasets}), training on a random 50\% subset of normal samples and testing on the held-out normals plus all anomalies, capped at 50{,}000 samples. Results are mean with standard error over the 57 datasets (over 5 random seeds); scalers are fit on training normals only and all methods share identical splits. Our primary metric is AUC-PR, which better reflects retrieval quality under class imbalance than AUC-ROC \citep{pr_roc,mcdermott2024a}; we include AUC-ROC and F1 as secondary metrics. Every method uses a single fixed configuration across all datasets (paper-recommended values; details in Appendix~\ref{sec:model_architecture}). Our kurtosis-guided rule is explicitly label-free and requires no validation set.

\subsection{Results}
\label{sec:results}

We evaluate single-sigma DSM, multi-scale score matching (MSM), and our kurtosis-guided DSM (K-DSM) on ADBench alongside the top baselines identified by \citet{DTE} and the recent DDAE \citep{DAE_diffusion}; results are in Figure~\ref{fig:pr and aur roc}. Two findings stand out. First, plain single-scale DSM with an expressive MLP (Appendix~\ref{sec:model_architecture}) is already competitive with DDAE and KNN and clearly outperforms MSM, DTE, SLAD, and LOF, suggesting that much of the gap reported in prior work stems from under-specified architectures rather than from a fundamental limitation of single-scale training; to isolate the effect of the architecture, we trained DTE and MSM with the same Resnet-like MLP backbone used by K-DSM \citep{gorishniy2021revisiting}. Second, kurtosis scaling improves mean AUC-PR from 0.584 (single-scale DSM) to 0.627 (K-DSM), the best among all 14 methods; the improvements over both DSM and MSM are significant on all three metrics. We note that K-DSM and DSM also have much faster inference than all other methods (Figure~\ref{fig:inf_time}). We also evaluate on the image benchmarks MVTec-AD \citep{bergmann2019mvtec} and VisA \citep{zou2022visa} with DINOv3 \citep{siméoni2025dinov3} embeddings, where K-DSM is competitive but does not dominate because standardised vision features are approximately Gaussian per coordinate  (Appendix~\ref{sec:image_benchmarks_appendix}).




\begin{figure}[H]
    \centering
    \begin{subfigure}[b]{0.48\textwidth}
        \includegraphics[width=\textwidth]{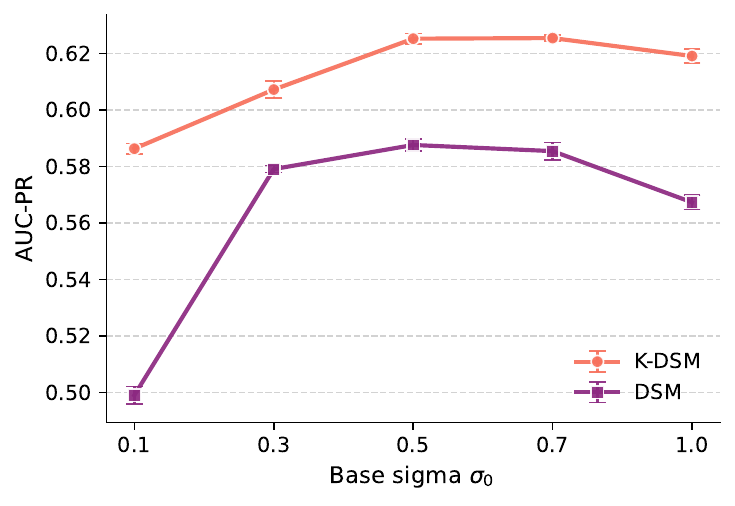}
        \caption{Sensitivity to base noise scale $\sigma_0$.}
        \label{fig:ablation_sigma}
    \end{subfigure}
    \begin{subfigure}[b]{0.48\textwidth}
        \includegraphics[width=\textwidth]{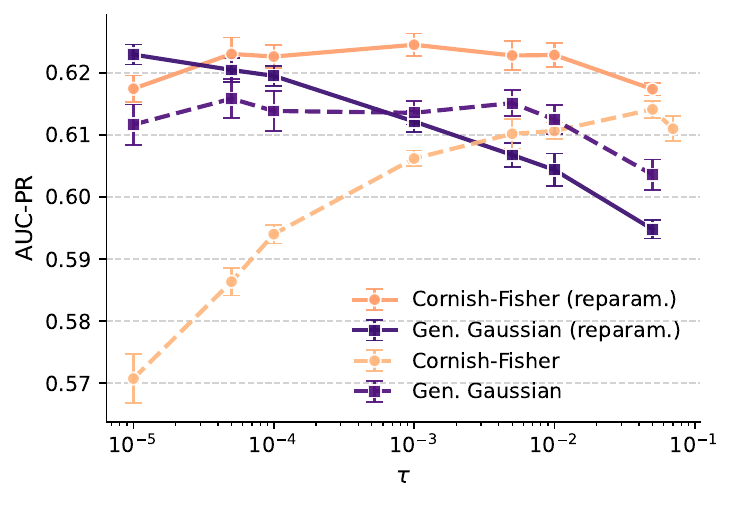}
        \caption{CF vs.\ GGD mapping across $\tau$, w/wo $\varepsilon$-prediction.}
        \label{fig:ablation_tau}
    \end{subfigure}
    \caption{Ablation studies on ADBench.
    \textbf{(a)}~K-DSM is substantially more robust than single-scale DSM to the choice of~$\sigma_0$.
    \textbf{(b)}~CF and GGD mappings under the standard DSM loss (dotted) and the $\varepsilon$-prediction loss (solid); $\varepsilon$-prediction helps both, and the affine CF rule dominates GGD for all $\tau$.}
    \label{fig:ablations}
    \vspace{-1em}
\end{figure}

\paragraph{Robustness to base noise scale.}
We compare K-DSM and single-scale DSM as the base noise scale $\sigma_0$ varies (Figure~\ref{fig:ablation_sigma}). Single-scale DSM is highly sensitive to this choice: its AUC-PR fluctuates substantially across the range, with performance degrading sharply at both extremes. Kurtosis-guided scaling, by contrast, remains stable across the full sweep, consistently outperforming its single-scale counterpart. This is the core practical advantage of our approach: the per-feature kurtosis adjustment absorbs much of the sensitivity that would otherwise fall on $\sigma_0$, making the method far more forgiving of the base hyperparameter. 

\paragraph{Cornish--Fisher vs.\ GGD noise mapping.}
The affine CF rule of Eq.~\eqref{eq:sigma_formula} is a first-order approximation of an exact tail-radius mapping; to check whether that approximation is adequate, we compare it against the exact mapping on the Generalised Gaussian (GGD) family. For each feature we solve Eq.~\eqref{eq:ggd-kurtosis} of Appendix~\ref{app:kurtosis-proofs} for the shape $\beta_j$ whose unit-variance kurtosis matches $\kappa_j$, and set $\sigma_j = \sigma_{\text{base}}\cdot R_\tau(\beta_j)/R_\tau(\beta{=}2)$ using the exact GGD tail radius. Both mappings are parameterised by the tail-probability threshold $\tau$, and we evaluate each under the standard DSM loss and the $\varepsilon$-prediction loss (Figure~\ref{fig:ablation_tau}). $\varepsilon$-prediction uniformly improves both mappings, confirming the decoupling argument of Section~\ref{sec:per_feature}, and the affine CF rule dominates GGD across four orders of magnitude in $\tau$, placing our default $c=0.33$ ($\tau\approx 10^{-3}$) inside a stable plateau.

\subsection{Unsupervised (Contaminated) Setting}
\label{sec:unsup_main}

In the fully unsupervised regime, the training set is unlabelled and contains a minority of anomalies, which violates the clean-data assumption of DSM: the denoising objective regresses the score toward zero at every training point, including contaminants, so anomalies are actively taught to have a small score norm. This hurts every score- or reconstruction-based detector in our benchmark.

\paragraph{EMA-teacher filtering.} We address this with a simple training rule that repurposes the EMA-teacher construction of semi- and self-supervised learning \citep{tarvainen2017mean,he2020moco,grill2020byol,caron2021dino}. Alongside the online network $s_\theta$, we maintain a shadow copy $s_{\theta_{\text{ema}}}$ updated as $\theta_{\text{ema}} \leftarrow 0.999\,\theta_{\text{ema}} + 0.001\,\theta$; each minibatch, samples whose teacher score norm $\|s_{\theta_{\text{ema}}}(x)\|$ exceeds the $\gamma$-th batch percentile are dropped and the gradient step is taken on the survivors. Lagging the teacher behind the student avoids the positive-feedback loop in which whatever the online model currently calls anomalous becomes more anomalous on the next step. Our contribution here is not the EMA mechanism itself but its repurposing as a per-sample density proxy for contamination filtering; the same recipe adapted to DDAE, DTE, and MSM is given in Appendix~\ref{app:ema_teacher}. We set $\gamma = 80$ based on an ablation study (Appendix~\ref{app:gamma_ablation}).

\paragraph{Results.} Table~\ref{tab:unsup_summary} reports mean AUC-PR, AUC-ROC, and F1 across the 57 datasets (full per-dataset tables in Appendix~\ref{app:unsup_full}). Under contamination, standard DSM and K-DSM collapse as predicted: they lose a large fraction of their semi-supervised performance and fall behind every classical baseline, because contaminants are explicitly taught to have small score norm. The EMA teacher neatly removes this failure mode, lifting every score- or reconstruction-based method. Kurtosis then compounds on top of the EMA teacher rather than being absorbed by it: K-DSM-EMA reaches state-of-the-art on the contaminated benchmark, claiming the best AUC-PR and F1 across all 14 methods, and K-DSM-EMA beats DSM-EMA on every metric, confirming that per-feature kurtosis scaling remains beneficial in the contaminated regime. We note that classical MCD is the strongest method on AUC-ROC and narrowly trails K-DSM-EMA on AUC-PR and F1, which indicates that a non-parametric baseline can remain highly competitive in this setting.

\begin{table}[h]
\caption{Mean performance ($\times 100$) across 57 ADBench datasets under the \emph{unsupervised} (contaminated) setting (5 seeds). Each score-based method shows EMA-teacher variant (top) and base variant (bottom, grey); see Appendix~\ref{app:ema_teacher}. Best per metric in bold; standard errors in parentheses.}
\centering
\label{tab:unsup_summary}
\setlength{\tabcolsep}{4pt}
\renewcommand{\arraystretch}{1.05}
\small
\resizebox{\textwidth}{!}{
\begin{tabular}{l l ccccc ccccc}
\toprule
Metric & Variant & K-DSM & DSM & DDAE & DTE & MSM & SLAD & ICL & MCD & LOF & KNN \\
\midrule
\multirow{2}{*}{AUC-PR}
 & +EMA  & \textbf{34.3} (0.4) & 32.1 (0.5) & 31.5 (0.4) & 29.0 (0.2) & 16.9 (0.1)
 & \multirow{2}{*}{24.5 (0.1)} & \multirow{2}{*}{24.0 (0.3)} & \multirow{2}{*}{32.1 (0.5)} & \multirow{2}{*}{17.5 (0.1)} & \multirow{2}{*}{23.9 (0.1)} \\
 & base  & {\color{gray}\scriptsize 18.8 (0.3)} & {\color{gray}\scriptsize 19.4 (0.2)}
 & {\color{gray}\scriptsize 22.8 (0.3)} & {\color{gray}\scriptsize 28.1 (0.2)}
 & {\color{gray}\scriptsize 17.7 (0.2)} & & & & & \\
\addlinespace[2pt]
\cmidrule(lr){1-12}
\multirow{2}{*}{AUC-ROC}
 & +EMA  & 70.1 (0.3) & 71.0 (0.3) & 72.5 (0.1) & 72.5 (0.3) & 64.1 (0.3)
 & \multirow{2}{*}{67.3 (0.1)} & \multirow{2}{*}{69.5 (0.1)} & \multirow{2}{*}{\textbf{73.0} (0.4)} & \multirow{2}{*}{60.8 (0.2)} & \multirow{2}{*}{66.5 (0.1)} \\
 & base  & {\color{gray}\scriptsize 60.2 (0.2)} & {\color{gray}\scriptsize 60.4 (0.1)}
 & {\color{gray}\scriptsize 68.5 (0.5)} & {\color{gray}\scriptsize 72.6 (0.2)}
 & {\color{gray}\scriptsize 62.8 (0.5)} & & & & & \\
\addlinespace[2pt]
\cmidrule(lr){1-12}
\multirow{2}{*}{F1}
 & +EMA  & \textbf{34.4} (0.5) & 30.8 (0.5) & 31.2 (0.4) & 28.4 (0.3) & 16.6 (0.1)
 & \multirow{2}{*}{23.8 (0.2)} & \multirow{2}{*}{24.5 (0.6)} & \multirow{2}{*}{30.8 (0.6)} & \multirow{2}{*}{18.4 (0.2)} & \multirow{2}{*}{25.0 (0.1)} \\
 & base 
{\color{gray}\scriptsize \textit{base}}
 & {\color{gray}\scriptsize 19.5 (0.3)} & {\color{gray}\scriptsize 20.0 (0.3)}
 & {\color{gray}\scriptsize 24.1 (0.3)} & {\color{gray}\scriptsize 27.9 (0.4)}
 & {\color{gray}\scriptsize 18.2 (0.2)} & & & & & \\
\bottomrule
\vspace{-2em}
\end{tabular}}
\end{table}

\section{Conclusion}

We have revisited denoising score matching for anomaly detection and shown that the long-standing challenge of noise scale selection can be addressed without resorting to computationally expensive multi-scale models. Our kurtosis-guided noise scaling rule adapts the per-feature noise level to the statistical shape of each marginal, improving coverage of low-density regions while preserving sensitivity in high-density areas. A lightweight EMA-teacher filtering step extends the same model to the fully unsupervised, contaminated regime by removing likely anomalies from each gradient update at negligible extra cost. Together, these two ingredients suggest that careful, data-adaptive noise scaling combined with a contamination-robust training recipe is a practical and scalable alternative to heavier multi-scale or ensemble approaches for real-world anomaly detection.

\paragraph{Limitations.}
\label{sec:limitations_section}
Two limitations remain. First, on image-style data where features have been standardised to near-Gaussian marginals, per-feature kurtosis is close to 3 almost everywhere, so the kurtosis rule collapses toward single-scale DSM and its specific advantage disappears (Appendix~\ref{sec:image_benchmarks_appendix}). Kurtosis-guided scaling delivers its largest gains precisely when marginals exhibit meaningful tail heterogeneity, and we view this as a clean statement of the regime of validity of our method rather than a failure mode. Second, while the EMA-teacher extension of Section~\ref{sec:unsup_main} restores K-DSM to the top of AUC-PR and F1 under contamination, the method is not fully solved: the filter uses a fixed $\gamma$-th-percentile cutoff and implicitly assumes that the contamination rate is below $(100-\gamma)\%$, and methods that do not rely on density estimation (KNN, Isolation Forest) still degrade more gracefully when the contamination level is very low. Making the filtering threshold adaptive to the estimated contamination rate is a natural direction for future work.

\section*{Acknowledgments}
We gratefully acknowledge the support of the CIFAR AI Chairs program. We also thank Mila and the Digital Research Alliance of Canada for providing the computational resources that made this work possible.

\bibliography{references}

@inproceedings{bergmann2019mvtec,
  author    = {Bergmann, Paul and Fauser, Michael and Sattlegger, David and Steger, Carsten},
  title     = {{MVTec AD}: A Comprehensive Real-World Dataset for Unsupervised Anomaly Detection},
  booktitle = {IEEE/CVF Conference on Computer Vision and Pattern Recognition (CVPR)},
  year      = {2019}
}

@inproceedings{zou2022visa,
  author    = {Zou, Yang and Jeong, Jongheon and Pemula, Latha and Zhang, Dongqing and Dabeer, Onkar},
  title     = {{SPot-the-Difference} Self-Supervised Pre-training for Anomaly Detection and Segmentation},
  booktitle = {European Conference on Computer Vision (ECCV)},
  year      = {2022}
}

@inproceedings{ho2020denoisingdiffusionprobabilisticmodels,
  author    = {Ho, Jonathan and Jain, Ajay and Abbeel, Pieter},
  title     = {Denoising Diffusion Probabilistic Models},
  booktitle = {Advances in Neural Information Processing Systems},
  year      = {2020}
}

@book{LiebLoss2001,
  author    = {Elliott H. Lieb and Michael Loss},
  title     = {Analysis},
  edition   = {2nd},
  series    = {Graduate Studies in Mathematics},
  volume    = {14},
  publisher = {American Mathematical Society},
  address   = {Providence, RI},
  year      = {2001},
  isbn      = {978-0-8218-2783-3},
}

@book{HardyLittlewoodPolya1952,
  author    = {Godfrey Harold Hardy and John Edensor Littlewood and George P\'{o}lya},
  title     = {Inequalities},
  edition   = {2nd},
  publisher = {Cambridge University Press},
  address   = {Cambridge},
  year      = {1952},
}

@book{Devroye1987,
  author    = {Luc Devroye},
  title     = {A Course in Density Estimation},
  series    = {Progress in Probability and Statistics},
  volume    = {14},
  publisher = {Birkh\"{a}user},
  address   = {Boston},
  year      = {1987},
  isbn      = {978-0-8176-3365-3},
}

@book{Hall1992,
  author    = {Peter Hall},
  title     = {The Bootstrap and {E}dgeworth Expansion},
  series    = {Springer Series in Statistics},
  publisher = {Springer},
  address   = {New York},
  year      = {1992},
  isbn      = {978-0-387-94575-1},
  doi       = {10.1007/978-1-4612-4384-7},
}

@inproceedings{
mahmood2021multiscale,
title={Multiscale Score Matching for Out-of-Distribution Detection},
author={Ahsan Mahmood and Junier Oliva and Martin Andreas Styner},
booktitle={International Conference on Learning Representations},
year={2021},
url={https://openreview.net/forum?id=xoHdgbQJohv}
}

@InProceedings{Shin_2023_ICCV,
    author    = {Shin, Woosang and Lee, Jonghyeon and Lee, Taehan and Lee, Sangmoon and Yun, Jong Pil},
    title     = {Anomaly Detection using Score-based Perturbation Resilience},
    booktitle = {Proceedings of the IEEE/CVF International Conference on Computer Vision (ICCV)},
    month     = {October},
    year      = {2023},
    pages     = {23372-23382}
}

@inproceedings{DAE_diffusion,
author = {Sattarov, Timur and Schreyer, Marco and Borth, Damian},
title = {Diffusion-Scheduled Denoising Autoencoders for Anomaly Detection in Tabular Data},
year = {2025},
isbn = {9798400714542},
publisher = {Association for Computing Machinery},
address = {New York, NY, USA},
url = {https://doi.org/10.1145/3711896.3736910},
doi = {10.1145/3711896.3736910},
booktitle = {Proceedings of the 31st ACM SIGKDD Conference on Knowledge Discovery and Data Mining V.2},
pages = {2478–2489},
numpages = {12},
location = {Toronto ON, Canada},
series = {KDD '25}
}

@inproceedings{han2022adbench,  
      title={ADBench: Anomaly Detection Benchmark},   
      author={Songqiao Han and Xiyang Hu and Hailiang Huang and Mingqi Jiang and Yue Zhao},  
      booktitle={Neural Information Processing Systems (NeurIPS)},
      year={2022},  
}

@article{kNN,
author = {Ramaswamy, Sridhar and Rastogi, Rajeev and Shim, Kyuseok},
title = {Efficient algorithms for mining outliers from large data sets},
year = {2000},
issue_date = {June 2000},
publisher = {Association for Computing Machinery},
address = {New York, NY, USA},
volume = {29},
number = {2},
issn = {0163-5808},
url = {https://doi.org/10.1145/335191.335437},
doi = {10.1145/335191.335437},
journal = {SIGMOD Rec.},
month = may,
pages = {427–438},
numpages = {12}
}

@inproceedings{ICL,
title={Anomaly Detection for Tabular Data with Internal Contrastive Learning},
author={Tom Shenkar and Lior Wolf},
booktitle={International Conference on Learning Representations},
year={2022},
url={https://openreview.net/forum?id=_hszZbt46bT}
}

@inproceedings{DTE,
title={On Diffusion Modeling for Anomaly Detection},
author={Victor Livernoche and Vineet Jain and Yashar Hezaveh and Siamak Ravanbakhsh},
booktitle={The Twelfth International Conference on Learning Representations},
year={2024},
url={https://openreview.net/forum?id=lR3rk7ysXz}
}

@inproceedings{SLAD,
  title={Fascinating supervisory signals and where to find them: Deep anomaly detection with scale learning},
  author={Xu, Hongzuo and Wang, Yijie and Wei, Juhui and Jian, Songlei and Li, Yizhou and Liu, Ning},
  booktitle={International Conference on Machine Learning},
  pages={38655--38673},
  year={2023},
  organization={PMLR}
}

@incollection{Spencer2025Endotext,
  author    = {Spencer, Carole A.},
  title     = {Assay of Thyroid Hormone and Related Substances},
  booktitle = {Endotext [Internet]},
  editor    = {Feingold, Kenneth R. and Ahmed, S. Faisal and Anawalt, Bradley and Boyce, Alison and others},
  address   = {South Dartmouth (MA)},
  publisher = {MDText.com, Inc.},
  year      = {2025},
  note      = {Last updated July 3, 2025},
  url       = {https://www.ncbi.nlm.nih.gov/books/NBK279113/},
  urldate   = {2025-09-21}
}

@inproceedings{
mcdermott2024a,
title={A Closer Look at {AUROC} and {AUPRC} under Class Imbalance},
author={Matthew B.A. McDermott and Haoran Zhang and Lasse Hyldig Hansen and Giovanni Angelotti and Jack Gallifant},
booktitle={The Thirty-eighth Annual Conference on Neural Information Processing Systems},
year={2024},
url={https://openreview.net/forum?id=S3HvA808gk}
}

@incollection{Dunlap1990ClinicalMethods,
  author    = {Dunlap, Dickson B.},
  title     = {Thyroid Function Tests},
  booktitle = {Clinical Methods: The History, Physical, and Laboratory Examinations},
  editor    = {Walker, H. Kenneth and Hall, William D. and Hurst, J. Willis},
  edition   = {3rd},
  address   = {Boston},
  publisher = {Butterworths},
  year      = {1990},
  url       = {https://www.ncbi.nlm.nih.gov/books/NBK249/},
  urldate   = {2025-09-21}
}

@article{arXiv:2210.06959,
author = {Li, Zhong and Zhu, Yuxuan and Van Leeuwen, Matthijs},
title = {A Survey on Explainable Anomaly Detection},
year = {2023},
issue_date = {January 2024},
publisher = {Association for Computing Machinery},
address = {New York, NY, USA},
volume = {18},
number = {1},
issn = {1556-4681},
url = {https://doi.org/10.1145/3609333},
doi = {10.1145/3609333},
journal = {ACM Trans. Knowl. Discov. Data},
month = sep,
articleno = {23},
numpages = {54}
}

@article{kurtosis_peak,
author = {Peter H. Westfall},
title = {Kurtosis as Peakedness, 1905–2014. R.I.P.},
journal = {The American Statistician},
volume = {68},
number = {3},
pages = {191--195},
year = {2014},
publisher = {ASA Website},
doi = {10.1080/00031305.2014.917055},

    note ={PMID: 25678714},


URL = { 
    
        https://doi.org/10.1080/00031305.2014.917055
    
    

},
eprint = { 
    
        https://doi.org/10.1080/00031305.2014.917055
    
    

}

}

@inproceedings{pr_roc,
author = {Davis, Jesse and Goadrich, Mark},
title = {The relationship between Precision-Recall and ROC curves},
year = {2006},
isbn = {1595933832},
publisher = {Association for Computing Machinery},
address = {New York, NY, USA},
url = {https://doi.org/10.1145/1143844.1143874},
doi = {10.1145/1143844.1143874},
booktitle = {Proceedings of the 23rd International Conference on Machine Learning},
pages = {233–240},
numpages = {8},
location = {Pittsburgh, Pennsylvania, USA},
series = {ICML '06}
}

@article{hyvarinen2005score,
  title={Estimation of non-normalized statistical models by score matching},
  author={Hyv{\"a}rinen, Aapo},
  journal={Journal of Machine Learning Research},
  volume={6},
  number={Apr},
  pages={695--709},
  year={2005}
}

@misc{song2020generativemodelingestimatinggradients,
      title={Generative Modeling by Estimating Gradients of the Data Distribution}, 
      author={Yang Song and Stefano Ermon},
      year={2020},
      eprint={1907.05600},
      archivePrefix={arXiv},
      primaryClass={cs.LG},
      url={https://arxiv.org/abs/1907.05600}, 
}

@ARTICLE{DSM,
  author={Vincent, Pascal},
  journal={Neural Computation}, 
  title={A Connection Between Score Matching and Denoising Autoencoders}, 
  year={2011},
  volume={23},
  number={7},
  pages={1661-1674},
  doi={10.1162/NECO_a_00142}}

@article{AD_def,
  author       = {Guansong Pang and
                  Chunhua Shen and
                  Longbing Cao and
                  Anton van den Hengel},
  title        = {Deep Learning for Anomaly Detection: {A} Review},
  journal      = {CoRR},
  volume       = {abs/2007.02500},
  year         = {2020},
  url          = {https://arxiv.org/abs/2007.02500},
  eprinttype    = {arXiv},
  eprint       = {2007.02500},
  bibsource    = {dblp computer science bibliography, https://dblp.org}
}

@article{financial_example,
  author       = {Meng{-}Chieh Lee and
                  Yue Zhao and
                  Aluna Wang and
                  Pierre Jinghong Liang and
                  Leman Akoglu and
                  Vincent S. Tseng and
                  Christos Faloutsos},
  title        = {AutoAudit: Mining Accounting and Time-Evolving Graphs},
  journal      = {CoRR},
  volume       = {abs/2011.00447},
  year         = {2020},
  url          = {https://arxiv.org/abs/2011.00447},
  eprinttype    = {arXiv},
  eprint       = {2011.00447},
  bibsource    = {dblp computer science bibliography, https://dblp.org}
}

@article{medical_example,
  author       = {Wenyuan Li and
                  Yunlong Wang and
                  Yong Cai and
                  Corey W. Arnold and
                  Emily Zhao and
                  Yilian Yuan},
  title        = {Semi-supervised Rare Disease Detection Using Generative Adversarial
                  Network},
  journal      = {CoRR},
  volume       = {abs/1812.00547},
  year         = {2018},
  url          = {http://arxiv.org/abs/1812.00547},
  eprinttype    = {arXiv},
  eprint       = {1812.00547},
  bibsource    = {dblp computer science bibliography, https://dblp.org}
}

@article{manufacture,
  author       = {Juan C. Quiroz and
                  Norman Bin Mariun and
                  Mohammad Rezazadeh Mehrjou and
                  Mahdi Izadi and
                  Norhisam Misron and
                  Mohd Amran Mohd Radzi},
  title        = {Fault Detection of Broken Rotor Bar in {LS-PMSM} Using Random Forests},
  journal      = {CoRR},
  volume       = {abs/1711.02510},
  year         = {2017},
  url          = {http://arxiv.org/abs/1711.02510},
  eprinttype    = {arXiv},
  eprint       = {1711.02510},
  bibsource    = {dblp computer science bibliography, https://dblp.org}
}

@INPROCEEDINGS{density_based,
  author={Papadimitriou, S. and Kitagawa, H. and Gibbons, P.B. and Faloutsos, C.},
  booktitle={Proceedings 19th International Conference on Data Engineering (Cat. No.03CH37405)}, 
  title={LOCI: fast outlier detection using the local correlation integral}, 
  year={2003},
  volume={},
  number={},
  pages={315-326},
  doi={10.1109/ICDE.2003.1260802}}

@article{stats_method,
   title={Minimum covariance determinant and extensions},
   volume={10},
   ISSN={1939-0068},
   url={http://dx.doi.org/10.1002/wics.1421},
   DOI={10.1002/wics.1421},
   number={3},
   journal={WIREs Computational Statistics},
   publisher={Wiley},
   author={Hubert, Mia and Debruyne, Michiel and Rousseeuw, Peter J.},
   year={2017},
   month=dec }

@inproceedings{
DAGMM,
title={Deep Autoencoding Gaussian Mixture Model for Unsupervised Anomaly Detection},
author={Bo Zong and Qi Song and Martin Renqiang Min and Wei Cheng and Cristian Lumezanu and Daeki Cho and Haifeng Chen},
booktitle={International Conference on Learning Representations},
year={2018},
url={https://openreview.net/forum?id=BJJLHbb0-},
}

@article{ECOD,
  author       = {Zheng Li and
                  Yue Zhao and
                  Xiyang Hu and
                  Nicola Botta and
                  Cezar Ionescu and
                  George H. Chen},
  title        = {{ECOD:} Unsupervised Outlier Detection Using Empirical Cumulative
                  Distribution Functions},
  journal      = {CoRR},
  volume       = {abs/2201.00382},
  year         = {2022},
  url          = {https://arxiv.org/abs/2201.00382},
  eprinttype    = {arXiv},
  eprint       = {2201.00382},
  bibsource    = {dblp computer science bibliography, https://dblp.org}
}

@article{PCA_AD,
  title={A novel anomaly detection scheme based on principal component classifier},
  author={Shyu, Mei-Ling and Chen, Shu-Ching and Sarinnapakorn, Kanoksri and Chang, LiWu},
  journal={Proceedings of the IEEE Foundations and New Directions of Data Mining Workshop},
  pages={172--179},
  year={2003}
}

@article{OCSVM,
  title={Estimating the support of a high-dimensional distribution},
  author={Sch{\"o}lkopf, Bernhard and Platt, John C and Shawe-Taylor, John and Smola, Alex J and Williamson, Robert C},
  journal={Neural Computation},
  volume={13},
  number={7},
  pages={1443--1471},
  year={2001}
}

@article{MCD,
  title={A fast algorithm for the minimum covariance determinant estimator},
  author={Rousseeuw, Peter J and Driessen, Katrien Van},
  journal={Technometrics},
  volume={41},
  number={3},
  pages={212--223},
  year={1999}
}

@inproceedings{LOF,
  title={LOF: identifying density-based local outliers},
  author={Breunig, Markus M and Kriegel, Hans-Peter and Ng, Raymond T and Sander, J{\"o}rg},
  booktitle={Proceedings of the 2000 ACM SIGMOD International Conference on Management of Data},
  pages={93--104},
  year={2000}
}

@article{LODA,
  title={Loda: Lightweight on-line detector of anomalies},
  author={Pevn{\`y}, Tom{\'a}{\v{s}}},
  journal={Machine Learning},
  volume={102},
  number={2},
  pages={275--304},
  year={2016}
}

@inproceedings{FeatureBagging,
  title={Feature bagging for outlier detection},
  author={Lazarevic, Aleksandar and Kumar, Vipin},
  booktitle={Proceedings of the 11th ACM SIGKDD International Conference on Knowledge Discovery in Data Mining},
  pages={157--166},
  year={2005}
}

@inproceedings{HBOS,
  title={Histogram-based outlier score (HBOS): A fast unsupervised anomaly detection algorithm},
  author={Goldstein, Markus and Dengel, Andreas},
  booktitle={KI-2012: Poster and Demo Track},
  pages={59--63},
  year={2012}
}

@article{COPOD,
  title={COPOD: Copula-based outlier detection},
  author={Li, Zheng and Zhao, Yue and Botta, Nicola and Ionescu, Cezar and Hu, Xiyang},
  journal={2020 IEEE International Conference on Data Mining (ICDM)},
  pages={1118--1123},
  year={2020}
}

@article{CBLOF,
  title={Discovering cluster-based local outliers},
  author={He, Zengyou and Xu, Xiaofei and Deng, Shengchun},
  journal={Pattern Recognition Letters},
  volume={24},
  number={9-10},
  pages={1641--1650},
  year={2003}
}

@INPROCEEDINGS{isolation_forest,
  author={Liu, Fei Tony and Ting, Kai Ming and Zhou, Zhi-Hua},
  booktitle={2008 Eighth IEEE International Conference on Data Mining}, 
  title={Isolation Forest}, 
  year={2008},
  volume={},
  number={},
  pages={413-422},
  doi={10.1109/ICDM.2008.17}}

@InProceedings{DSVDD,
  title = 	 {Deep One-Class Classification},
  author =       {Ruff, Lukas and Vandermeulen, Robert and Goernitz, Nico and Deecke, Lucas and Siddiqui, Shoaib Ahmed and Binder, Alexander and M{\"u}ller, Emmanuel and Kloft, Marius},
  booktitle = 	 {Proceedings of the 35th International Conference on Machine Learning},
  pages = 	 {4393--4402},
  year = 	 {2018},
  editor = 	 {Dy, Jennifer and Krause, Andreas},
  volume = 	 {80},
  series = 	 {Proceedings of Machine Learning Research},
  month = 	 {10--15 Jul},
  publisher =    {PMLR},
  url = 	 {https://proceedings.mlr.press/v80/ruff18a.html}
}

@book{JohnsonKotzBalakrishnan1995,
  author    = {Norman L. Johnson and Samuel Kotz and N. Balakrishnan},
  title     = {Continuous Univariate Distributions},
  volume    = {2},
  edition   = {2nd},
  publisher = {Wiley},
  address   = {New York},
  year      = {1995},
}

@inproceedings{
GOAD,
title={Classification-Based Anomaly Detection for General Data},
author={Liron Bergman and Yedid Hoshen},
booktitle={International Conference on Learning Representations},
year={2020},
url={https://openreview.net/forum?id=H1lK_lBtvS}
}

@inproceedings{Neutral_AD,
  title={Neural transformation learning for deep anomaly detection beyond images},
  author={Qiu, Chen and Pfrommer, Timo and Kloft, Marius and Mandt, Stephan and Rudolph, Maja},
  booktitle={International Conference on Machine Learning},
  pages={8703--8714},
  year={2021},
  organization={PMLR}
}

@inproceedings{tarvainen2017mean,
  author    = {Tarvainen, Antti and Valpola, Harri},
  title     = {Mean Teachers are Better Role Models: Weight-Averaged Consistency Targets Improve Semi-Supervised Deep Learning Results},
  booktitle = {Advances in Neural Information Processing Systems (NeurIPS)},
  year      = {2017}
}

@inproceedings{laine2017temporal,
  author    = {Laine, Samuli and Aila, Timo},
  title     = {Temporal Ensembling for Semi-Supervised Learning},
  booktitle = {International Conference on Learning Representations (ICLR)},
  year      = {2017}
}

@inproceedings{grill2020byol,
  author    = {Grill, Jean-Bastien and Strub, Florian and Altch{\'e}, Florent and Tallec, Corentin and Richemond, Pierre H. and Buchatskaya, Elena and Doersch, Carl and Pires, Bernardo Avila and Guo, Zhaohan Daniel and Azar, Mohammad Gheshlaghi and Piot, Bilal and Kavukcuoglu, Koray and Munos, R{\'e}mi and Valko, Michal},
  title     = {Bootstrap Your Own Latent: A New Approach to Self-Supervised Learning},
  booktitle = {Advances in Neural Information Processing Systems (NeurIPS)},
  year      = {2020}
}

@inproceedings{he2020moco,
  author    = {He, Kaiming and Fan, Haoqi and Wu, Yuxin and Xie, Saining and Girshick, Ross},
  title     = {Momentum Contrast for Unsupervised Visual Representation Learning},
  booktitle = {IEEE/CVF Conference on Computer Vision and Pattern Recognition (CVPR)},
  year      = {2020}
}

@inproceedings{caron2021dino,
  author    = {Caron, Mathilde and Touvron, Hugo and Misra, Ishan and J{\'e}gou, Herv{\'e} and Mairal, Julien and Bojanowski, Piotr and Joulin, Armand},
  title     = {Emerging Properties in Self-Supervised Vision Transformers},
  booktitle = {IEEE/CVF International Conference on Computer Vision (ICCV)},
  year      = {2021}
}

@article{gorishniy2021revisiting,
  title={Revisiting deep learning models for tabular data},
  author={Gorishniy, Yury and Rubachev, Ivan and Khrulkov, Valentin and Babenko, Artem},
  journal={Advances in neural information processing systems},
  volume={34},
  pages={18932--18943},
  year={2021}
}

@article{polyak1992acceleration,
  author    = {Polyak, Boris T. and Juditsky, Anatoli B.},
  title     = {Acceleration of Stochastic Approximation by Averaging},
  journal   = {SIAM Journal on Control and Optimization},
  volume    = {30},
  number    = {4},
  pages     = {838--855},
  year      = {1992}
}

@misc{siméoni2025dinov3,
      title={DINOv3}, 
      author={Oriane Siméoni and Huy V. Vo and Maximilian Seitzer and Federico Baldassarre and Maxime Oquab and Cijo Jose and Vasil Khalidov and Marc Szafraniec and Seungeun Yi and Michaël Ramamonjisoa and Francisco Massa and Daniel Haziza and Luca Wehrstedt and Jianyuan Wang and Timothée Darcet and Théo Moutakanni and Leonel Sentana and Claire Roberts and Andrea Vedaldi and Jamie Tolan and John Brandt and Camille Couprie and Julien Mairal and Hervé Jégou and Patrick Labatut and Piotr Bojanowski},
      year={2025},
      eprint={2508.10104},
      archivePrefix={arXiv},
      primaryClass={cs.CV},
      url={https://arxiv.org/abs/2508.10104}, 
}
\bibliographystyle{plainnat}

\newpage
\appendix
\clearpage

\appendix
\section*{Appendix}
\startcontents[appendix]
\printcontents[appendix]{}{1}{}

\clearpage

\section{Broader Impact}
\label{app:broader_impact}

Anomaly detection has broad beneficial applications, including fraud and intrusion detection, medical diagnosis support, industrial quality control, and environmental monitoring. The methods developed in this paper are general-purpose and not tied to any specific deployment, but we briefly discuss potential concerns.

\paragraph{Misuse of anomaly detection systems.} Any anomaly detection system, including K-DSM, can be misused when deployed in contexts involving people rather than inanimate measurements. A system trained to flag unusual behaviour in a population may encode and amplify historical biases present in the training data, causing certain demographic groups to be systematically over-flagged as anomalous. This risk is particularly acute in high-stakes settings such as credit scoring, hiring, law enforcement, or health insurance. A separate misuse vector is deliberate deployment for surveillance: an unsupervised anomaly detector can be repurposed to identify individuals whose behaviour deviates from a prescribed norm, enabling authoritarian control or the targeting of whistleblowers, journalists, and political dissidents. We urge practitioners to subject any application of anomaly detection to individuals to rigorous fairness audits, to maintain meaningful human oversight of automated decisions, and to consider whether the deployment context is consistent with respect for civil liberties.

\paragraph{Positive impact.} Our work reduces the computational cost of deploying score-based anomaly detection relative to multi-scale alternatives, which lowers the barrier to adoption in resource-constrained settings such as embedded industrial sensors or clinical decision support in under-resourced healthcare systems.

\section{LLM Use}
\label{app:llm_use}

Large language models were used in three supporting roles during this research. First, as a reasoning aid for mathematical proofs: LLM outputs were used to explore proof strategies and identify relevant lemmas, but every step was independently verified by the authors, who take full responsibility for the correctness of all theoretical results. Second, as a writing assistant for drafting and editing manuscript text. Third, as a coding assistant for implementing baseline methods, evaluation pipelines, and plotting utilities. In all cases, LLM-generated content was reviewed and either accepted, modified, or discarded by the authors.

\section{Interpretability}
\label{sec:interpretability}

A practical advantage of score matching for anomaly detection is that it yields an \emph{interpretable} vector field. At the training optimum, the learned score network approximates
\[
s_\theta(x) \;\approx\; \nabla_x \log p_\theta(x),
\]
i.e., the local direction of steepest increase in modeled normality. In practice $s_\theta$ is a finite-capacity approximation, so this equality holds only approximately; the interpretations below should be understood as reflecting the model's learned view of normality rather than the true data score. With anomaly score $A(x)=\lVert s_\theta(x)\rVert_2$, each component $s_{\theta,i}(x)$ quantifies both responsibility and remedy: its magnitude ranks the influence of feature $i$ on the alarm, and its sign prescribes how to change that feature to reduce $A(x)$ ($s_{\theta,i}(x)>0 \Rightarrow$ increase $i$; $s_{\theta,i}(x)<0 \Rightarrow$ decrease $i$). This gives an immediate, actionable counterfactual $x' = x + \eta s_\theta(x)$ that moves toward higher density. Explainability is critical in anomaly detection because anomalies are rare, ambiguous, and often tied to high-stakes situations; practitioners must audit alerts, decide whether to intervene, and know which sensor or attribute to inspect, and regulators increasingly require a clear justification. As emphasized by \citet{arXiv:2210.06959}, many explanation tools added on top of existing detectors tend to be inaccurate, fragile, and computationally expensive. DSM addresses these concerns because the detector and the explanation are the same: the gradient that defines $A(x)$ also explains it. This alignment ensures that what the model uses to decide is exactly what we show as explanation, removes the risk of mismatch with external tools, adds almost no extra computation, and yields actionable, signed per-feature directions. We therefore adopt $s_\theta(x)$ as a reliable built-in explanation for the model's decisions, while noting that it explains the model’s own view of normality rather than an absolute ground truth. 

\begin{figure}[h]
  \centering
  \includegraphics[width=0.85\linewidth]{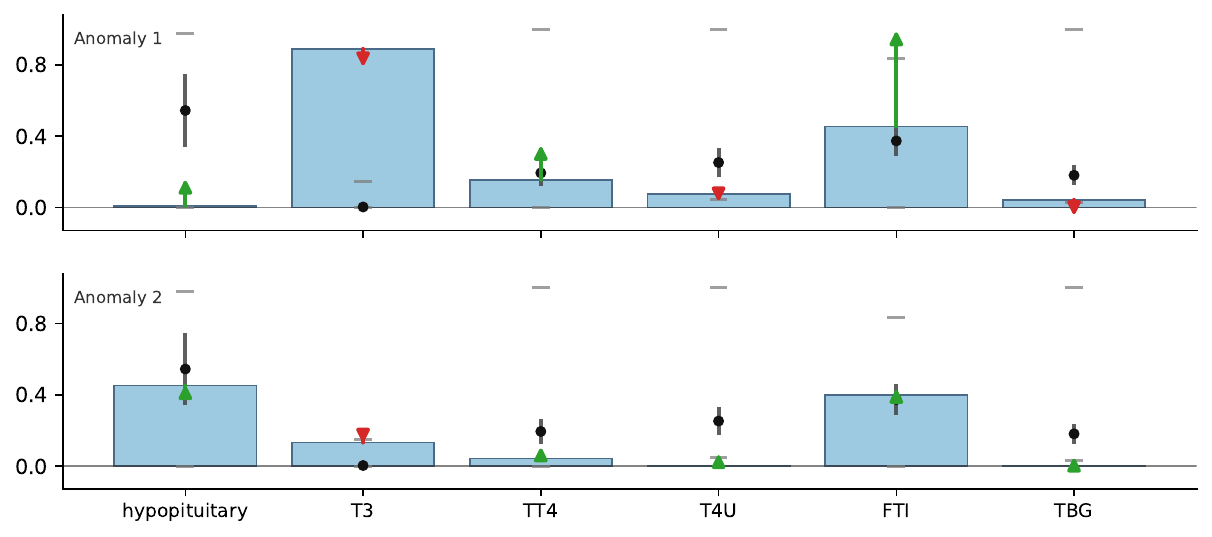}

  \caption{Feature values with directional attributions for two anomaly \texttt{thyroid} cases. Arrows encode suggested adjustments based on the estimated score (green/red; magnitude $\propto |score|$). Grey ticks denote normal min/max; black whiskers show mean ± 1 SD and dots the mean, estimated from normal training data. }
  \vspace{-1em}
  \label{fig:thyroid-grid}
\end{figure}

\paragraph{Tabular case (UCI \texttt{thyroid}).}
We illustrate the interpretability of score matching on tabular data using the \texttt{thyroid} dataset in Figure~\ref{fig:thyroid-grid}. For each example, we plot raw feature values (bars) together with per-feature arrows derived from the estimated score vector $s_\theta(x)$; the sign encodes the direction the model would adjust that feature to move x toward higher training-data density, and the arrow length is proportional to $|s_i(x)|$. To provide clinical context, grey ticks mark the empirical normal range (min/max) and black whiskers denote mean with standard deviation across normal training samples, with the mean shown as a dot. Two representative anomalies highlight how s(x) captures what the model “thinks is wrong.” Anomaly 1 shows a low hypopituitary indicator alongside elevated T3, while its FTI is not correspondingly high; the score suggests a substantial increase in FTI (achievable via higher TT4 or uptake/lower TBG), making the profile more self-consistent with the hyperthyroid-like pattern in the training normals, where T3, TT4/uptake, and thus FTI tend to rise together. Anomaly 2 presents low T4U and low TBG, a physiologically discordant pair because T3 uptake (T4U) varies inversely with TBG; reduced TBG typically increases uptake, not lowers it. The score highlights these features as principal drivers of the anomaly, pushing them toward a configuration seen in normal data. We emphasize that s(x) explains movement toward the data manifold learned from training normals, which often (not always) aligns with textbook physiology. The qualitative interpretations above are consistent with standard thyroid-function test relationships (e.g., the Free Thyroxine Index reflecting TT4 and binding/uptake, and the inverse coupling of T3 uptake and TBG in binding-affecting states) as summarized in clinical endocrinology references \citep{Spencer2025Endotext,Dunlap1990ClinicalMethods}.

\begin{figure}[h]
  \centering
  \begin{subfigure}[t]{.22\linewidth}
    \includegraphics[width=\linewidth]{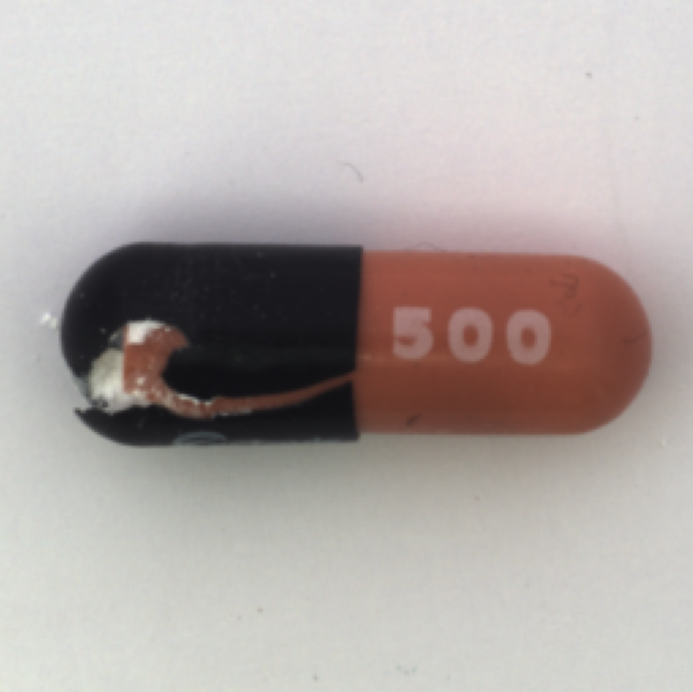}
    \caption*{(a) Input}
  \end{subfigure}\hfill
  \begin{subfigure}[t]{.22\linewidth}
    \includegraphics[width=\linewidth]{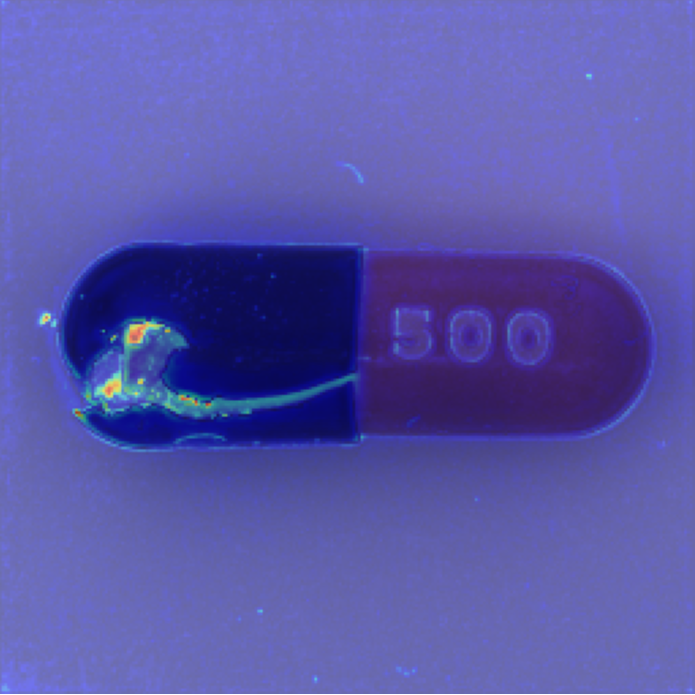}
    \caption*{(b) $\|s_\theta(x)\|$}
  \end{subfigure}\hfill
  \begin{subfigure}[t]{.22\linewidth}
    \includegraphics[width=\linewidth]{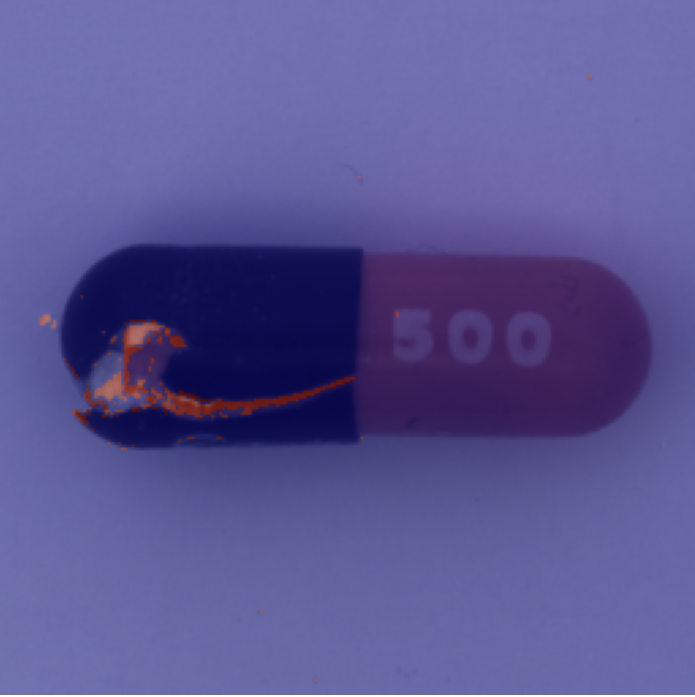}
    \caption*{(c) Overlay}
  \end{subfigure}\hfill
  \begin{subfigure}[t]{.22\linewidth}
    \includegraphics[width=\linewidth]{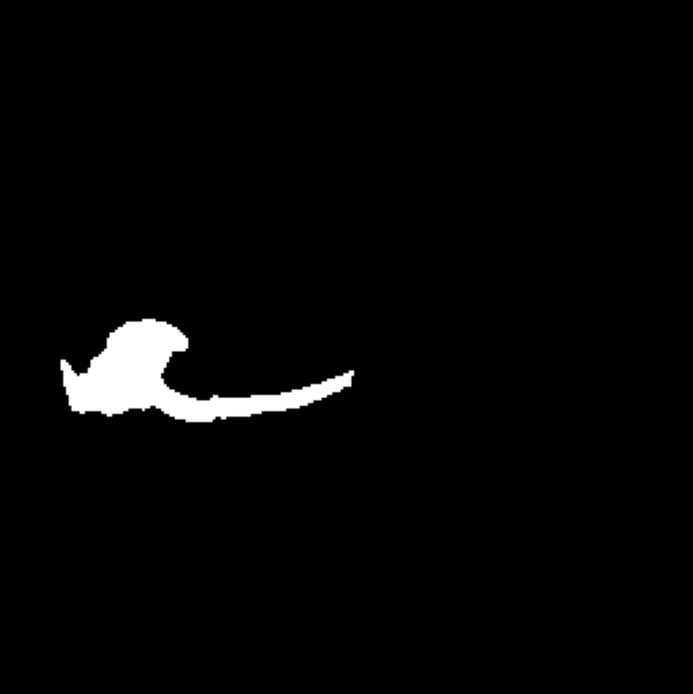}
    \caption*{(d) Mask}
  \end{subfigure}
  \vspace{-0.4em}
  \caption{Score-based localization on MVTec–capsule.
(a) Input image, (b) per-pixel score magnitude $\|s_\theta(x)\|$, (c) overlay of scores on the input, and (d) ground-truth defect mask. A DSM model trained only on normal images highlights the defect region. }
  \label{fig:capsule-vis}
\end{figure}

\paragraph{Vision case (\texttt{MVTec}–capsule).}
We also illustrate the interpretability of score-matching for image anomaly detection. As seen in Figure~\ref{fig:capsule-vis}, the score field $s_\theta(x)$ exposes regions of the image that deviate from the learned normal manifold of defect-free capsules. Trained only on normals, the per-pixel magnitude $\|s_\theta(x)\|$ concentrates along the crack and spilled material near the seam, while remaining low on intact areas such as the printed “500” and the uniform background. The overlay makes this concentration visually apparent, and the ground-truth mask confirms that the saliency peaks align with the annotated defect. No defect labels or auxiliary segmentation modules are involved: the signal arises directly from the score learned on normals. Although we display magnitude for clarity, the full vector field is actionable: moving locally opposite the salient gradients would attenuate sharp luminance and texture discontinuities along the seam, nudging the image toward an intact, high-likelihood appearance. This is a simple example, single-scale score matching does not capture high-dimensional image structure particularly well; multi-scale formulations are better suited for such data, but the example illustrates the general interpretability argument for score-based methods.

\section{Proofs for Proposition~\ref{prop:kurtosis-exact}}
\label{app:kurtosis-proofs}

\subsection{Part (a): Generalised Gaussian Distribution}

The generalised Gaussian (GGD) with shape $\beta>0$ and location zero has density $p(x;\beta,a) \propto \exp(-(|x|/a)^\beta)$. Setting $a(\beta)=[\Gamma(1/\beta)/\Gamma(3/\beta)]^{1/2}$ gives unit variance.

\paragraph{Kurtosis.}
Using the substitution $u=(x/a)^\beta$, the $2k$-th moment of the GGD reduces to
\[
  \mathbb{E}[X^{2k}] = \frac{a^{2k}\,\Gamma((2k+1)/\beta)}{\Gamma(1/\beta)}.
\]
Setting $k{=}2$ and dividing by $(\mathbb{E}[X^2])^2$ (the scale $a$ cancels under unit variance):
\begin{equation}
  \kappa(\beta) = \frac{\Gamma(5/\beta)\,\Gamma(1/\beta)}{\Gamma(3/\beta)^2}.
  \label{eq:ggd-kurtosis}
\end{equation}
Special cases: Laplace ($\beta{=}1$): $\kappa=24/4=6$; Gaussian ($\beta{=}2$): $\kappa=3$; Uniform ($\beta{\to}\infty$): $\kappa\to 9/5=1.8$.

\paragraph{Strict decrease of $\kappa(\beta)$.}
We give a fully elementary proof via the Weierstrass product
representation of the gamma function. Setting $x=1/\beta>0$, let
$L(x)=\ln\kappa(1/x)=\ln\Gamma(5x)+\ln\Gamma(x)-2\ln\Gamma(3x)$;
strict monotonicity of $L$ in $x$ is equivalent to strict
monotonicity of $\kappa$ in $\beta$ (with opposite sign). Weierstrass
gives, for $z>0$,
\[
  \ln\Gamma(z)\;=\;-\ln z - \gamma z
  + \sum_{k=1}^{\infty}\Bigl[\tfrac{z}{k}-\ln\!\bigl(1+\tfrac{z}{k}\bigr)\Bigr].
\]
Substituting $z\in\{5x,x,3x\}$ and collecting terms, the linear and
$-\ln z$ contributions combine to the constant $\ln(9/5)$ (the
Euler constants and $\sum z/k$ terms cancel because
$5x+x-2\cdot 3x=0$), leaving
\begin{equation}
  L(x)\;=\;\ln\tfrac{9}{5}
  \;+\;\sum_{k=1}^{\infty}\ln\frac{(1+3x/k)^{2}}{(1+5x/k)(1+x/k)}.
  \label{eq:weierstrass-L}
\end{equation}
Differentiating term by term (justified by uniform convergence on
compact subsets of $x>0$, since the $k$-th term is $O(x^{2}/k^{2})$):
\[
  \frac{dL}{dx}
  \;=\;\sum_{k=1}^{\infty}\!\left[\frac{6}{k+3x}-\frac{5}{k+5x}-\frac{1}{k+x}\right]
  \;=\;\sum_{k=1}^{\infty}\frac{8\,k\,x}{(k+x)(k+3x)(k+5x)},
\]
where the second equality follows from direct algebraic combination
over the common denominator
$(k+x)(k+3x)(k+5x)$: expanding numerators gives
$6(k{+}5x)(k{+}x)-5(k{+}3x)(k{+}x)-(k{+}3x)(k{+}5x)=8kx$. Every summand
is strictly positive for $x>0$, so $L'(x)>0$; equivalently
$d\kappa/d\beta<0$ for all $\beta>0$. \qed

\paragraph{Strict decrease of $R_\tau(\beta)$ in the tail regime.}
Unlike kurtosis, strict monotonicity of the standardised tail
radius does \emph{not} hold for all $\tau$: near the bulk
($\tau\approx 0.5$), lighter-tailed GGDs have \emph{larger} standardised
quantiles because they concentrate mass near the mode, and this
bulk effect can dominate. The claim therefore has to be localised
to small $\tau$.

Write $R_\tau(\beta)=q(\beta,\tau)/v(\beta)^{1/2}$, where
$q(\beta,\tau)=Q_{1-\tau/2}$ of the unit-scale density
$p(x;\beta)\propto e^{-|x|^{\beta}}$ and
$v(\beta)=\Gamma(3/\beta)/\Gamma(1/\beta)$ is its variance. As
$\tau\to 0$, Laplace's method applied to the survival function
$\bar F(x;\beta)\sim c(\beta)\,x^{1-\beta}e^{-x^{\beta}}$ gives the
classical tail asymptotic $q(\beta,\tau)\sim\bigl(\ln(2/\tau)\bigr)^{1/\beta}$,
whereas $v(\beta)$ stays bounded as $\beta\downarrow 0$ and decreases
monotonically in $\beta$ on the relevant range. Thus for every fixed
$\beta_{1}<\beta_{2}$ in any compact sub-interval of $(0,\infty)$,
there exists $\tau_{0}(\beta_{1},\beta_{2})>0$ such that
$R_\tau(\beta_{1})>R_\tau(\beta_{2})$ for all
$\tau\in(0,\tau_{0})$.
Taking the infimum of $\tau_{0}$ over the practical range
$\beta\in[0.5,50]$ gives a uniform threshold
$\tau^{*}_{\mathrm{GGD}}$. Direct numerical computation of
$R_\tau(\beta)$ over a dense grid in this range determines
$\tau^{*}_{\mathrm{GGD}}\approx 0.028$, so every choice
$\tau\in(0,0.02]$ is safely in the strictly-monotone regime.

\paragraph{Composition.}
Combining the two monotonicities: on $\tau\in(0,\tau^{*}_{\mathrm{GGD}})$
both $\kappa(\beta)$ and $R_\tau(\beta)$ are strictly decreasing in
$\beta$, so the composite $R_\tau\circ\kappa^{-1}$ is strictly
\emph{increasing} in $\kappa$. \qed

\subsection{Part (b): Student-$t$ Distribution}

The Student-$t$ with $\nu>4$ has variance $\nu/(\nu-2)$. The general even moment follows from the Beta-function integral $\int_0^\infty t^{k-1/2}(1+t)^{-(\nu+1)/2}dt=B(k{+}1/2,\,(\nu-2k)/2)$, giving $\mathbb{E}[X^{2k}]=\nu^k\Gamma(k{+}1/2)\Gamma((\nu-2k)/2)/(\sqrt{\pi}\,\Gamma(\nu/2))$. For $k{=}1,2$ and applying the gamma recurrence $\Gamma(z+1)=z\Gamma(z)$:
\[
  \mathbb{E}[X^2] = \frac{\nu}{\nu-2}, \qquad \mathbb{E}[X^4] = \frac{3\nu^2}{(\nu-2)(\nu-4)}.
\]
The Pearson kurtosis is therefore
\begin{equation}
  \kappa(\nu) = \frac{\mathbb{E}[X^4]}{(\mathbb{E}[X^2])^2} = \frac{3(\nu-2)}{\nu-4} = 3+\frac{6}{\nu-4}.
  \label{eq:t-kurtosis}
\end{equation}
Kurtosis is scale-invariant, so the standardised version $Z=X\sqrt{(\nu-2)/\nu}$ has the same $\kappa$.

\paragraph{Strict decrease of $\kappa(\nu)$.}
$d\kappa/d\nu=-6/(\nu-4)^2<0$ for all $\nu>4$. \qed

\paragraph{Strict decrease of $R_\tau(\nu)$ in the tail regime.}
As in the GGD case, \emph{raw} (unit-scale) Student-$t$ quantiles
are strictly decreasing in $\nu$ for every fixed
$\tau\in(0,1)$: the family is stochastically ordered, because
$\bar{F}(x;\nu_{1})>\bar{F}(x;\nu_{2})$ for $\nu_{1}<\nu_{2}$ and all
$x>0$ \citep{JohnsonKotzBalakrishnan1995}. The standardised radius
multiplies the raw quantile by the scaling factor
$\sqrt{(\nu-2)/\nu}$, which is strictly \emph{increasing} in $\nu$
and therefore pushes in the opposite direction. For small $\tau$
the raw-quantile decrease dominates, but near the bulk
(large $\tau$) the rescaling can win and strict monotonicity fails.
The classical Student-$t$ tail bound
$\bar{F}_\nu(x)\sim c_{\nu}x^{-\nu}$ as $x\to\infty$ yields
$q(\nu,\tau)\sim (2c_{\nu}/\tau)^{1/\nu}$, so for each fixed pair
$4<\nu_{1}<\nu_{2}$ there is a threshold $\tau_{0}(\nu_{1},\nu_{2})>0$
below which the raw-quantile ratio exceeds
$\sqrt{\nu_{2}(\nu_{1}-2)/(\nu_{1}(\nu_{2}-2))}$ and the standardised
radius is strictly decreasing in $\nu$. Taking the infimum over
$\nu\in(4,\infty)$ (equivalently over the compact kurtosis range
used in practice) gives a uniform
$\tau^{*}_{t}$. A numerical sweep over
$\nu\in[4.1,500]$ gives
$\tau^{*}_{t}\approx 0.031$, so any choice
$\tau\in(0,0.02]$ lies strictly inside the monotone regime.

\paragraph{Composition.}
$\kappa(\nu)$ is invertible with $\nu(\kappa)=4+6/(\kappa-3)$. On
$\tau\in(0,\tau^{*}_{t})$ both $\kappa$ and $R_\tau$ decrease strictly
in $\nu$, so the composite $R_\tau(\nu(\kappa))$ is strictly
\emph{increasing} in $\kappa$. \qed

\subsection{Histogram Rearrangement}
\label{app:rearrangement}

Our histogram rearrangement is the discrete analogue of the \emph{symmetric decreasing rearrangement} \citep{LiebLoss2001,HardyLittlewoodPolya1952}. The rearrangement $f^*$ of a density $f$ is the unique function that is symmetric about zero, non-increasing on $[0,\infty)$, and equimeasurable with $f$:
\[
  |\{x : f^*(x)>t\}| = |\{x : f(x)>t\}| \quad \forall\, t>0.
\]
Equimeasurability preserves all $L^p$ norms of $f$, but \emph{relocates mass in $x$}, thereby changing the distribution's moments. This is the intended effect: kurtosis computed on $f^*$ reflects only the rate of decay from the mode, free of confounders such as bimodality and skewness. As $B\to\infty$ and $n\to\infty$, the rearranged histogram converges to $f^*$ in $L^1$ by histogram consistency \citep{Devroye1987} and the $L^p$-contraction of the rearrangement \citep[Theorem~3.4]{LiebLoss2001}.


\begin{algorithm}[t]
\caption{Histogram Rearrangement}
\label{alg:rearrangement_new}
\begin{algorithmic}[1]
\Require Data vector $\mathbf{x}$, number of bins $B$
\Ensure Transformed vector $\mathbf{x}^*$

\State \textbf{Histogram construction}
\State Compute histogram counts $(h_1,\dots,h_B)$ and bin edges
\State Assign each sample to a bin index $b_i \in \{1,\dots,B\}$

\vspace{0.5em}
\State \textbf{Rank bins by frequency}
\State Compute permutation $\pi$ such that
\Statex \hspace{1.5em} $h_{\pi(1)} \ge h_{\pi(2)} \ge \cdots \ge h_{\pi(B)}$

\vspace{0.5em}
\State \textbf{Construct symmetric rank embedding}
\State Initialize $c^*_{\pi(1)} \gets 0$
\State Set $s \gets 1$
\For{$j = 2$ to $B$}
    \If{$j$ is even}
        \State $c^*_{\pi(j)} \gets +s$
    \Else
        \State $c^*_{\pi(j)} \gets -s$
        \State $s \gets s + 1$
    \EndIf
\EndFor

\vspace{0.5em}
\State \textbf{Transform data}
\For{each sample $x_i$}
    \State $x_i^* \gets c^*_{b_i}$
\EndFor

\Return $\mathbf{x}^*$
\end{algorithmic}
\end{algorithm}

\section{Derivation of the Cornish--Fisher Tail-Radius Expansion}
\label{app:cornish-fisher}
 
We give the full derivation of
Proposition~\ref{prop:cornish-fisher}.  The strategy is standard:
start from the Edgeworth expansion of the CDF, then invert it to
obtain a quantile expansion.
 
\paragraph{Setup.}
Let $X$ be a random variable with a symmetric, standardised
($\mu=0$, $\sigma^2=1$) density.  Because $X$ is symmetric, all
odd cumulants vanish: $\kappa_3 = \kappa_5 = \cdots = 0$.  The
first non-Gaussian cumulant is the excess kurtosis
$\kappa_4 = \kappa - 3$, where $\kappa = \mathbb{E}[X^4]$.
 
\paragraph{Step 1: Edgeworth expansion of the CDF.}
The Edgeworth expansion~\citep{Hall1992} expresses the CDF $F(x)$ of
$X$ as a perturbation of the standard normal CDF $\Phi(x)$ in terms
of cumulants.  Retaining only the leading term and using
$\kappa_3 = 0$:
\begin{equation}
  F(x) \;=\; \Phi(x)
  \;-\; \frac{\kappa - 3}{24}\,\mathrm{He}_3(x)\,\phi(x)
  \;+\; O\!\bigl((\kappa-3)^2\bigr),
  \label{eq:edgeworth}
\end{equation}
where $\phi(x)$ is the standard normal density and
$\mathrm{He}_3(x) = x^3 - 3x$ is the third probabilist's Hermite
polynomial.
 
\paragraph{Step 2: Inversion to obtain the quantile function.}
Fix a probability level $p = 1 - \tau/2$ and let $z_\tau =
\Phi^{-1}(p)$ be the corresponding Gaussian quantile.  We seek the
quantile $R_\tau = F^{-1}(p)$.  Write $R_\tau = z_\tau + \delta$
where $\delta$ is a small correction.  Substituting into
\eqref{eq:edgeworth} and imposing $F(R_\tau) = p = \Phi(z_\tau)$:
\begin{equation}
  \Phi(z_\tau + \delta)
  \;-\; \frac{\kappa - 3}{24}\,\mathrm{He}_3(z_\tau + \delta)\,
        \phi(z_\tau + \delta)
  \;+\; O\!\bigl((\kappa-3)^2\bigr)
  \;=\; \Phi(z_\tau).
\end{equation}
Expanding $\Phi(z_\tau + \delta) \approx \Phi(z_\tau) +
\delta\,\phi(z_\tau)$ and replacing $\mathrm{He}_3(z_\tau + \delta)$
and $\phi(z_\tau + \delta)$ by their values at $z_\tau$ (since
$\delta$ is itself of order $(\kappa-3)$, the cross terms are
second-order), we obtain
\begin{equation}
  \delta\,\phi(z_\tau)
  \;=\;
  \frac{\kappa - 3}{24}\,\mathrm{He}_3(z_\tau)\,\phi(z_\tau)
  \;+\; O\!\bigl((\kappa-3)^2\bigr).
\end{equation}
Dividing both sides by $\phi(z_\tau) > 0$:
\begin{equation}
  \delta
  \;=\;
  \frac{\kappa - 3}{24}\bigl(z_\tau^3 - 3\,z_\tau\bigr)
  \;+\; O\!\bigl((\kappa-3)^2\bigr).
\end{equation}
 
\paragraph{Step 3: Tail-radius expansion.}
Substituting $R_\tau = z_\tau + \delta$ yields
Equation~\eqref{eq:cornishfisher}:
\begin{equation}
  R_\tau
  \;=\;
  z_\tau
  \;+\;
  \frac{\kappa - 3}{24}\bigl(z_\tau^3 - 3\,z_\tau\bigr)
  \;+\;
  O\!\bigl((\kappa-3)^2\bigr).
  \tag{\ref{eq:cornishfisher}}
\end{equation}
 
\paragraph{Step 4: Monotonicity of $R_\tau$ in $\kappa$.}
Differentiating with respect to $\kappa$:
\begin{equation}
  \frac{\partial R_\tau}{\partial\kappa}
  \;=\;
  \frac{z_\tau^3 - 3\,z_\tau}{24}
  \;=\;
  \frac{z_\tau(z_\tau^2 - 3)}{24}.
\end{equation}
This is strictly positive when $z_\tau > \sqrt{3} \approx 1.732$,
i.e.\ $\tau < 2(1-\Phi(\sqrt{3})) \approx 0.083$.  For the standard
choice $\tau = 0.05$ we have $z_\tau \approx 1.960$, giving
$\partial R_\tau / \partial\kappa \approx 0.0687 > 0$. Note that
this validity threshold $\tau\lesssim 0.083$ is consistent with the
exact small-$\tau$ monotonicity regimes established for GGD and
Student-$t$ in Appendix~\ref{app:kurtosis-proofs} and justifies the
practical range $\tau\in(0,0.02]$ used in Proposition~\ref{prop:kurtosis-exact}. \qed

\section{Proofs for Theorem~\ref{thm:min-noise-coverage} and Corollary~\ref{cor:sigma-monotone}}
\label{app:min-noise-proofs}

\subsection{Proof of Theorem~\ref{thm:min-noise-coverage}}

We prove the three claims of the theorem in turn: (i) rewriting the
coverage constraint as an explicit lower bound on $\sigma$; (ii)
identifying the unique minimiser; and (iii) establishing strict
monotonicity of $\sigma^{*}$ in $\kappa$.

\paragraph{Step 1: Rewriting the constraint.}
Let $\varepsilon\sim\mathcal{N}(0,\sigma^{2})$ with $\sigma>0$ and let
$R=R_\tau(\kappa)>0$ denote the $\tau$-tail radius of the marginal
$X$. Because $\varepsilon/\sigma\sim\mathcal{N}(0,1)$ and the standard
normal is symmetric about zero,
\begin{align}
  \Pr\!\bigl(|\varepsilon|\geq R\bigr)
  &\;=\;\Pr\!\bigl(\varepsilon\geq R\bigr)+\Pr\!\bigl(\varepsilon\leq -R\bigr)\notag\\
  &\;=\;\bigl(1-\Phi(R/\sigma)\bigr)+\Phi(-R/\sigma)\notag\\
  &\;=\;2\bigl(1-\Phi(R/\sigma)\bigr),
\label{eq:coverage-explicit}
\end{align}
using $\Phi(-u)=1-\Phi(u)$ in the last step. The coverage constraint
$\Pr(|\varepsilon|\ge R)\ge\alpha$ is therefore equivalent to
$1-\Phi(R/\sigma)\ge\alpha/2$, i.e.\
\begin{equation}
  \Phi(R/\sigma)\;\leq\;1-\alpha/2.
\label{eq:phi-constraint}
\end{equation}
Since $\Phi$ is a strictly increasing bijection from $\mathbb{R}$ onto
$(0,1)$ and $1-\alpha/2\in(0,1)$, its inverse
$\Phi^{-1}(1-\alpha/2)$ is a finite positive real number (positive
because $\alpha<1$, hence $1-\alpha/2>1/2=\Phi(0)$). Applying
$\Phi^{-1}$ to both sides of~\eqref{eq:phi-constraint} preserves the
direction of the inequality and yields
\[
  \frac{R}{\sigma}\;\leq\;\Phi^{-1}(1-\alpha/2),
\]
which, since $R>0$ and $\sigma>0$, is equivalent to
\begin{equation}
  \sigma\;\geq\;\frac{R}{\Phi^{-1}(1-\alpha/2)}
  \;=\;\frac{R_\tau(\kappa)}{\Phi^{-1}(1-\alpha/2)}.
\label{eq:sigma-lower}
\end{equation}

\paragraph{Step 2: Uniqueness of the minimiser.}
Define
\[
  \sigma_{0}\;:=\;\frac{R_\tau(\kappa)}{\Phi^{-1}(1-\alpha/2)}\;>\;0.
\]
Step~1 shows that the feasible set of~\eqref{eq:min-noise-problem} is
exactly $\{\sigma>0:\sigma\ge\sigma_{0}\}=[\sigma_{0},\infty)$. The
objective $\sigma\mapsto\sigma$ is strictly increasing on
$(0,\infty)$, so its infimum over $[\sigma_{0},\infty)$ is attained
at the unique point $\sigma=\sigma_{0}$. Strictness of the
monotonicity rules out any other minimiser, establishing
uniqueness and confirming~\eqref{eq:min-noise-sol}:
\[
  \sigma^{*}(\kappa)\;=\;\frac{R_\tau(\kappa)}{\Phi^{-1}(1-\alpha/2)}.
\]
We note for later use that, because $\sigma_{0}$ is the unique
feasible point where the constraint binds, we have
\[
  \Pr_{\varepsilon\sim\mathcal{N}(0,(\sigma^{*})^{2})}\!\bigl(|\varepsilon|\ge R_\tau(\kappa)\bigr)\;=\;\alpha,
\]
i.e.\ the optimal rule achieves \emph{exactly} the prescribed
coverage, not merely at least $\alpha$.

\paragraph{Step 3: Strict monotonicity in $\kappa$.}
Write $c_{\alpha}:=1/\Phi^{-1}(1-\alpha/2)>0$, so that the closed
form of Step~2 reads
$\sigma^{*}(\kappa)=c_{\alpha}\,R_\tau(\kappa)$. The factor
$c_{\alpha}$ depends only on $\alpha$ and is in particular
independent of $\kappa$. Consequently, for any
$\kappa_{1}<\kappa_{2}$ inside the strictly-monotone regime of
Proposition~\ref{prop:kurtosis-exact} (i.e.\ either within the GGD
family with $\tau\in(0,\tau^{*}_{\mathrm{GGD}})$ or within the
standardised Student-$t$ family with $\tau\in(0,\tau^{*}_{t})$),
Proposition~\ref{prop:kurtosis-exact} gives
$R_\tau(\kappa_{1})<R_\tau(\kappa_{2})$, and multiplying by the
positive constant $c_{\alpha}$ preserves the inequality:
\[
  \sigma^{*}(\kappa_{1})\;=\;c_{\alpha}\,R_\tau(\kappa_{1})
  \;<\;c_{\alpha}\,R_\tau(\kappa_{2})\;=\;\sigma^{*}(\kappa_{2}).
\]
Hence $\sigma^{*}$ is strictly increasing in $\kappa$ under the
hypotheses of Proposition~\ref{prop:kurtosis-exact}. Combining
Steps~1--3 proves Theorem~\ref{thm:min-noise-coverage}. \qed

\subsection{Proof of Corollary~\ref{cor:sigma-monotone}}

The corollary has two parts: a first-order expansion
of~\eqref{eq:min-noise-sol} around the Gaussian reference $\kappa=3$,
and a uniqueness statement identifying the affine
rule~\eqref{eq:sigma_formula} as the unique first-order Taylor
approximation of $\sigma^{*}$ at $\kappa=3$.

\paragraph{Step 1: Cornish--Fisher substitution.}
Under the hypotheses of Proposition~\ref{prop:cornish-fisher}, the
tail radius admits the Cornish--Fisher expansion
\begin{equation}
  R_\tau(\kappa)\;=\;z_\tau
  \;+\;\frac{\kappa-3}{24}\bigl(z_\tau^{3}-3z_\tau\bigr)
  \;+\;O\!\bigl((\kappa-3)^{2}\bigr),
\label{eq:cf-recap}
\end{equation}
where $z_\tau=\Phi^{-1}(1-\tau/2)$. At the Gaussian reference the
excess kurtosis vanishes and the expansion reduces to
$R_\tau(3)=z_\tau$. Substituting~\eqref{eq:cf-recap} into the
closed-form~\eqref{eq:min-noise-sol} of Theorem~\ref{thm:min-noise-coverage},
\[
  \sigma^{*}(\kappa)
  \;=\;\frac{R_\tau(\kappa)}{\Phi^{-1}(1-\alpha/2)}
  \;=\;\frac{1}{\Phi^{-1}(1-\alpha/2)}\!\left[z_\tau+\frac{\kappa-3}{24}(z_\tau^{3}-3z_\tau)+O((\kappa-3)^{2})\right].
\]
Evaluating at $\kappa=3$ gives $\sigma^{*}(3)=z_\tau/\Phi^{-1}(1-\alpha/2)$,
which is strictly positive.

\paragraph{Step 2: Cancellation and simplification.}
Forming the ratio $\sigma^{*}(\kappa)/\sigma^{*}(3)$, the positive
factor $1/\Phi^{-1}(1-\alpha/2)$ cancels identically:
\begin{align}
  \frac{\sigma^{*}(\kappa)}{\sigma^{*}(3)}
  &\;=\;\frac{R_\tau(\kappa)}{R_\tau(3)}
  \;=\;\frac{z_\tau+\tfrac{\kappa-3}{24}(z_\tau^{3}-3z_\tau)+O((\kappa-3)^{2})}{z_\tau}\notag\\
  &\;=\;1\;+\;\frac{\kappa-3}{24}\cdot\frac{z_\tau^{3}-3z_\tau}{z_\tau}
      \;+\;O\!\bigl((\kappa-3)^{2}\bigr)\notag\\
  &\;=\;1\;+\;\frac{z_\tau^{2}-3}{24}(\kappa-3)\;+\;O\!\bigl((\kappa-3)^{2}\bigr),
\label{eq:cor-expansion}
\end{align}
where the last line uses $z_\tau>0$ to divide numerator and
denominator by $z_\tau$ and the identity
$(z_\tau^{3}-3z_\tau)/z_\tau=z_\tau^{2}-3$. This
establishes~\eqref{eq:optimal-affine} with slope
$c_{\mathrm{CF}}(\tau)=(z_\tau^{2}-3)/24$.

\paragraph{Step 3: Sign of the slope.}
The slope $c_{\mathrm{CF}}(\tau)=(z_\tau^{2}-3)/24$ is a strictly
increasing function of $z_\tau\ge 0$ and vanishes precisely at
$z_\tau=\sqrt{3}$. Since $z_\tau=\Phi^{-1}(1-\tau/2)$ is itself
strictly decreasing in $\tau$, the condition $z_\tau>\sqrt{3}$ is
equivalent to
\[
  1-\tau/2\;>\;\Phi(\sqrt{3})\quad\Longleftrightarrow\quad
  \tau\;<\;2\bigl(1-\Phi(\sqrt{3})\bigr)\;\approx\;0.083.
\]
In this regime $c_{\mathrm{CF}}(\tau)>0$ and the first-order
correction in~\eqref{eq:cor-expansion} has the same sign as
$\kappa-3$, which is consistent with the strict monotonicity
established in Theorem~\ref{thm:min-noise-coverage}.

\paragraph{Step 4: Uniqueness of the affine Taylor approximation.}
It remains to show that the affine rule
$\sigma^{\mathrm{aff}}(\kappa)=\sigma_{\text{base}}[1+c(\kappa-3)]$
appearing in Eq.~\eqref{eq:sigma_formula} is the unique first-order
Taylor approximation of $\sigma^{*}(\kappa)$ around $\kappa=3$. Let
$\sigma^{\mathrm{aff}}$ be any affine function of $\kappa$ matching
$\sigma^{*}$ to first order at $\kappa=3$; we must determine
$\sigma_{\text{base}}$ and $c$. The zeroth-order matching condition
$\sigma^{\mathrm{aff}}(3)=\sigma^{*}(3)$ immediately gives
\[
  \sigma_{\text{base}}\;=\;\sigma^{*}(3)\;=\;\frac{z_\tau}{\Phi^{-1}(1-\alpha/2)}.
\]
For the first-order condition, differentiate
$\sigma^{\mathrm{aff}}(\kappa)$ and $\sigma^{*}(\kappa)$ at
$\kappa=3$. The affine ansatz gives
$d\sigma^{\mathrm{aff}}/d\kappa\big|_{\kappa=3}=\sigma_{\text{base}}\,c$,
whereas differentiating~\eqref{eq:cor-expansion} (or equivalently
$\sigma^{*}(\kappa)=\sigma^{*}(3)\cdot[1+(z_\tau^{2}-3)(\kappa-3)/24+\cdots]$)
at $\kappa=3$ gives
\[
  \frac{d\sigma^{*}}{d\kappa}\bigg|_{\kappa=3}
  \;=\;\sigma^{*}(3)\cdot\frac{z_\tau^{2}-3}{24}
  \;=\;\sigma_{\text{base}}\cdot c_{\mathrm{CF}}(\tau).
\]
Equating the two derivatives and cancelling the common factor
$\sigma_{\text{base}}>0$ yields $c=c_{\mathrm{CF}}(\tau)$. Both
parameters of the affine ansatz are thereby uniquely determined,
proving that Eq.~\eqref{eq:sigma_formula} with
$c=c_{\mathrm{CF}}(\tau)$ and
$\sigma_{\text{base}}=\sigma^{*}(3)$ is the unique first-order
Taylor approximation of $\sigma^{*}(\kappa)$ around the Gaussian
reference. Combining Steps~1--4 completes the proof of
Corollary~\ref{cor:sigma-monotone}. \qed

\subsection{Relating the empirical c to the theoretical slope}
\label{app:c-tau}
Theorem~\ref{thm:min-noise-coverage} parameterises the minimum-noise rule by two quantities: the tail-level $\tau$ (how deep into the tail a training point must be displaced) and the coverage level $\alpha$ (with what probability the displacement reaches that depth). When we form the ratio $\sigma^{*}(\kappa)/\sigma^{*}(3)$ in Corollary~\ref{cor:sigma-monotone}, the $\Phi^{-1}(1-\alpha/2)$ factor that encodes $\alpha$ cancels identically, leaving a first-order slope $c_{\mathrm{CF}}(\tau)=(z_\tau^{2}-3)/24$ that depends only on $\tau$. The coverage level $\alpha$ therefore affects only the overall scale $\sigma_{\text{base}}$, not the rate at which $\sigma$ should grow with $\kappa$; any mismatch between the empirical $c$ and the theoretical slope must be read as a mismatch in $\tau$, not in $\alpha$. Table~\ref{tab:tau-c} lists the correspondence induced by this slope for the range of $\tau$ we ablate in Figure~\ref{fig:ablation_tau}:
\begin{table}[h]
\centering
\small
\caption{Implicit $(c,\tau)$ correspondence induced by $c_{\mathrm{CF}}(\tau)=(z_\tau^{2}-3)/24$, evaluated on the grid swept in the ablation of Figure~\ref{fig:ablation_tau}.}
\label{tab:tau-c}
\begin{tabular}{lcccccccc}
\toprule
$\tau$ & $10^{-5}$ & $5\!\times\!10^{-5}$ & $10^{-4}$ & $10^{-3}$ & $5\!\times\!10^{-3}$ & $10^{-2}$ & $0.05$ & $0.07$ \\
$c_{\mathrm{CF}}(\tau)$ & $0.875$ & $0.560$ & $0.506$ & $0.326$ & $0.203$ & $0.151$ & $0.035$ & $0.012$ \\
\bottomrule
\end{tabular}
\end{table}
Two observations close the gap between $c=0.33$ and the nominal $c_{\mathrm{CF}}(0.05)\approx 0.035$. First, the empirical optimum $c=0.326$ in Figure~\ref{fig:ablation_tau} corresponds to $\tau\approx 10^{-3}$, not to the nominal $\tau=0.05$: the score network benefits from \emph{dense} coverage of the tail region rather than \emph{minimal} coverage at the boundary, so the rule that minimises displacement under a $\tau=0.05$ constraint is not the rule that maximises detection quality. Concretely, $c_{\mathrm{CF}}(0.05)\approx 0.035$ is the minimum slope consistent with any training point reaching the 5\%-tail of the marginal with probability $\alpha$; increasing the slope to $c=0.33$ multiplies the per-feature displacement by roughly a factor of ten and shifts the coverage from ``one training point lands in the tail boundary'' to ``most training points cover the tail region''. Second, the ablation in Section~\ref{sec:results} shows that detection quality is a flat function of $c$ across the entire range: every value of $c$ from $0.012$ ($\tau=0.07$) to $0.875$ ($\tau=10^{-5}$) leaves mean AUC-PR within a narrow band, and the best value of $0.326$ beats the worst only by a small margin. We therefore fix $c=0.33\approx 1/3$ as a simple, interpretable default that sits at the empirical plateau and neatly matches the Gaussian anchor point $\kappa_j=3$ at which the formula reduces to $\sigma_{\text{base}}$. The theoretical contribution of Corollary~\ref{cor:sigma-monotone} is the functional form $\sigma_j=\sigma_{\text{base}}[1+c(\kappa_j-3)]$ and the sign of $c$; the specific value of $c$ used in practice is fixed empirically from the ablation, not derived from a single nominal $\tau$. We flag as an open question the fact that $\alpha$ cancels in the first-order expansion: any higher-order analysis that retains $\alpha$ would recover it as a second-order correction coupling $\kappa$ and the coverage level, and could in principle pin down the joint $(\alpha,\tau)$ that matches $c=0.33$ exactly.

\section{Implementation Details}
\subsection{Benchmark Description}

We conduct experiments on the ADBench benchmark suite \citep{han2022adbench}, which provides a standardised evaluation framework across 57 datasets and 30 anomaly detection algorithms. The benchmark includes 47 established tabular datasets and 10 additional datasets introduced by the authors to extend coverage into computer vision and natural language processing. For these latter domains, features are obtained from pre-trained models such as ResNet for images and Transformer-based encoders for text, ensuring that all methods are compared on fixed representations rather than raw inputs. The 10 CV/NLP datasets are further organized into 74 subclasses to enable fine-grained evaluation. In our experiments, we follow the ADBench protocol, using the provided train/test splits and preprocessing for the tabular benchmarks, and including the CV/NLP feature-based datasets where applicable. Table~\ref{table:datasets} shows a description of all datasets as reported by \citet{han2022adbench}.

\begin{table}[!p]
    \centering
    \caption{Description of all datasets in ADBench}
    \resizebox{0.8\textwidth}{!}{\begin{tabular}{llllll}
    \toprule
        Dataset Name & \# Samples & \# Features & \# Anomaly & \% Anomaly & Category \\ \midrule
        ALOI & 49534 & 27 & 1508 & 3.04 & Image \\ 
        annthyroid & 7200 & 6 & 534 & 7.42 & Healthcare \\ 
        backdoor & 95329 & 196 & 2329 & 2.44 & Network \\ 
        breastw & 683 & 9 & 239 & 34.99 & Healthcare \\ 
        campaign & 41188 & 62 & 4640 & 11.27 & Finance \\ 
        cardio & 1831 & 21 & 176 & 9.61 & Healthcare \\ 
        Cardiotocography & 2114 & 21 & 466 & 22.04 & Healthcare \\ 
        celeba & 202599 & 39 & 4547 & 2.24 & Image \\ 
        census & 299285 & 500 & 18568 & 6.20 & Sociology \\ 
        cover & 286048 & 10 & 2747 & 0.96 & Botany \\ 
        donors & 619326 & 10 & 36710 & 5.93 & Sociology \\ 
        fault & 1941 & 27 & 673 & 34.67 & Physical \\ 
        fraud & 284807 & 29 & 492 & 0.17 & Finance \\ 
        glass & 214 & 7 & 9 & 4.21 & Forensic \\ 
        Hepatitis & 80 & 19 & 13 & 16.25 & Healthcare \\ 
        http & 567498 & 3 & 2211 & 0.39 & Web \\ 
        InternetAds & 1966 & 1555 & 368 & 18.72 & Image \\ 
        Ionosphere & 351 & 32 & 126 & 35.90 & Oryctognosy \\ 
        landsat & 6435 & 36 & 1333 & 20.71 & Astronautics \\ 
        letter & 1600 & 32 & 100 & 6.25 & Image \\ 
        Lymphography & 148 & 18 & 6 & 4.05 & Healthcare \\ 
        magic.gamma & 19020 & 10 & 6688 & 35.16 & Physical \\ 
        mammography & 11183 & 6 & 260 & 2.32 & Healthcare \\ 
        mnist & 7603 & 100 & 700 & 9.21 & Image \\ 
        musk & 3062 & 166 & 97 & 3.17 & Chemistry \\ 
        optdigits & 5216 & 64 & 150 & 2.88 & Image \\ 
        PageBlocks & 5393 & 10 & 510 & 9.46 & Document \\ 
        pendigits & 6870 & 16 & 156 & 2.27 & Image \\ 
        Pima & 768 & 8 & 268 & 34.90 & Healthcare \\ 
        satellite & 6435 & 36 & 2036 & 31.64 & Astronautics \\ 
        satimage-2 & 5803 & 36 & 71 & 1.22 & Astronautics \\ 
        shuttle & 49097 & 9 & 3511 & 7.15 & Astronautics \\ 
        skin & 245057 & 3 & 50859 & 20.75 & Image \\ 
        smtp & 95156 & 3 & 30 & 0.03 & Web \\ 
        SpamBase & 4207 & 57 & 1679 & 39.91 & Document \\ 
        speech & 3686 & 400 & 61 & 1.65 & Linguistics \\ 
        Stamps & 340 & 9 & 31 & 9.12 & Document \\ 
        thyroid & 3772 & 6 & 93 & 2.47 & Healthcare \\ 
        vertebral & 240 & 6 & 30 & 12.50 & Biology \\ 
        vowels & 1456 & 12 & 50 & 3.43 & Linguistics \\ 
        Waveform & 3443 & 21 & 100 & 2.90 & Physics \\ 
        WBC & 223 & 9 & 10 & 4.48 & Healthcare \\ 
        WDBC & 367 & 30 & 10 & 2.72 & Healthcare \\ 
        Wilt & 4819 & 5 & 257 & 5.33 & Botany \\ 
        wine & 129 & 13 & 10 & 7.75 & Chemistry \\ 
        WPBC & 198 & 33 & 47 & 23.74 & Healthcare \\ 
        yeast & 1484 & 8 & 507 & 34.16 & Biology \\ 
        CIFAR10 & 5263 & 512 & 263 & 5.00 & Image \\ 
        FashionMNIST & 6315 & 512 & 315 & 5.00 & Image \\ 
        MNIST-C & 10000 & 512 & 500 & 5.00 & Image \\ 
        MVTec-AD & 5354 & 512 & 1258 & 23.50 & Image \\ 
        SVHN & 5208 & 512 & 260 & 5.00 & Image \\ 
        Agnews & 10000 & 768 & 500 & 5.00 & NLP \\ 
        Amazon & 10000 & 768 & 500 & 5.00 & NLP \\ 
        Imdb & 10000 & 768 & 500 & 5.00 & NLP \\ 
        Yelp & 10000 & 768 & 500 & 5.00 & NLP \\ 
        20newsgroups & 11905 & 768 & 591 & 4.96 & NLP \\ \bottomrule
    \end{tabular}}
    \label{table:datasets}
\end{table}
\subsection{Algorithm}
\label{sec:algorithm}
\begin{algorithm}[H]
\caption{Training for Kurtosis-Guided Single-Scale DSM}
\hspace*{\algorithmicindent}\textbf{Parameters: } $\lambda$ learning rate,\; $\sigma_{\mathrm{base}}>0$ base noise scale

\hspace*{\algorithmicindent}\textbf{Input: } training data $\mathcal{D}=\{x^{(i)}\}_{i=1}^n \subset \mathbb{R}^d$

\hspace*{\algorithmicindent}\textbf{Output: } learned parameters $\theta$,\; per-feature noise vector $\boldsymbol{\sigma}\in\mathbb{R}^d_{>0}$
\begin{algorithmic}[1]
\State $\theta \gets \theta_0$ \Comment initialize network weights
\State $\mu_j,\, s_j \gets \mathrm{mean\_std}(\mathcal{D})$ for $j{=}1,\dots,d$ \Comment feature-wise mean and std
\State $z^{(i)}_j \gets (x^{(i)}_j-\mu_j)/s_j$ for all $i,j$ \Comment standardise features
\State rearrange each feature histogram so that the mode is centred (see Section~\ref{sec:histo}) \Comment compute $\kappa_j$ on rearranged histogram ($\sigma_j$ estimation only; $x$ itself is unchanged downstream)
\State
$k_j \gets \frac{\frac{1}{n}\sum_{i=1}^n \big(z^{(i)}_j\big)^{4}}{\left(\frac{1}{n}\sum_{i=1}^n \big(z^{(i)}_j\big)^{2}\right)^{2}}$ \Comment Pearson kurtosis per feature (using standardised $z^{(i)}_j$)
\State $\sigma_j \gets \sigma_{\mathrm{base}}\cdot \bigl(1 + c\,(k_j - 3)\bigr) $,\quad $c=0.33$ \Comment CF-form formula (Eq.~\ref{eq:sigma_formula}); Gaussian anchor $k_j{=}3$
\State $\boldsymbol{\sigma} \gets \mathrm{clip}(\sigma_1,\dots,\sigma_d;\;0.1,\,2.0)$ \Comment stack and clip feature-wise scales; Notation: $\odot$ = elementwise product\
\For{\textbf{each} epoch}
  \For{\textbf{each} $x \in \mathcal{D}$}
    \State $\varepsilon \sim \mathcal{N}(0,I_d)$ \Comment sample standard Gaussian noise
    \State $z \gets \boldsymbol{\sigma}\odot \varepsilon$ \Comment draw feature-wise noise
    \State $\tilde{x} \gets x + z$ \Comment form noisy input
    \State $\hat{s} \gets f_\theta(\tilde{x})$ \Comment predict score at perturbed point
    \State
      $\mathcal{L} \;\gets\; \tfrac{1}{2}\,\big\|\hat{s} \;+\; \varepsilon\big\|_2^{2}$
\Comment DSM loss with $\varepsilon$-reparameterisation: regress on unit-variance noise rather than score (see Appendix~\ref{sec:model_architecture})
    \State $\theta \gets \theta - \lambda \nabla_\theta \mathcal{L}$ \Comment gradient step
  \EndFor
\EndFor
\State $a(x) \gets \|f_\theta(x)\|_2$ \Comment anomaly score at test time (score norm)
\end{algorithmic}
\end{algorithm}

\subsection{Model architecture}
\label{sec:model_architecture}
We used a ResNet-like tabular model as backbone for K-DSM, DSM, DTE and MSM \citep{gorishniy2021revisiting}. We detail the used parameters for these models in Table~\ref{tab:ddpm_hyperparameters}. For DTE, this did not change the performance significantly compared to the original MLP implementation. MSM was originally made of image data, so we had to adapt it with a tabular backbone.

\begin{table}[h]
\centering
\caption{Hyperparameters for Kurtosis-DSM and DSM models. We use the same hyperparameters for DTE and MSM for the backbone.}
\label{tab:ddpm_hyperparameters}
\begin{tabular}{ll}
\toprule
\textbf{Hyperparameter} & \textbf{Value} \\
\midrule
Number of blocks & 6 \\
Main layer size & 512 \\
Hidden layer size & 512 \\
Optimizer & Adam \\
Learning rate & 0.0005 \\
Dropout layer 1 & 0.2 \\
Dropout layer 2 & 0.1 \\
Batch size & 128 \\
Number of epochs & 500 \\
Base sigma & 0.5\\
\midrule
c & 0.33 \\
Sigma clipping range & $[0.1,2.0]$\\
\bottomrule
\end{tabular}
\end{table}


\subsection{Loss parameterisation} 
\label{ap:loss}

We train with the $\varepsilon$-prediction objective, the per-feature analogue of the standard diffusion-model parameterisation of \citet{ho2020denoisingdiffusionprobabilisticmodels}. The motivation is a subtle pitfall that arises when the $\sigma_j$ vary across features: the score target $-\Sigma^{-1}(\tilde{x}-x)$ has expected squared magnitude $1/\sigma_j^2$ per feature, so the per-feature gradient signal scales as $1/\sigma_j^2$, meaning features with small $\sigma_j$ dominate optimisation. Each $\sigma_j$ therefore plays two distinct roles simultaneously: a \emph{coverage role}, determining how far perturbations reach into low-density regions around feature $j$, and a \emph{loss-weighting role}, setting the relative importance of feature $j$ in the training objective. Our kurtosis rule is a statement about coverage only, so to let it act on the channel it was designed for, we decouple the two roles by rescaling. Starting from the standard diagonal-covariance DSM loss
\begin{equation}
\mathcal{L}_{\text{DSM}}(\theta) =
\tfrac{1}{2}\,\mathbb{E}_{x,\tilde{x}}
\Bigl[\bigl\| s_\theta(\tilde{x}) + \Sigma^{-1}(\tilde{x} - x)\bigr\|_2^2\Bigr],
\label{eq:dsm_standard}
\end{equation}
we rescale each per-feature residual by $\sigma_j$
\begin{equation}
\mathcal{L}_{\text{DSM}}(\theta) =
\tfrac{1}{2}\,\mathbb{E}_{x,\tilde{x}}
\Bigl[\bigl\| s_\theta(\tilde{x}) + \Sigma^{-1/2}(\tilde{x} - x)\bigr\|_2^2\Bigr],
\label{eq:dsm_rescaled}
\end{equation}
equivalently regressing on the unit-variance noise $\varepsilon = \Sigma^{-1/2}(\tilde{x}-x)$:
\begin{equation}
\mathcal{L}_{\text{DSM}}^{\varepsilon}(\theta) =
\tfrac{1}{2}\,\mathbb{E}_{x,\varepsilon}
\Bigl[\bigl\| s_\theta(x + \Sigma^{1/2}\varepsilon) + \varepsilon\bigr\|_2^2\Bigr].
\label{eq:dsm_epsilon}
\end{equation}
This rescaling ensures that every feature contributes equal gradient magnitude in expectation regardless of its assigned $\sigma_j$, so that the kurtosis rule acts only on coverage and not on implicit feature reweighting. The anomaly score $A(x) = \|s_\theta(x)\|_2$ evaluated on clean test inputs is unchanged, since no noise is added at test time and only the function the network is trained to approximate differs. For all other baseline models, hyperparameters were maintained at the values reported in the respective original publications.

\clearpage

\section{Inference and Training Time}
\label{sec:unsup_inf_time}

\begin{figure}[H]
    \centering
        \centering
        \begin{subfigure}[b]{0.48\textwidth}
          \includegraphics[width=\textwidth]{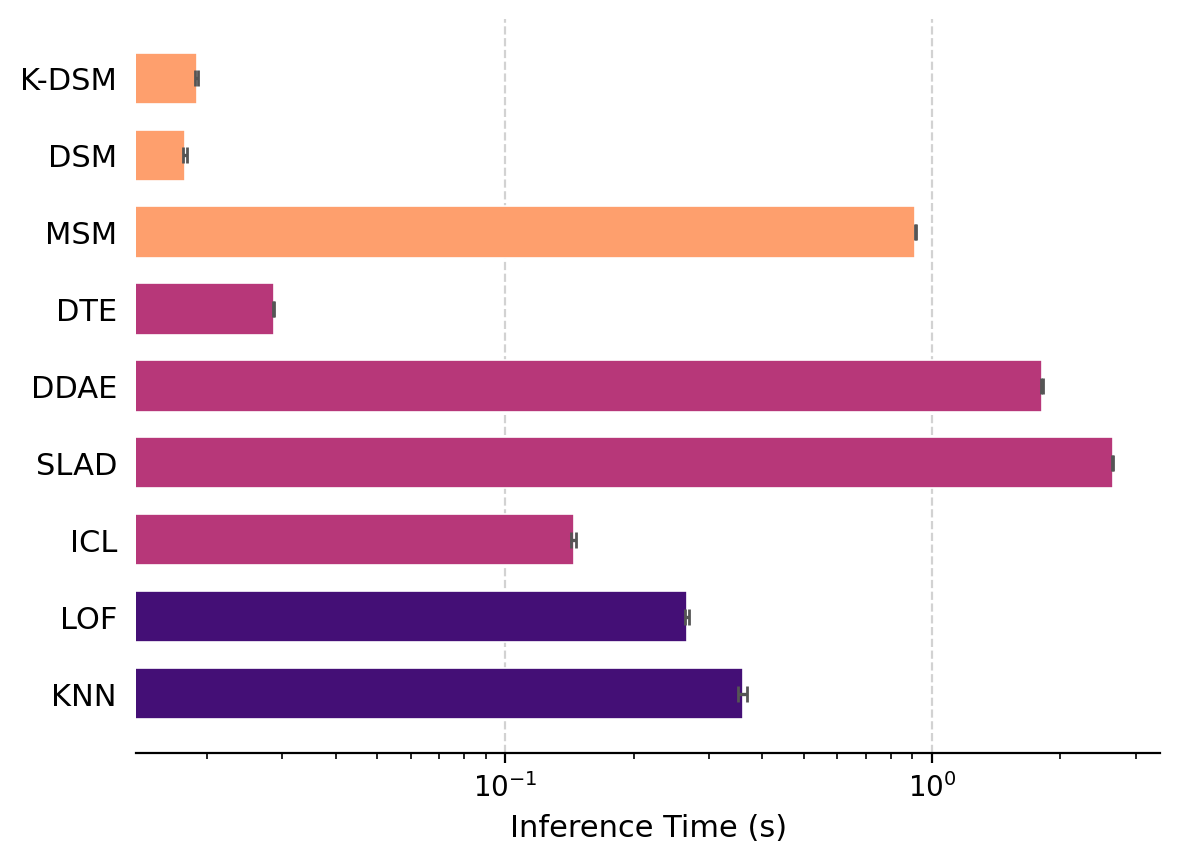}
          \caption{Inference Time}
        \end{subfigure}
        \begin{subfigure}[b]{0.48\textwidth}
          \includegraphics[width=\textwidth]{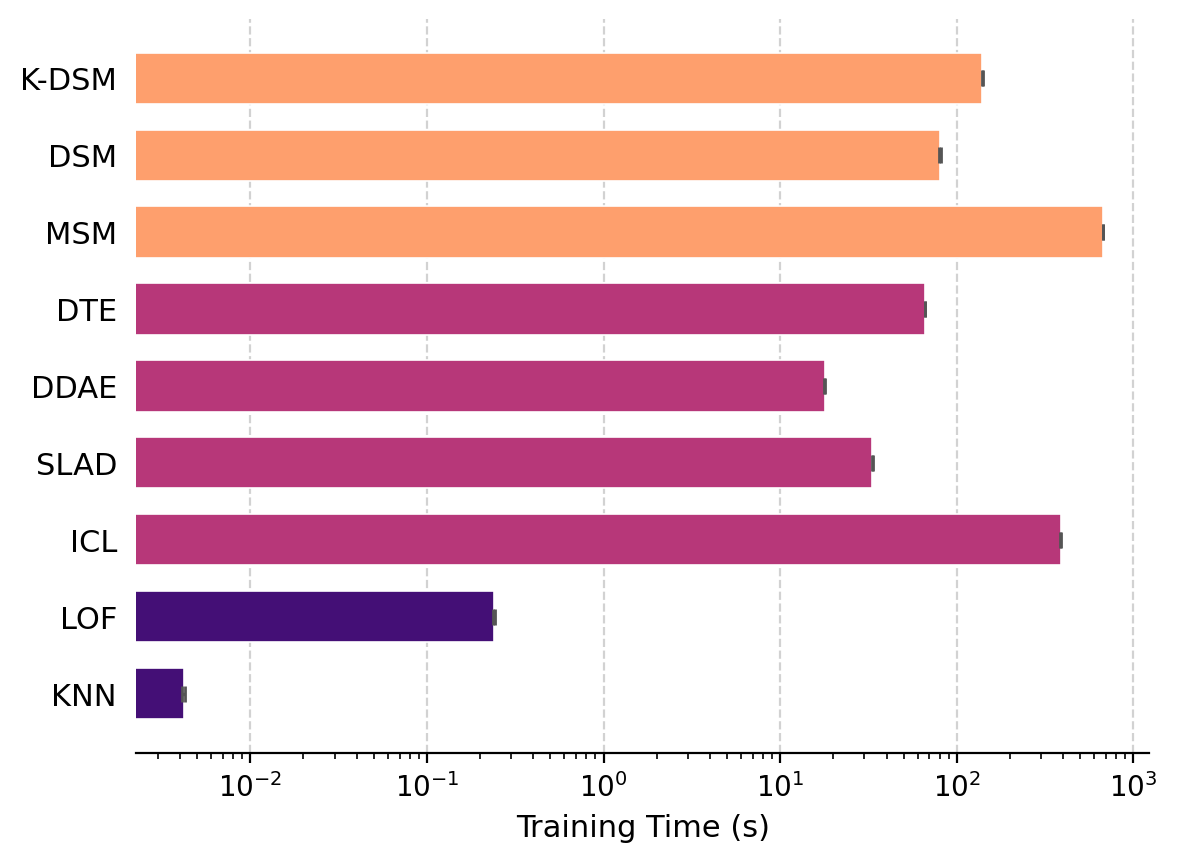}
          \caption{Training Time}
        \end{subfigure}
        \caption{Average inference and training time on the 57 datasets from ADBench (semi-supervised setting). Colour scheme: orange (score-based), magenta (deep learning methods), purple (classical methods).}
        \label{fig:inf_time}
\end{figure}

\begin{figure}[H]
    \centering
        \begin{subfigure}[b]{0.48\textwidth}
          \includegraphics[width=\textwidth]{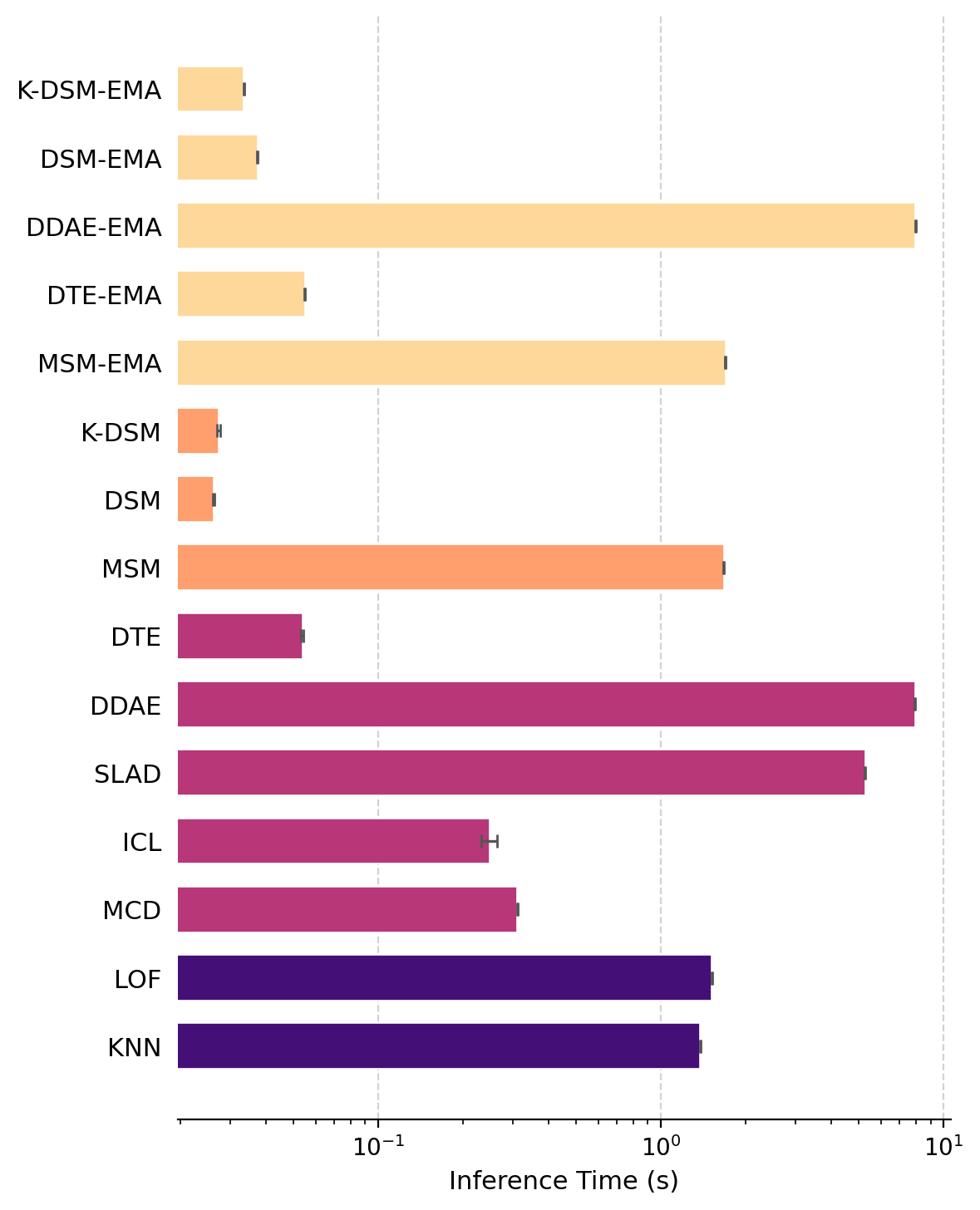}
          \caption{Inference Time}
        \end{subfigure}
        \begin{subfigure}[b]{0.48\textwidth}
          \includegraphics[width=\textwidth]{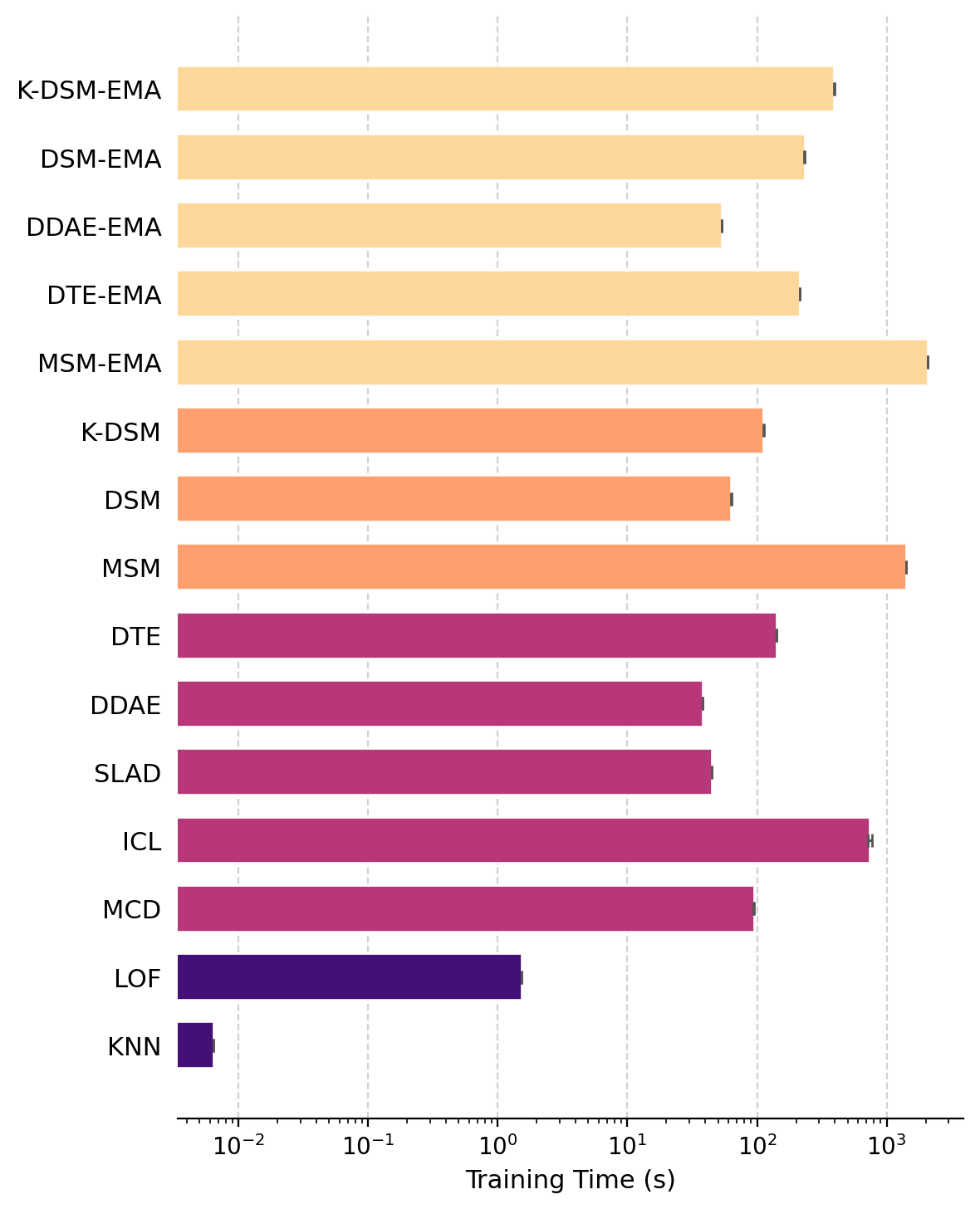}
          \caption{Training Time}
        \end{subfigure}
        \caption{Average inference and training time on the 57 datasets from ADBench (unsupervised setting). EMA-teacher variants are labelled with the ``-EMA'' suffix. Colour scheme as in Figure~\ref{fig:inf_time}.}
        \label{fig:unsup_inf_time}
\end{figure}

\paragraph{Inference time.}
DSM and K-DSM are the fastest methods at inference in both settings, requiring only a single forward pass through an MLP to compute the score norm $\|s_\theta(x)\|_2$; this takes roughly $0.02$--$0.03$ seconds per dataset on average. DDAE is the slowest ($\sim$8 s), as evaluating its denoising objective requires a complete diffusion forward pass over multiple noise levels. MSM ($\sim$1.7 s) is the next slowest because inference requires computing the score at each of its noise scales before aggregation. SLAD is similarly slow due to its attention-based architecture. DTE is also MLP-based and performs a single forward pass, making it structurally similar to DSM; the modest inference gap is likely attributable to implementation overhead from its time-bin embedding rather than a fundamental algorithmic difference. LOF and KNN, while trivially fast to fit, require a full scan over the training set at inference ($O(n)$ per query), which places them in the middle of the field. The EMA teacher adds a second forward pass during \emph{training} but leaves \emph{inference} unchanged, so K-DSM-EMA and DSM-EMA incur no inference penalty over their base counterparts.

\paragraph{Practical importance of inference speed.}
Low inference latency is critical in operational anomaly detection settings such as real-time fraud detection, network intrusion monitoring, and industrial quality control, where decisions must be made on individual samples within milliseconds and throughput can reach millions of queries per second. In these regimes, the $O(1)$ per-query cost of a single MLP forward pass gives DSM and K-DSM a decisive practical advantage over methods whose inference cost grows with the training set size. KNN exemplifies this scaling failure: its inference requires computing the distance from the query to every training point, so latency grows linearly in $n$ and becomes prohibitive once the training set reaches tens or hundreds of thousands of samples. The ADBench datasets used in this benchmark are small enough that KNN's inference time appears modest in the plots ($\sim$1.4 s per dataset averaged over 57 datasets), but this masks the scaling problem entirely. LOF shares the same $O(n)$ inference bottleneck. Score-based methods such as DSM and K-DSM avoid this by amortising all data-dependent computation into the trained weights, leaving inference as a constant-time operation regardless of how large the training set is.

\paragraph{Training time.}
Training times are broadly comparable across the deep-learning methods ($\sim$30--400 s per dataset for the score-based family in the unsupervised setting), with two clear outliers at the extremes. LOF ($\sim$1.5 s) and KNN ($<$0.01 s) are essentially instantaneous, since they require no iterative optimisation. MSM and MSM-EMA are the most expensive ($\sim$1,400 s and $\sim$2,070 s respectively) because the multi-scale score evaluation multiplies both forward and backward pass cost by the number of noise levels at every training step. K-DSM-EMA ($\sim$393 s) is about $3.5\times$ slower than base K-DSM ($\sim$112 s), and DSM-EMA ($\sim$233 s) is about $3.7\times$ slower than base DSM ($\sim$63 s); this overhead comes from the additional teacher forward pass and the batch-filtering bookkeeping performed at each gradient step.

\subsection{Compute Cost}
\label{app:compute_cost}

All experiments were run on a custom cluster using NVIDIA L40S GPUs, with one GPU allocated per random seed and four CPU cores and 24 GB of RAM per worker. Since all five seeds for a given dataset run simultaneously, the wall time scales with the single-seed training time multiplied by the number of datasets.

\paragraph{Main ADBench experiments.} We estimate GPU-hours as $\bar{t} \times 57 \times 5 / 3600$, where $\bar{t}$ is the mean training time (in seconds) reported across the 57 datasets. For the semi-supervised setting this yields approximately 11 GPU-hours for K-DSM, 6 for DSM, 53 for MSM, 5 for DTE, 1 for DDAE, 3 for SLAD, and 31 for ICL, for a total of roughly \textbf{110 GPU-hours}. The unsupervised setting requires running both base and EMA-teacher variants of each score-based method: the dominant costs are MSM-EMA ($\sim$164 GPU-hours), MSM ($\sim$111 GPU-hours), and ICL ($\sim$59 GPU-hours), with the remaining methods (K-DSM-EMA, DSM-EMA, DTE, DTE-EMA, K-DSM, DSM, DDAE, DDAE-EMA) contributing a further $\sim$100 GPU-hours, for a total of roughly \textbf{430 GPU-hours}. Classical methods (MCD, LOF, KNN) run on CPU and their cost is negligible by comparison.

\paragraph{Ablation studies.} Three ablation sweeps were run. The base-noise-scale ($\sigma_0$) ablation and the Cornish--Fisher versus GGD comparison are semi-supervised, each running K-DSM across the full benchmark for a small number of parameter values; scaling from the per-dataset K-DSM timing gives approximately 50 and 45 GPU-hours respectively. The filtering-percentile ($\gamma$) ablation sweeps six values of K-DSM-EMA in the unsupervised setting, contributing approximately $6 \times 31 \approx 185$ GPU-hours. Ablations therefore account for roughly \textbf{280 GPU-hours} in total.

\paragraph{Image benchmarks.} VisA and MVTec used far fewer samples per category than the ADBench tabular datasets, so the training cost is modest. Scaling from the per-category timing data, the full method sweep across VisA (12 categories) and MVTec (15 categories) at five seeds each amounts to approximately 5.3 and 2.1 GPU-hours respectively. DINOv3 feature extraction was run once on a single L40S and took roughly 2 hours. The image benchmarks therefore account for approximately \textbf{9 GPU-hours} in total.

\paragraph{Development and exploration.} Prior to converging on the final method design, we explored a range of approaches to the semi-supervised noise-scale problem, including entropy-based and IQR-based scaling (Appendix~\ref{app:alternative_strategies}), various loss formulations, and different MLP architectures. We estimate this exploratory phase at roughly \textbf{200 GPU-hours}.

\paragraph{Total.} Summing the above, the estimated total compute for the paper is approximately $110 + 430 + 280 + 9 + 200 \approx \mathbf{1{,}030}$ \textbf{GPU-hours} on NVIDIA L40S hardware.

\clearpage
\section{EMA Teacher for Unsupervised Anomaly Detection}
\label{app:ema_teacher}

This appendix expands the description of the EMA-teacher filtering rule introduced in Section~\ref{sec:unsup_main}. The method is a training-loop companion that can be wrapped around any score- or reconstruction-based detector to improve robustness under contamination. We describe the score-based instantiation that we use for DSM and K-DSM, then summarise the analogous variants for DDAE, DTE, and MSM.

\paragraph{Motivation.} Denoising score matching regresses the learned score $s_\theta(\tilde x)$ toward the noise residual $-(\tilde x - x)/\sigma^2$ (or its per-feature analogue) for every training sample. Under the semi-supervised assumption that training data are clean, this is exactly the density-gradient estimator we want: the score of the perturbed data distribution. Under contamination the same objective treats every training point, including contaminants, as ``normal'', so the score at anomalies is explicitly pulled toward zero during training. At test time, the anomaly signal $\|s_\theta(x)\|$ is therefore suppressed at the very points it is supposed to highlight, which is why density-based detectors degrade disproportionately under contamination while distance-based detectors such as KNN or Isolation Forest degrade more gracefully.

\paragraph{EMA teacher.} We maintain a shadow copy $s_{\theta_{\text{ema}}}$ of the score network whose weights follow an exponential moving average of the online model, $\theta_{\text{ema}} \leftarrow \rho\,\theta_{\text{ema}} + (1-\rho)\,\theta$ with decay $\rho = 0.999$. For each minibatch, we score every sample with the teacher and compute its score norm $\|s_{\theta_{\text{ema}}}(x)\|$. Samples whose score norm exceeds the $\gamma$-th batch percentile are flagged as probable contaminants and dropped. The online model takes its gradient step only on the surviving $\gamma\%$ of the batch; the EMA is updated after the step. We set $\gamma = 80$ based on an ablation study (Appendix~\ref{app:gamma_ablation}). The teacher is intentionally slow: using the online model directly to filter its own batch creates a positive-feedback loop where whatever the model currently thinks is anomalous becomes more anomalous on the next step, while a lag-smoothed copy provides a stable density estimate that cannot be warped by the current gradient step. This is the same decoupling principle used in Temporal Ensembling \citep{laine2017temporal} and Mean Teacher \citep{tarvainen2017mean} for semi-supervised classification, and later in momentum-encoder self-supervised representation learning such as MoCo \citep{he2020moco}, BYOL \citep{grill2020byol}, and DINO \citep{caron2021dino}; weight averaging itself dates back to Polyak--Ruppert iterate averaging \citep{polyak1992acceleration}. Our contribution here is not the EMA mechanism but its repurposing as a \emph{sample-selection} signal for contamination-robust density estimation: prior work uses the teacher to produce pseudo-labels or representation targets, whereas we use the teacher's score norm as a per-sample density proxy and train only on the head of the batch.

\paragraph{Variants for other detectors.} The same filtering template works for the baselines, with the teacher's score replaced by a model-appropriate density proxy:
\begin{itemize}[leftmargin=1.25em,nosep]
\item \textbf{DDAE-EMA}: reconstruction error $\|x_0 - \hat x_0^{\text{ema}}(x_t, t)\|$ at the sampled diffusion time, on the grounds that contaminants produce larger reconstruction error at any $t$ than the inlier core.
\item \textbf{DTE-EMA}: expected predicted diffusion-time bin $\sum_k k\cdot\text{softmax}(s_{\theta_{\text{ema}}}(x_t))_k$; contaminants are predicted to be ``further along'' the diffusion than they are.
\item \textbf{MSM-EMA}: score norm at the smallest training $\sigma$ (the most density-sensitive scale in MSM's schedule).
\end{itemize}
In all cases the filter cutoff is the $\gamma$-th batch percentile (with $\gamma = 80$ selected by ablation; see Appendix~\ref{app:gamma_ablation}) and the EMA decay is 0.999.

\begin{algorithm}[H]
\caption{K-DSM training with EMA-teacher filtering}
\label{alg:kdsm-ema}
\begin{algorithmic}[1]
\Require clean loader $\mathcal D$, online network $s_\theta$, per-feature noise vector $\boldsymbol\sigma$ from Algorithm~\ref{alg:rearrangement_new}, EMA decay $\rho$, filtering percentile $\gamma \in (0, 100)$
\State $\theta_{\text{ema}} \gets \theta$ \Comment{initialise teacher from a copy of the student}
\For{each epoch}
  \For{each minibatch $\{x_i\}$ of size $B$ from $\mathcal D$}
    \State $n_i \gets \|s_{\theta_{\text{ema}}}(x_i)\|_2$ for $i=1,\dots,B$ \Comment{teacher density proxy, no grad}
    \State $t \gets \mathrm{quantile}(\{n_i\}, \gamma/100)$
    \State $\mathcal B \gets \{\,i : n_i \le t\,\}$ \Comment{drop top-$(100-\gamma)\%$ of batch}
    \State $\varepsilon_i \sim \mathcal N(0, I_d)$ for $i \in \mathcal B$
    \State $\tilde x_i \gets x_i + \boldsymbol\sigma \odot \varepsilon_i$
    \State $\mathcal L \gets \tfrac{1}{2|\mathcal B|}\sum_{i\in\mathcal B} \|s_\theta(\tilde x_i) + \varepsilon_i\|_2^2$
    \State $\theta \gets \theta - \lambda \nabla_\theta \mathcal L$
    \State $\theta_{\text{ema}} \gets \rho\,\theta_{\text{ema}} + (1-\rho)\,\theta$ \Comment{EMA update after the step}
  \EndFor
\EndFor
\end{algorithmic}
\end{algorithm}

\paragraph{Design notes.} Three design choices matter in practice. First, the teacher is updated \emph{after} the student's step, not before, so the filter signal used in step $k$ is based purely on state from step $k{-}1$ and earlier. Second, the filter acts batchwise on percentiles, not on absolute thresholds, so it is invariant to the overall scale of the score norms at any point during training. Third, if fewer than two samples survive the filter (an edge case on very small batches), we fall back to the whole batch rather than propagate an empty gradient.

\paragraph{Connection to contamination rate.} The $\gamma$-th-percentile cutoff implicitly assumes that no more than $(100-\gamma)\%$ of a batch is contaminated. With $\gamma = 80$, this means the method tolerates up to $20\%$ contamination per batch. In our ADBench experiments the mean contamination across 57 datasets is $11.9\%$, so this is a conservative choice. On datasets with much higher contamination (e.g.\ SpamBase, Pima, magic.gamma, where anomaly rates exceed $30\%$), the fixed cutoff under-filters and the method's gains are correspondingly smaller. An adaptive cutoff tied to an estimate of the contamination rate is a natural extension.

\subsection{Filtering Percentile Ablation}
\label{app:gamma_ablation}

To select the filtering percentile $\gamma$, we sweep $\gamma \in \{70, 75, 80, 85, 90, 95\}$ on K-DSM-EMA and report mean AUC-PR and AUC-ROC averaged across all 57 ADBench datasets (Figure~\ref{fig:gamma_ablation}). Performance is relatively stable in the range $\gamma \in [75, 85]$, but $\gamma = 80$ yields the best mean AUC-PR and competitive AUC-ROC across seeds. Values below $75$ over-filter, discarding a substantial fraction of clean samples and starving the gradient updates; values above $85$ under-filter and allow enough contaminants through to partially suppress the score norm at anomalous points. We therefore adopt $\gamma = 80$ as the default for all experiments in the paper. Consistent with the contamination-rate analysis above, this choice conservatively tolerates up to $20\%$ batch contamination, covering the large majority of ADBench datasets.

\begin{figure}[h]
\centering
\begin{subfigure}[b]{0.48\textwidth}
    \centering
    \includegraphics[width=\linewidth]{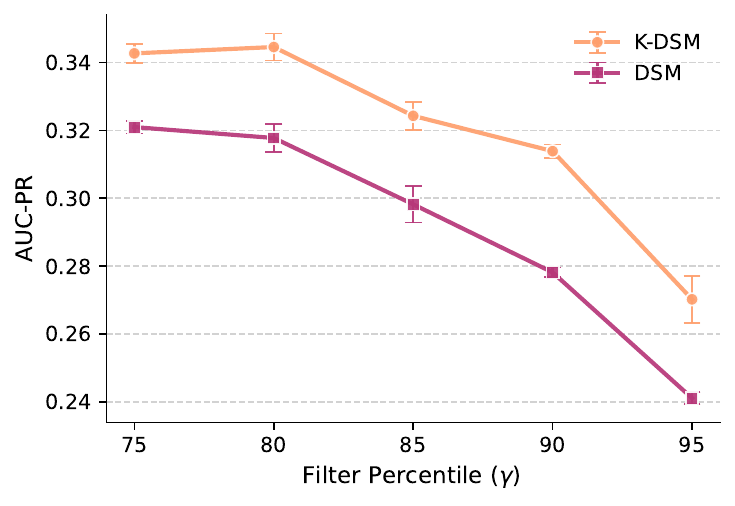}
    \caption{Mean AUC-PR vs.\ filtering percentile $\gamma$.}
    \label{fig:gamma_ablation_pr}
\end{subfigure}
\hfill
\begin{subfigure}[b]{0.48\textwidth}
    \centering
    \includegraphics[width=\linewidth]{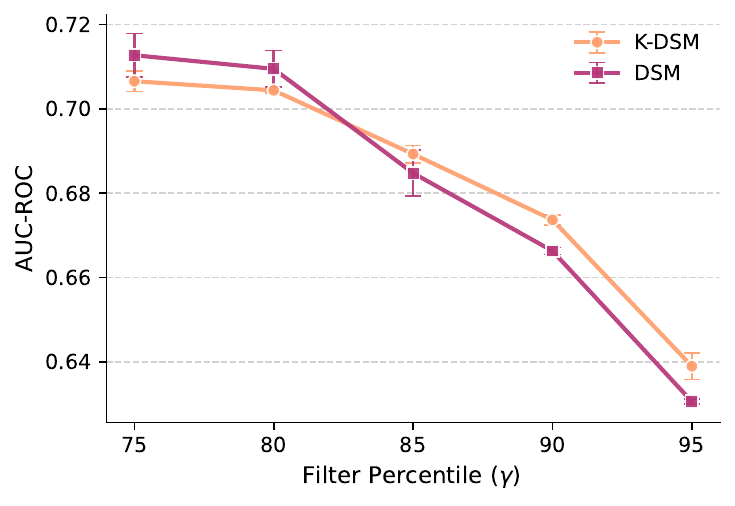}
    \caption{Mean AUC-ROC vs.\ filtering percentile $\gamma$.}
    \label{fig:gamma_ablation_roc}
\end{subfigure}
\caption{Ablation over the filtering percentile $\gamma$ for K-DSM-EMA on 57 ADBench datasets (unsupervised setting, 5 seeds). Shaded bands show $\pm 1$ standard error. Performance is robust across a broad range; $\gamma = 80$ (vertical dashed line) achieves the best AUC-PR and is used in all reported experiments.}
\label{fig:gamma_ablation}
\end{figure}

\clearpage
\section{Alternative Noise Selection Strategies}
\label{app:alternative_strategies}

We investigated Shannon entropy, variance of log density and the inter-quartile range (IQR) for selecting the per-feature noise scale. The performance is reported in Table ~\ref{tab:noise_selection_comparison}:

\begin{table}[htbp]
\centering
\caption{Comparison of noise-scale selection strategies on ADBench (57 datasets, semi-supervised setting). K-DSM uses kurtosis-guided per-feature scaling. DSM uses a single global noise scale. Entropy, variance of log-density, and IQR are alternative marginal-statistic-based scaling rules described in Appendix~\ref{app:alternative_strategies}.}
\label{tab:noise_selection_comparison}
\setlength{\tabcolsep}{6pt}
\renewcommand{\arraystretch}{1.2}
\begin{tabular}{l c c c c c}
\toprule
Metric & DSM & Entropy & Var (log-density) & IQR & K-DSM \\
\midrule
AUC-ROC & 0.813 & 0.805 & 0.8076 & 0.814 & \textbf{0.823} \\
AUC-PR  & 0.584 & 0.597 & 0.5978 & 0.605 & \textbf{0.627} \\
F1      & 0.574 & 0.578 & 0.5860 & 0.585 & \textbf{0.608} \\
\bottomrule
\end{tabular}
\end{table}
\paragraph{Shared preprocessing and postprocessing.}
All scaling strategies follow a common preprocessing pipeline. Each marginal is first subjected to histogram rearrangement and standardisation, followed by truncation to a bounded interval and discretisation where required (e.g., entropy estimation). Binary and discrete features are handled separately where applicable. All resulting noise scales are finally clipped to a predefined range $[\sigma_{\min}, \sigma_{\max}]$.

\paragraph{Entropy-based scaling.}
We estimate the Shannon entropy of each marginal using a histogram-based approximation. After preprocessing, discretised samples define an empirical distribution, and entropy is computed as $H(X) = -\sum_i p_i \log p_i.$
As a reference, we compute the entropy of a truncated standard Gaussian and use it as a baseline. Noise scales are assigned via a piecewise linear mapping centered at this Gaussian entropy such that lower entropy (more concentrated distributions) correspond to larger noise, and higher entropy to smaller noise.

\paragraph{IQR-based scaling.}
We compute the interquartile range (IQR) of each marginal after preprocessing. The IQR captures the spread of the central $50\%$ of the distribution and is normalised by the Gaussian reference value $\mathrm{IQR}_{\mathcal{N}} \approx 1.349$. Noise scales are then set inversely proportional to this ratio:$\sigma_j \propto \frac{1}{\mathrm{IQR}_j},$
such that more concentrated features receive larger perturbations.

\paragraph{Variance of log-density scaling.}
We evaluate each feature under the standard Gaussian log-density
$\log p(x) = -\frac{1}{2}\log(2\pi) - \frac{1}{2}x^2.$
The variance of these log-density values is computed per feature as a measure of variability under the Gaussian model. Noise scales are then defined via inverse scaling: $\sigma_j \propto \frac{1}{\mathrm{Var}[\log p_j(X)]}.$

\clearpage
\section{Full Unsupervised Results}
\label{app:unsup_full}

We report the full per-dataset results for the unsupervised (contaminated) setting summarised in Section~\ref{sec:unsup_main}. Following the protocol of \citet{DTE}, the training set is obtained by sampling the whole dataset with replacement (bootstrapping), so it contains anomalies at the dataset's native contamination rate; inference is performed on the full dataset. The ``-EMA'' columns correspond to the EMA-teacher variant of Appendix~\ref{app:ema_teacher}. Best per row in bold; standard errors over 5 seeds in parentheses.

\begin{table}[H]
    \centering
    \caption{Average AUC-PR and standard deviations over five seeds for the unsupervised setting on ADBench.}
    \resizebox{\textwidth}{!}{
\begin{tabular}{lccccccccccccccc}
\toprule
 & K-DSM-EMA & DSM-EMA & DDAE-EMA & DTE-EMA & MSM-EMA & K-DSM & DSM & MSM & DTE & DDAE & SLAD & ICL & MCD & LOF & KNN \\
\midrule
ALOI & 0.040 (0.000) & 0.041 (0.001) & 0.048 (0.000) & 0.033 (0.000) & 0.061 (0.002) & 0.040 (0.000) & 0.042 (0.000) & 0.088 (0.001) & 0.033 (0.000) & 0.051 (0.001) & 0.037 (0.000) & 0.039 (0.001) & 0.038 (0.000) & \textbf{0.098 (0.001)} & 0.074 (0.000) \\
annthyroid & 0.192 (0.006) & 0.276 (0.025) & 0.219 (0.012) & \textbf{0.677 (0.007)} & 0.272 (0.010) & 0.130 (0.007) & 0.201 (0.008) & 0.322 (0.013) & 0.650 (0.010) & 0.225 (0.011) & 0.132 (0.014) & 0.222 (0.011) & 0.501 (0.002) & 0.160 (0.004) & 0.229 (0.001) \\
backdoor & 0.217 (0.096) & 0.050 (0.000) & 0.082 (0.002) & 0.064 (0.000) & 0.168 (0.025) & 0.495 (0.010) & 0.546 (0.010) & 0.210 (0.008) & 0.073 (0.004) & 0.104 (0.004) & 0.043 (0.011) & \textbf{0.685 (0.028)} & 0.129 (0.086) & 0.069 (0.005) & 0.068 (0.001) \\
breastw & 0.806 (0.014) & 0.715 (0.009) & 0.920 (0.003) & 0.752 (0.007) & 0.607 (0.029) & 0.600 (0.014) & 0.404 (0.009) & 0.674 (0.023) & 0.736 (0.012) & 0.900 (0.005) & 0.663 (0.015) & 0.830 (0.023) & \textbf{0.960 (0.007)} & 0.292 (0.012) & 0.834 (0.005) \\
campaign & 0.218 (0.008) & 0.307 (0.008) & \textbf{0.393 (0.001)} & 0.334 (0.009) & 0.232 (0.009) & 0.232 (0.009) & 0.200 (0.003) & 0.201 (0.004) & 0.334 (0.006) & 0.282 (0.006) & 0.243 (0.003) & 0.253 (0.003) & 0.311 (0.001) & 0.127 (0.001) & 0.335 (0.001) \\
cardio & 0.324 (0.021) & \textbf{0.518 (0.034)} & 0.353 (0.014) & 0.295 (0.007) & 0.124 (0.003) & 0.127 (0.004) & 0.104 (0.003) & 0.125 (0.003) & 0.291 (0.016) & 0.218 (0.013) & 0.187 (0.010) & 0.124 (0.003) & 0.348 (0.012) & 0.165 (0.008) & 0.305 (0.010) \\
Cardiotocography & 0.344 (0.018) & \textbf{0.398 (0.016)} & 0.298 (0.006) & 0.262 (0.007) & 0.211 (0.003) & 0.227 (0.005) & 0.215 (0.004) & 0.213 (0.005) & 0.270 (0.005) & 0.235 (0.007) & 0.223 (0.015) & 0.192 (0.002) & 0.309 (0.004) & 0.259 (0.004) & 0.277 (0.003) \\
celeba & 0.046 (0.006) & 0.022 (0.000) & 0.042 (0.002) & 0.046 (0.005) & 0.049 (0.003) & 0.027 (0.001) & 0.025 (0.001) & 0.041 (0.003) & 0.047 (0.003) & 0.036 (0.001) & 0.056 (0.001) & 0.046 (0.002) & \textbf{0.088 (0.002)} & 0.034 (0.001) & 0.033 (0.001) \\
census & 0.115 (0.004) & \textbf{0.214 (0.016)} & 0.154 (0.003) & 0.089 (0.003) & 0.071 (0.003) & 0.083 (0.001) & 0.082 (0.001) & 0.069 (0.003) & 0.089 (0.002) & 0.108 (0.001) & 0.103 (0.017) & 0.102 (0.001) & 0.060 (0.007) & 0.058 (0.001) & 0.129 (0.002) \\
cover & \textbf{0.332 (0.082)} & 0.097 (0.055) & 0.028 (0.004) & 0.026 (0.003) & 0.034 (0.001) & 0.012 (0.001) & 0.012 (0.001) & 0.032 (0.003) & 0.033 (0.006) & 0.031 (0.005) & 0.025 (0.004) & 0.035 (0.012) & 0.071 (0.001) & 0.018 (0.001) & 0.034 (0.002) \\
donors & 0.159 (0.012) & 0.172 (0.029) & 0.150 (0.006) & 0.133 (0.017) & 0.094 (0.004) & 0.129 (0.001) & 0.107 (0.003) & 0.127 (0.003) & 0.168 (0.011) & 0.153 (0.008) & 0.079 (0.002) & 0.108 (0.008) & \textbf{0.184 (0.004)} & 0.107 (0.001) & 0.156 (0.001) \\
fault & 0.406 (0.005) & 0.400 (0.007) & 0.506 (0.003) & 0.425 (0.014) & 0.404 (0.007) & 0.459 (0.004) & 0.452 (0.004) & 0.415 (0.007) & 0.430 (0.008) & 0.502 (0.004) & 0.517 (0.006) & 0.471 (0.004) & 0.428 (0.005) & 0.387 (0.005) & \textbf{0.519 (0.003)} \\
fraud & \textbf{0.697 (0.031)} & 0.589 (0.042) & 0.648 (0.027) & 0.458 (0.035) & 0.023 (0.003) & 0.285 (0.081) & 0.134 (0.040) & 0.021 (0.006) & 0.633 (0.031) & 0.155 (0.030) & 0.202 (0.030) & 0.157 (0.018) & 0.514 (0.030) & 0.003 (0.000) & 0.095 (0.009) \\
glass & 0.121 (0.029) & 0.120 (0.022) & 0.147 (0.013) & 0.169 (0.013) & 0.091 (0.010) & 0.069 (0.012) & 0.086 (0.015) & 0.093 (0.026) & 0.193 (0.019) & \textbf{0.195 (0.027)} & 0.147 (0.021) & 0.148 (0.019) & 0.116 (0.008) & 0.169 (0.026) & 0.163 (0.029) \\
Hepatitis & 0.299 (0.040) & 0.311 (0.042) & 0.228 (0.014) & 0.280 (0.029) & 0.201 (0.007) & 0.141 (0.011) & 0.152 (0.011) & 0.197 (0.022) & 0.223 (0.030) & 0.152 (0.005) & 0.163 (0.006) & 0.228 (0.011) & \textbf{0.316 (0.017)} & 0.180 (0.022) & 0.178 (0.008) \\
http & 0.830 (0.012) & 0.415 (0.013) & 0.396 (0.149) & 0.382 (0.013) & 0.041 (0.015) & 0.071 (0.035) & 0.334 (0.090) & 0.021 (0.003) & 0.463 (0.068) & 0.008 (0.002) & 0.381 (0.017) & 0.655 (0.008) & \textbf{0.861 (0.008)} & 0.049 (0.009) & 0.021 (0.004) \\
InternetAds & 0.542 (0.036) & \textbf{0.634 (0.003)} & 0.532 (0.008) & 0.240 (0.029) & 0.325 (0.016) & 0.392 (0.025) & 0.324 (0.008) & 0.292 (0.005) & 0.246 (0.024) & 0.317 (0.007) & 0.292 (0.006) & 0.205 (0.010) & 0.441 (0.043) & 0.319 (0.013) & 0.374 (0.013) \\
Ionosphere & 0.687 (0.013) & 0.757 (0.011) & 0.815 (0.010) & 0.869 (0.011) & 0.493 (0.022) & 0.494 (0.008) & 0.544 (0.006) & 0.522 (0.014) & 0.871 (0.012) & 0.676 (0.012) & 0.808 (0.011) & 0.536 (0.028) & \textbf{0.929 (0.005)} & 0.798 (0.022) & 0.715 (0.013) \\
landsat & 0.193 (0.005) & 0.220 (0.006) & 0.272 (0.003) & 0.210 (0.006) & 0.243 (0.005) & 0.237 (0.002) & 0.246 (0.001) & 0.242 (0.004) & 0.215 (0.003) & 0.275 (0.006) & 0.306 (0.001) & \textbf{0.317 (0.008)} & 0.194 (0.001) & 0.250 (0.002) & 0.250 (0.002) \\
letter & 0.210 (0.011) & 0.271 (0.012) & 0.316 (0.015) & 0.260 (0.009) & 0.087 (0.008) & 0.396 (0.024) & \textbf{0.457 (0.023)} & 0.233 (0.013) & 0.253 (0.008) & 0.299 (0.015) & 0.303 (0.016) & 0.233 (0.012) & 0.206 (0.005) & 0.422 (0.011) & 0.323 (0.011) \\
Lymphography & 0.663 (0.042) & \textbf{0.894 (0.014)} & 0.868 (0.036) & 0.752 (0.080) & 0.052 (0.008) & 0.045 (0.012) & 0.043 (0.009) & 0.045 (0.010) & 0.368 (0.062) & 0.132 (0.033) & 0.690 (0.012) & 0.420 (0.053) & 0.796 (0.044) & 0.123 (0.042) & 0.209 (0.045) \\
magic.gamma & 0.688 (0.002) & 0.679 (0.005) & 0.698 (0.002) & 0.651 (0.011) & 0.622 (0.010) & 0.532 (0.002) & 0.516 (0.002) & 0.637 (0.002) & 0.670 (0.010) & 0.695 (0.004) & 0.515 (0.016) & 0.572 (0.016) & 0.673 (0.000) & 0.519 (0.002) & \textbf{0.729 (0.001)} \\
mammography & 0.107 (0.001) & 0.060 (0.004) & 0.140 (0.012) & \textbf{0.190 (0.010)} & 0.076 (0.007) & 0.144 (0.010) & 0.052 (0.001) & 0.100 (0.008) & 0.173 (0.008) & 0.128 (0.012) & 0.078 (0.008) & 0.067 (0.010) & 0.037 (0.002) & 0.085 (0.003) & 0.158 (0.005) \\
mnist & \textbf{0.400 (0.008)} & 0.322 (0.008) & 0.349 (0.003) & 0.372 (0.010) & 0.291 (0.008) & 0.319 (0.010) & 0.307 (0.006) & 0.210 (0.009) & 0.391 (0.016) & 0.317 (0.007) & 0.148 (0.056) & 0.255 (0.014) & 0.379 (0.009) & 0.226 (0.008) & 0.376 (0.002) \\
musk & \textbf{1.000 (0.000)} & 1.000 (0.000) & \textbf{1.000 (0.000)} & 0.775 (0.114) & 0.112 (0.010) & 0.292 (0.057) & 0.468 (0.035) & 0.109 (0.007) & 0.493 (0.074) & 0.323 (0.040) & 0.277 (0.032) & 0.575 (0.088) & 0.131 (0.022) & 0.125 (0.028) & 0.166 (0.017) \\
optdigits & 0.035 (0.002) & \textbf{0.037 (0.002)} & 0.028 (0.000) & 0.024 (0.001) & 0.027 (0.001) & 0.032 (0.001) & 0.032 (0.000) & 0.034 (0.002) & 0.028 (0.002) & 0.031 (0.001) & 0.030 (0.001) & 0.035 (0.001) & 0.028 (0.001) & 0.034 (0.001) & 0.022 (0.000) \\
PageBlocks & \textbf{0.672 (0.015)} & 0.443 (0.007) & 0.555 (0.018) & 0.585 (0.014) & 0.470 (0.015) & 0.272 (0.005) & 0.299 (0.014) & 0.457 (0.022) & 0.538 (0.014) & 0.516 (0.016) & 0.398 (0.011) & 0.500 (0.007) & 0.636 (0.004) & 0.289 (0.011) & 0.456 (0.006) \\
pendigits & \textbf{0.426 (0.173)} & 0.284 (0.120) & 0.064 (0.003) & 0.048 (0.004) & 0.033 (0.003) & 0.051 (0.004) & 0.048 (0.003) & 0.035 (0.003) & 0.051 (0.006) & 0.059 (0.005) & 0.044 (0.005) & 0.098 (0.009) & 0.054 (0.001) & 0.045 (0.004) & 0.072 (0.002) \\
Pima & 0.466 (0.024) & 0.434 (0.011) & 0.502 (0.009) & 0.432 (0.014) & 0.354 (0.009) & 0.349 (0.007) & 0.370 (0.008) & 0.380 (0.010) & 0.444 (0.005) & \textbf{0.505 (0.009)} & 0.354 (0.015) & 0.420 (0.006) & 0.502 (0.013) & 0.404 (0.008) & 0.504 (0.012) \\
satellite & 0.486 (0.009) & 0.426 (0.014) & 0.518 (0.007) & \textbf{0.578 (0.004)} & 0.364 (0.007) & 0.353 (0.004) & 0.346 (0.003) & 0.388 (0.005) & 0.552 (0.010) & 0.517 (0.007) & 0.533 (0.004) & 0.550 (0.028) & 0.529 (0.002) & 0.375 (0.004) & 0.518 (0.003) \\
satimage-2 & 0.791 (0.157) & \textbf{0.840 (0.046)} & 0.572 (0.083) & 0.164 (0.003) & 0.026 (0.003) & 0.041 (0.011) & 0.031 (0.007) & 0.033 (0.009) & 0.160 (0.012) & 0.179 (0.026) & 0.339 (0.027) & 0.378 (0.101) & 0.372 (0.010) & 0.040 (0.011) & 0.330 (0.022) \\
shuttle & \textbf{0.914 (0.002)} & 0.592 (0.128) & 0.091 (0.001) & 0.627 (0.055) & 0.211 (0.024) & 0.085 (0.001) & 0.116 (0.006) & 0.256 (0.010) & 0.685 (0.022) & 0.096 (0.005) & 0.584 (0.118) & 0.420 (0.029) & 0.867 (0.003) & 0.108 (0.002) & 0.160 (0.001) \\
skin & 0.317 (0.020) & 0.194 (0.008) & 0.394 (0.007) & 0.308 (0.007) & 0.258 (0.004) & 0.186 (0.001) & 0.190 (0.001) & 0.271 (0.005) & 0.322 (0.005) & 0.351 (0.005) & 0.354 (0.006) & 0.215 (0.010) & \textbf{0.490 (0.001)} & 0.221 (0.001) & 0.266 (0.001) \\
smtp & \textbf{0.657 (0.058)} & 0.441 (0.058) & 0.339 (0.038) & 0.379 (0.028) & 0.013 (0.008) & 0.131 (0.052) & 0.415 (0.026) & 0.019 (0.004) & 0.419 (0.017) & 0.292 (0.066) & 0.425 (0.021) & 0.175 (0.037) & 0.006 (0.000) & 0.033 (0.012) & 0.210 (0.045) \\
SpamBase & 0.429 (0.002) & 0.435 (0.006) & 0.450 (0.004) & 0.395 (0.006) & 0.416 (0.004) & 0.453 (0.003) & 0.399 (0.002) & 0.428 (0.004) & 0.401 (0.007) & \textbf{0.457 (0.002)} & 0.384 (0.001) & 0.379 (0.006) & 0.399 (0.006) & 0.358 (0.001) & 0.397 (0.002) \\
speech & \textbf{0.030 (0.005)} & 0.025 (0.001) & 0.024 (0.004) & 0.018 (0.002) & 0.019 (0.001) & 0.023 (0.004) & 0.022 (0.004) & 0.017 (0.001) & 0.019 (0.002) & 0.020 (0.002) & 0.018 (0.001) & 0.019 (0.001) & 0.019 (0.001) & 0.021 (0.001) & 0.020 (0.001) \\
Stamps & 0.264 (0.020) & \textbf{0.289 (0.034)} & 0.253 (0.021) & 0.244 (0.024) & 0.134 (0.015) & 0.089 (0.005) & 0.095 (0.006) & 0.115 (0.009) & 0.212 (0.018) & 0.189 (0.011) & 0.156 (0.020) & 0.141 (0.017) & 0.265 (0.032) & 0.152 (0.019) & 0.180 (0.015) \\
thyroid & 0.336 (0.013) & 0.414 (0.036) & 0.421 (0.028) & \textbf{0.755 (0.007)} & 0.166 (0.023) & 0.068 (0.002) & 0.045 (0.006) & 0.155 (0.008) & 0.705 (0.022) & 0.189 (0.020) & 0.161 (0.014) & 0.168 (0.027) & 0.701 (0.003) & 0.069 (0.007) & 0.261 (0.010) \\
vertebral & 0.103 (0.005) & 0.108 (0.005) & 0.092 (0.006) & 0.121 (0.006) & 0.143 (0.010) & \textbf{0.166 (0.015)} & 0.147 (0.010) & 0.143 (0.014) & 0.112 (0.005) & 0.095 (0.004) & 0.102 (0.008) & 0.112 (0.009) & 0.106 (0.012) & 0.133 (0.015) & 0.112 (0.005) \\
vowels & 0.101 (0.025) & 0.171 (0.011) & 0.475 (0.017) & 0.345 (0.071) & 0.042 (0.002) & 0.173 (0.010) & 0.221 (0.031) & 0.160 (0.017) & 0.375 (0.060) & 0.430 (0.020) & 0.320 (0.009) & 0.239 (0.038) & 0.294 (0.022) & 0.303 (0.032) & \textbf{0.499 (0.031)} \\
Waveform & 0.038 (0.000) & 0.040 (0.001) & 0.052 (0.002) & 0.040 (0.001) & 0.029 (0.000) & 0.036 (0.001) & 0.041 (0.002) & 0.039 (0.006) & 0.043 (0.003) & 0.050 (0.003) & 0.024 (0.000) & \textbf{0.120 (0.011)} & 0.040 (0.000) & 0.071 (0.004) & 0.104 (0.004) \\
WBC & 0.839 (0.041) & 0.843 (0.038) & \textbf{0.852 (0.025)} & 0.235 (0.011) & 0.095 (0.016) & 0.044 (0.007) & 0.048 (0.006) & 0.079 (0.014) & 0.194 (0.006) & 0.202 (0.038) & 0.651 (0.078) & 0.344 (0.033) & 0.814 (0.044) & 0.074 (0.008) & 0.239 (0.046) \\
WDBC & \textbf{0.561 (0.033)} & 0.449 (0.026) & 0.418 (0.036) & 0.159 (0.033) & 0.045 (0.009) & 0.027 (0.004) & 0.031 (0.005) & 0.042 (0.009) & 0.177 (0.034) & 0.206 (0.085) & 0.243 (0.062) & 0.087 (0.008) & 0.446 (0.051) & 0.084 (0.017) & 0.309 (0.040) \\
Wilt & 0.041 (0.001) & 0.051 (0.001) & 0.042 (0.000) & \textbf{0.174 (0.010)} & 0.106 (0.005) & 0.042 (0.000) & 0.067 (0.002) & 0.104 (0.003) & 0.148 (0.003) & 0.043 (0.001) & 0.101 (0.003) & 0.097 (0.004) & 0.153 (0.001) & 0.084 (0.002) & 0.059 (0.001) \\
wine & 0.087 (0.013) & 0.090 (0.009) & 0.074 (0.009) & 0.126 (0.029) & 0.104 (0.015) & 0.123 (0.012) & 0.091 (0.016) & 0.106 (0.010) & 0.106 (0.016) & 0.070 (0.011) & 0.067 (0.005) & 0.085 (0.006) & \textbf{0.663 (0.124)} & 0.068 (0.007) & 0.087 (0.002) \\
WPBC & 0.235 (0.007) & 0.239 (0.014) & 0.226 (0.005) & 0.228 (0.009) & 0.254 (0.015) & 0.255 (0.006) & 0.252 (0.009) & \textbf{0.256 (0.012)} & 0.247 (0.016) & 0.247 (0.010) & 0.235 (0.003) & 0.221 (0.004) & 0.225 (0.006) & 0.222 (0.012) & 0.245 (0.011) \\
yeast & 0.329 (0.004) & 0.307 (0.007) & 0.295 (0.002) & 0.290 (0.002) & 0.322 (0.004) & 0.313 (0.004) & 0.333 (0.005) & 0.316 (0.003) & 0.301 (0.005) & 0.297 (0.003) & \textbf{0.342 (0.003)} & 0.298 (0.005) & 0.297 (0.002) & 0.317 (0.004) & 0.301 (0.005) \\
20news & 0.079 (0.004) & 0.079 (0.002) & 0.067 (0.002) & 0.070 (0.002) & 0.060 (0.001) & 0.075 (0.004) & 0.079 (0.003) & 0.070 (0.003) & 0.067 (0.002) & 0.071 (0.002) & 0.064 (0.002) & 0.061 (0.002) & 0.065 (0.002) & \textbf{0.086 (0.002)} & 0.074 (0.002) \\
agnews & 0.107 (0.000) & 0.096 (0.000) & 0.095 (0.001) & 0.073 (0.003) & 0.048 (0.003) & 0.098 (0.004) & 0.101 (0.002) & 0.041 (0.001) & 0.074 (0.002) & 0.092 (0.002) & 0.063 (0.000) & 0.062 (0.000) & 0.055 (0.001) & \textbf{0.123 (0.001)} & 0.089 (0.000) \\
amazon & 0.060 (0.001) & 0.062 (0.000) & 0.062 (0.001) & 0.058 (0.001) & 0.046 (0.003) & \textbf{0.065 (0.001)} & 0.064 (0.001) & 0.043 (0.002) & 0.057 (0.001) & 0.063 (0.001) & 0.051 (0.000) & 0.052 (0.000) & 0.053 (0.001) & 0.058 (0.001) & 0.063 (0.000) \\
CIFAR10 & 0.095 (0.001) & 0.095 (0.000) & 0.100 (0.000) & 0.095 (0.001) & 0.068 (0.001) & 0.100 (0.001) & 0.088 (0.001) & 0.074 (0.001) & 0.093 (0.001) & 0.097 (0.001) & 0.104 (0.000) & 0.073 (0.001) & 0.087 (0.003) & \textbf{0.115 (0.001)} & 0.105 (0.000) \\
FashionMNIST & \textbf{0.419 (0.010)} & 0.362 (0.002) & 0.358 (0.001) & 0.315 (0.005) & 0.136 (0.007) & 0.315 (0.002) & 0.228 (0.003) & 0.173 (0.003) & 0.266 (0.002) & 0.292 (0.002) & 0.336 (0.003) & 0.185 (0.002) & 0.282 (0.012) & 0.189 (0.001) & 0.326 (0.001) \\
imdb & 0.045 (0.000) & 0.046 (0.000) & 0.047 (0.000) & 0.048 (0.001) & 0.051 (0.002) & 0.046 (0.001) & 0.046 (0.000) & 0.049 (0.001) & 0.046 (0.001) & 0.047 (0.001) & 0.052 (0.000) & \textbf{0.054 (0.000)} & 0.049 (0.001) & 0.049 (0.000) & 0.047 (0.000) \\
MNIST-C & 0.210 (0.015) & \textbf{0.226 (0.001)} & 0.207 (0.002) & 0.163 (0.005) & 0.133 (0.004) & 0.182 (0.001) & 0.199 (0.001) & 0.136 (0.002) & 0.163 (0.001) & 0.191 (0.001) & 0.169 (0.004) & 0.110 (0.002) & 0.164 (0.006) & 0.126 (0.001) & 0.194 (0.000) \\
MVTec & \textbf{0.554 (0.005)} & 0.547 (0.004) & 0.509 (0.002) & 0.534 (0.006) & 0.354 (0.005) & 0.419 (0.007) & 0.386 (0.003) & 0.342 (0.005) & 0.524 (0.007) & 0.419 (0.002) & 0.507 (0.009) & 0.399 (0.004) & 0.508 (0.002) & 0.528 (0.005) & 0.453 (0.004) \\
SVHN & \textbf{0.083 (0.002)} & 0.080 (0.000) & 0.080 (0.000) & 0.078 (0.000) & 0.068 (0.000) & 0.079 (0.000) & 0.076 (0.000) & 0.073 (0.000) & 0.077 (0.001) & 0.079 (0.000) & 0.079 (0.000) & 0.069 (0.001) & 0.068 (0.000) & 0.082 (0.000) & 0.081 (0.000) \\
yelp & 0.084 (0.001) & 0.080 (0.000) & 0.083 (0.001) & 0.064 (0.006) & 0.042 (0.002) & 0.081 (0.003) & 0.083 (0.001) & 0.039 (0.001) & 0.071 (0.004) & 0.082 (0.001) & 0.052 (0.000) & 0.051 (0.000) & 0.056 (0.004) & 0.084 (0.001) & \textbf{0.089 (0.001)} \\
\midrule
Overall & \textbf{0.343 (0.004)} & 0.321 (0.005) & 0.315 (0.004) & 0.290 (0.002) & 0.169 (0.001) & 0.189 (0.003) & 0.194 (0.002) & 0.177 (0.002) & 0.281 (0.002) & 0.228 (0.003) & 0.245 (0.001) & 0.240 (0.003) & 0.321 (0.005) & 0.175 (0.001) & 0.239 (0.001) \\
\bottomrule
\end{tabular}
}
\end{table}

\begin{table}[H]
    \centering
    \caption{Average AUC-ROC and standard deviations over five seeds for the unsupervised setting on ADBench.}
    \resizebox{\textwidth}{!}{
\begin{tabular}{lccccccccccccccc}
\toprule
 & K-DSM-EMA & DSM-EMA & DDAE-EMA & DTE-EMA & MSM-EMA & K-DSM & DSM & MSM & DTE & DDAE & SLAD & ICL & MCD & LOF & KNN \\
\midrule
ALOI & 0.565 (0.001) & 0.603 (0.003) & 0.575 (0.001) & 0.525 (0.001) & 0.642 (0.002) & 0.565 (0.001) & 0.569 (0.001) & 0.642 (0.003) & 0.525 (0.001) & 0.576 (0.001) & 0.550 (0.001) & 0.517 (0.003) & 0.553 (0.000) & \textbf{0.768 (0.001)} & 0.703 (0.001) \\
annthyroid & 0.707 (0.003) & 0.777 (0.008) & 0.736 (0.004) & \textbf{0.959 (0.003)} & 0.848 (0.009) & 0.642 (0.010) & 0.713 (0.006) & 0.884 (0.006) & 0.956 (0.006) & 0.737 (0.004) & 0.579 (0.029) & 0.750 (0.022) & 0.917 (0.001) & 0.704 (0.001) & 0.802 (0.002) \\
backdoor & 0.732 (0.052) & 0.727 (0.007) & 0.813 (0.009) & 0.785 (0.004) & 0.832 (0.029) & 0.851 (0.046) & \textbf{0.924 (0.010)} & 0.821 (0.017) & 0.802 (0.012) & 0.819 (0.005) & 0.620 (0.074) & 0.913 (0.005) & 0.484 (0.120) & 0.683 (0.014) & 0.709 (0.005) \\
breastw & 0.833 (0.010) & 0.775 (0.004) & 0.971 (0.001) & 0.913 (0.004) & 0.825 (0.019) & 0.807 (0.008) & 0.623 (0.004) & 0.846 (0.015) & 0.908 (0.005) & 0.966 (0.001) & 0.806 (0.008) & 0.941 (0.010) & \textbf{0.985 (0.002)} & 0.449 (0.013) & 0.875 (0.007) \\
campaign & 0.706 (0.009) & 0.670 (0.012) & \textbf{0.787 (0.001)} & 0.786 (0.009) & 0.689 (0.003) & 0.720 (0.012) & 0.653 (0.003) & 0.641 (0.003) & 0.785 (0.004) & 0.762 (0.005) & 0.724 (0.003) & 0.731 (0.003) & 0.762 (0.000) & 0.549 (0.002) & 0.787 (0.001) \\
cardio & 0.758 (0.024) & \textbf{0.847 (0.016)} & 0.700 (0.018) & 0.740 (0.010) & 0.574 (0.002) & 0.510 (0.005) & 0.450 (0.010) & 0.571 (0.016) & 0.755 (0.026) & 0.626 (0.011) & 0.499 (0.018) & 0.495 (0.013) & 0.813 (0.009) & 0.546 (0.013) & 0.670 (0.012) \\
Cardiotocography & 0.610 (0.018) & \textbf{0.652 (0.009)} & 0.469 (0.005) & 0.511 (0.017) & 0.470 (0.013) & 0.460 (0.006) & 0.438 (0.005) & 0.470 (0.009) & 0.525 (0.009) & 0.447 (0.012) & 0.383 (0.018) & 0.374 (0.005) & 0.498 (0.002) & 0.530 (0.007) & 0.474 (0.006) \\
celeba & 0.706 (0.035) & 0.501 (0.004) & 0.683 (0.007) & 0.740 (0.025) & 0.729 (0.010) & 0.601 (0.011) & 0.563 (0.006) & 0.704 (0.014) & 0.751 (0.017) & 0.653 (0.007) & 0.775 (0.004) & 0.728 (0.012) & \textbf{0.820 (0.002)} & 0.590 (0.006) & 0.589 (0.005) \\
census & 0.691 (0.006) & \textbf{0.721 (0.011)} & 0.714 (0.002) & 0.612 (0.010) & 0.577 (0.008) & 0.652 (0.006) & 0.640 (0.002) & 0.578 (0.016) & 0.586 (0.007) & 0.696 (0.001) & 0.600 (0.041) & 0.651 (0.002) & 0.507 (0.045) & 0.515 (0.003) & 0.707 (0.002) \\
cover & \textbf{0.947 (0.019)} & 0.777 (0.065) & 0.747 (0.032) & 0.726 (0.030) & 0.864 (0.006) & 0.492 (0.020) & 0.499 (0.025) & 0.801 (0.013) & 0.746 (0.036) & 0.746 (0.032) & 0.742 (0.023) & 0.783 (0.041) & 0.923 (0.002) & 0.569 (0.008) & 0.751 (0.013) \\
donors & 0.791 (0.010) & 0.803 (0.027) & 0.824 (0.007) & 0.775 (0.029) & 0.721 (0.013) & 0.758 (0.003) & 0.752 (0.004) & 0.799 (0.005) & 0.831 (0.016) & 0.822 (0.007) & 0.658 (0.009) & 0.695 (0.023) & \textbf{0.863 (0.004)} & 0.629 (0.006) & 0.732 (0.003) \\
fault & 0.555 (0.006) & 0.553 (0.009) & \textbf{0.705 (0.002)} & 0.590 (0.015) & 0.572 (0.008) & 0.638 (0.006) & 0.630 (0.004) & 0.580 (0.008) & 0.592 (0.014) & 0.698 (0.003) & 0.704 (0.004) & 0.657 (0.004) & 0.604 (0.006) & 0.579 (0.006) & 0.703 (0.002) \\
fraud & 0.950 (0.006) & \textbf{0.961 (0.007)} & 0.958 (0.007) & 0.938 (0.010) & 0.926 (0.007) & 0.918 (0.008) & 0.910 (0.011) & 0.801 (0.010) & 0.947 (0.007) & 0.951 (0.006) & 0.942 (0.007) & 0.951 (0.004) & 0.957 (0.003) & 0.550 (0.033) & 0.946 (0.006) \\
glass & 0.480 (0.056) & 0.655 (0.039) & 0.872 (0.009) & 0.861 (0.012) & 0.581 (0.029) & 0.565 (0.055) & 0.606 (0.029) & 0.550 (0.053) & 0.879 (0.014) & \textbf{0.903 (0.014)} & 0.746 (0.013) & 0.838 (0.018) & 0.799 (0.005) & 0.659 (0.036) & 0.629 (0.046) \\
Hepatitis & 0.616 (0.043) & \textbf{0.664 (0.047)} & 0.555 (0.021) & 0.641 (0.038) & 0.532 (0.026) & 0.357 (0.038) & 0.378 (0.028) & 0.466 (0.028) & 0.580 (0.028) & 0.469 (0.012) & 0.455 (0.016) & 0.642 (0.020) & 0.644 (0.037) & 0.431 (0.046) & 0.475 (0.020) \\
http & 0.999 (0.000) & 0.995 (0.000) & 0.996 (0.001) & 0.994 (0.000) & 0.737 (0.097) & 0.887 (0.039) & 0.987 (0.004) & 0.400 (0.159) & 0.995 (0.001) & 0.379 (0.146) & 0.994 (0.000) & 0.998 (0.000) & \textbf{0.999 (0.000)} & 0.337 (0.005) & 0.096 (0.007) \\
InternetAds & 0.752 (0.018) & \textbf{0.795 (0.002)} & 0.745 (0.004) & 0.538 (0.034) & 0.686 (0.012) & 0.709 (0.015) & 0.664 (0.005) & 0.634 (0.011) & 0.540 (0.022) & 0.659 (0.005) & 0.609 (0.006) & 0.568 (0.020) & 0.714 (0.029) & 0.635 (0.008) & 0.681 (0.007) \\
Ionosphere & 0.670 (0.020) & 0.779 (0.017) & 0.868 (0.003) & 0.906 (0.007) & 0.690 (0.019) & 0.620 (0.005) & 0.724 (0.004) & 0.661 (0.018) & 0.909 (0.008) & 0.812 (0.006) & 0.885 (0.005) & 0.682 (0.033) & \textbf{0.939 (0.006)} & 0.866 (0.009) & 0.656 (0.017) \\
landsat & 0.463 (0.010) & 0.518 (0.009) & 0.625 (0.004) & 0.520 (0.011) & 0.552 (0.008) & 0.547 (0.001) & 0.559 (0.002) & 0.556 (0.006) & 0.522 (0.005) & 0.622 (0.007) & \textbf{0.674 (0.002)} & 0.617 (0.005) & 0.477 (0.002) & 0.544 (0.007) & 0.579 (0.003) \\
letter & 0.731 (0.011) & 0.769 (0.006) & 0.834 (0.007) & 0.792 (0.004) & 0.633 (0.011) & 0.854 (0.012) & \textbf{0.896 (0.006)} & 0.645 (0.015) & 0.784 (0.009) & 0.826 (0.010) & 0.863 (0.002) & 0.786 (0.006) & 0.797 (0.003) & 0.876 (0.002) & 0.881 (0.005) \\
Lymphography & 0.880 (0.049) & \textbf{0.996 (0.001)} & 0.995 (0.001) & 0.943 (0.032) & 0.339 (0.042) & 0.135 (0.021) & 0.194 (0.034) & 0.353 (0.066) & 0.789 (0.064) & 0.514 (0.029) & 0.969 (0.003) & 0.948 (0.007) & 0.989 (0.003) & 0.618 (0.077) & 0.544 (0.063) \\
magic.gamma & 0.752 (0.002) & 0.755 (0.004) & 0.768 (0.002) & 0.759 (0.005) & 0.745 (0.006) & 0.679 (0.003) & 0.666 (0.003) & 0.760 (0.002) & 0.773 (0.009) & 0.769 (0.005) & 0.640 (0.008) & 0.687 (0.010) & 0.746 (0.000) & 0.679 (0.002) & \textbf{0.797 (0.001)} \\
mammography & \textbf{0.838 (0.001)} & 0.738 (0.018) & 0.807 (0.014) & 0.834 (0.007) & 0.784 (0.011) & 0.838 (0.003) & 0.695 (0.012) & 0.811 (0.010) & 0.808 (0.012) & 0.820 (0.009) & 0.722 (0.008) & 0.705 (0.016) & 0.696 (0.008) & 0.703 (0.009) & 0.829 (0.002) \\
mnist & 0.756 (0.006) & 0.693 (0.009) & 0.779 (0.004) & 0.823 (0.011) & 0.682 (0.008) & 0.801 (0.006) & 0.745 (0.006) & 0.657 (0.015) & 0.843 (0.011) & 0.775 (0.003) & 0.554 (0.054) & 0.681 (0.020) & \textbf{0.846 (0.005)} & 0.658 (0.009) & 0.824 (0.003) \\
musk & \textbf{1.000 (0.000)} & \textbf{1.000 (0.000)} & \textbf{1.000 (0.000)} & 0.975 (0.016) & 0.876 (0.014) & 0.859 (0.039) & 0.919 (0.015) & 0.793 (0.011) & 0.969 (0.007) & 0.916 (0.011) & 0.836 (0.020) & 0.979 (0.006) & 0.615 (0.043) & 0.558 (0.029) & 0.597 (0.031) \\
optdigits & 0.530 (0.016) & 0.577 (0.015) & 0.497 (0.005) & 0.443 (0.019) & 0.519 (0.008) & 0.538 (0.005) & 0.533 (0.003) & 0.527 (0.009) & 0.512 (0.031) & 0.522 (0.016) & 0.557 (0.017) & \textbf{0.612 (0.014)} & 0.529 (0.021) & 0.529 (0.016) & 0.399 (0.012) \\
PageBlocks & \textbf{0.940 (0.006)} & 0.865 (0.007) & 0.918 (0.006) & 0.911 (0.007) & 0.913 (0.006) & 0.663 (0.003) & 0.711 (0.012) & 0.910 (0.008) & 0.915 (0.003) & 0.901 (0.006) & 0.735 (0.012) & 0.891 (0.008) & 0.930 (0.001) & 0.703 (0.007) & 0.817 (0.006) \\
pendigits & 0.842 (0.072) & 0.827 (0.081) & 0.711 (0.015) & 0.756 (0.020) & 0.579 (0.021) & 0.536 (0.017) & 0.541 (0.019) & 0.564 (0.015) & 0.759 (0.034) & 0.706 (0.031) & 0.664 (0.020) & \textbf{0.847 (0.008)} & 0.761 (0.006) & 0.546 (0.018) & 0.716 (0.012) \\
Pima & 0.588 (0.020) & 0.569 (0.007) & \textbf{0.691 (0.009)} & 0.607 (0.015) & 0.512 (0.011) & 0.465 (0.015) & 0.509 (0.016) & 0.529 (0.019) & 0.626 (0.009) & 0.685 (0.013) & 0.474 (0.020) & 0.572 (0.019) & 0.690 (0.008) & 0.561 (0.016) & 0.636 (0.017) \\
satellite & 0.582 (0.007) & 0.560 (0.009) & 0.715 (0.006) & 0.729 (0.006) & 0.586 (0.008) & 0.545 (0.004) & 0.528 (0.004) & 0.597 (0.007) & 0.706 (0.010) & 0.706 (0.006) & \textbf{0.751 (0.004)} & 0.705 (0.016) & 0.648 (0.002) & 0.541 (0.006) & 0.662 (0.002) \\
satimage-2 & 0.934 (0.065) & \textbf{0.998 (0.001)} & 0.983 (0.005) & 0.959 (0.001) & 0.701 (0.012) & 0.498 (0.024) & 0.470 (0.006) & 0.678 (0.011) & 0.952 (0.006) & 0.884 (0.026) & 0.965 (0.007) & 0.976 (0.011) & 0.984 (0.000) & 0.529 (0.036) & 0.919 (0.011) \\
shuttle & \textbf{0.994 (0.000)} & 0.831 (0.073) & 0.373 (0.004) & 0.975 (0.008) & 0.742 (0.057) & 0.333 (0.014) & 0.438 (0.029) & 0.845 (0.011) & 0.984 (0.002) & 0.399 (0.033) & 0.928 (0.053) & 0.942 (0.009) & 0.983 (0.000) & 0.524 (0.005) & 0.630 (0.002) \\
skin & 0.684 (0.023) & 0.495 (0.020) & 0.831 (0.006) & 0.748 (0.008) & 0.672 (0.005) & 0.497 (0.004) & 0.513 (0.003) & 0.698 (0.008) & 0.764 (0.006) & 0.793 (0.006) & 0.750 (0.003) & 0.550 (0.024) & \textbf{0.891 (0.001)} & 0.549 (0.001) & 0.672 (0.001) \\
smtp & 0.841 (0.025) & 0.921 (0.011) & 0.922 (0.013) & \textbf{0.953 (0.007)} & 0.930 (0.017) & 0.803 (0.007) & 0.933 (0.007) & 0.936 (0.010) & 0.952 (0.007) & 0.921 (0.007) & 0.919 (0.006) & 0.893 (0.018) & 0.951 (0.005) & 0.891 (0.022) & 0.894 (0.038) \\
SpamBase & 0.544 (0.002) & 0.528 (0.006) & 0.593 (0.005) & 0.517 (0.007) & 0.531 (0.005) & 0.599 (0.004) & 0.520 (0.004) & 0.531 (0.005) & 0.519 (0.009) & \textbf{0.611 (0.003)} & 0.491 (0.002) & 0.506 (0.010) & 0.531 (0.007) & 0.451 (0.002) & 0.494 (0.005) \\
speech & 0.564 (0.012) & \textbf{0.567 (0.011)} & 0.495 (0.010) & 0.475 (0.012) & 0.520 (0.015) & 0.467 (0.015) & 0.474 (0.017) & 0.486 (0.016) & 0.490 (0.005) & 0.494 (0.009) & 0.512 (0.011) & 0.516 (0.015) & 0.476 (0.005) & 0.509 (0.004) & 0.515 (0.005) \\
Stamps & 0.802 (0.035) & 0.812 (0.047) & 0.822 (0.011) & 0.759 (0.015) & 0.573 (0.040) & 0.388 (0.020) & 0.419 (0.030) & 0.504 (0.030) & 0.704 (0.011) & 0.734 (0.019) & 0.505 (0.020) & 0.558 (0.049) & \textbf{0.833 (0.014)} & 0.482 (0.034) & 0.567 (0.030) \\
thyroid & 0.963 (0.002) & 0.961 (0.004) & 0.966 (0.003) & 0.988 (0.001) & 0.903 (0.006) & 0.620 (0.004) & 0.536 (0.020) & 0.917 (0.005) & \textbf{0.989 (0.001)} & 0.928 (0.007) & 0.777 (0.010) & 0.898 (0.022) & 0.986 (0.000) & 0.651 (0.012) & 0.949 (0.006) \\
vertebral & 0.411 (0.021) & 0.445 (0.010) & 0.383 (0.011) & 0.507 (0.018) & 0.539 (0.035) & \textbf{0.620 (0.012)} & 0.576 (0.005) & 0.542 (0.036) & 0.452 (0.016) & 0.400 (0.013) & 0.399 (0.013) & 0.463 (0.025) & 0.412 (0.020) & 0.503 (0.028) & 0.483 (0.033) \\
vowels & 0.514 (0.026) & 0.712 (0.039) & \textbf{0.955 (0.005)} & 0.884 (0.020) & 0.593 (0.006) & 0.707 (0.028) & 0.808 (0.035) & 0.711 (0.009) & 0.913 (0.023) & 0.935 (0.007) & 0.931 (0.002) & 0.879 (0.022) & 0.906 (0.005) & 0.926 (0.007) & 0.943 (0.016) \\
Waveform & 0.553 (0.011) & 0.565 (0.012) & 0.656 (0.006) & 0.591 (0.006) & 0.510 (0.005) & 0.563 (0.004) & 0.604 (0.005) & 0.505 (0.013) & 0.602 (0.014) & 0.656 (0.008) & 0.438 (0.004) & \textbf{0.728 (0.015)} & 0.574 (0.002) & 0.694 (0.007) & 0.716 (0.003) \\
WBC & 0.970 (0.017) & 0.985 (0.005) & \textbf{0.989 (0.001)} & 0.867 (0.007) & 0.616 (0.028) & 0.299 (0.018) & 0.226 (0.030) & 0.539 (0.068) & 0.846 (0.006) & 0.815 (0.026) & 0.954 (0.017) & 0.935 (0.007) & 0.984 (0.004) & 0.577 (0.049) & 0.525 (0.042) \\
WDBC & \textbf{0.986 (0.001)} & 0.981 (0.003) & 0.979 (0.002) & 0.820 (0.086) & 0.546 (0.096) & 0.253 (0.036) & 0.286 (0.057) & 0.519 (0.098) & 0.857 (0.026) & 0.795 (0.084) & 0.817 (0.059) & 0.814 (0.028) & 0.976 (0.005) & 0.739 (0.090) & 0.809 (0.046) \\
Wilt & 0.411 (0.008) & 0.523 (0.007) & 0.427 (0.005) & 0.850 (0.011) & 0.766 (0.013) & 0.412 (0.005) & 0.592 (0.010) & 0.767 (0.007) & 0.834 (0.006) & 0.431 (0.012) & 0.681 (0.007) & 0.716 (0.010) & \textbf{0.858 (0.001)} & 0.684 (0.008) & 0.589 (0.005) \\
wine & 0.498 (0.060) & 0.514 (0.029) & 0.458 (0.050) & 0.584 (0.118) & 0.500 (0.073) & 0.594 (0.033) & 0.403 (0.014) & 0.485 (0.014) & 0.557 (0.065) & 0.336 (0.016) & 0.364 (0.036) & 0.490 (0.028) & \textbf{0.949 (0.034)} & 0.355 (0.061) & 0.484 (0.020) \\
WPBC & 0.491 (0.018) & 0.505 (0.017) & 0.488 (0.012) & 0.468 (0.014) & 0.503 (0.018) & \textbf{0.527 (0.009)} & 0.517 (0.014) & 0.521 (0.025) & 0.494 (0.029) & 0.519 (0.016) & 0.481 (0.009) & 0.462 (0.009) & 0.488 (0.009) & 0.461 (0.018) & 0.513 (0.017) \\
yeast & 0.452 (0.010) & 0.433 (0.012) & 0.405 (0.004) & 0.391 (0.003) & 0.459 (0.007) & 0.459 (0.007) & 0.465 (0.007) & 0.448 (0.008) & 0.412 (0.010) & 0.408 (0.006) & \textbf{0.497 (0.005)} & 0.420 (0.006) & 0.404 (0.004) & 0.455 (0.010) & 0.411 (0.008) \\
20news & 0.585 (0.006) & 0.603 (0.002) & 0.560 (0.007) & 0.580 (0.008) & 0.539 (0.005) & 0.546 (0.008) & 0.580 (0.006) & 0.547 (0.006) & 0.574 (0.007) & 0.559 (0.008) & 0.552 (0.004) & 0.543 (0.003) & 0.543 (0.006) & \textbf{0.608 (0.003)} & 0.571 (0.008) \\
agnews & 0.702 (0.001) & 0.701 (0.001) & 0.678 (0.002) & 0.616 (0.015) & 0.464 (0.011) & 0.683 (0.007) & 0.703 (0.001) & 0.413 (0.007) & 0.625 (0.007) & 0.680 (0.003) & 0.574 (0.000) & 0.559 (0.001) & 0.518 (0.008) & \textbf{0.712 (0.001)} & 0.660 (0.001) \\
amazon & 0.576 (0.006) & 0.599 (0.001) & 0.596 (0.007) & 0.567 (0.002) & 0.458 (0.020) & 0.601 (0.005) & 0.601 (0.004) & 0.439 (0.020) & 0.557 (0.008) & \textbf{0.607 (0.004)} & 0.518 (0.001) & 0.529 (0.001) & 0.519 (0.006) & 0.570 (0.002) & 0.605 (0.002) \\
CIFAR10 & 0.585 (0.004) & 0.610 (0.003) & 0.644 (0.001) & 0.643 (0.003) & 0.588 (0.004) & 0.662 (0.000) & 0.640 (0.002) & 0.596 (0.002) & 0.639 (0.003) & 0.649 (0.001) & 0.648 (0.001) & 0.545 (0.002) & 0.618 (0.008) & \textbf{0.687 (0.001)} & 0.662 (0.000) \\
FashionMNIST & 0.828 (0.006) & 0.839 (0.003) & \textbf{0.865 (0.001)} & 0.845 (0.002) & 0.794 (0.002) & 0.865 (0.000) & 0.798 (0.006) & 0.805 (0.004) & 0.843 (0.003) & 0.858 (0.001) & 0.862 (0.006) & 0.764 (0.001) & 0.760 (0.007) & 0.739 (0.001) & 0.863 (0.001) \\
imdb & 0.480 (0.002) & 0.485 (0.001) & 0.492 (0.001) & 0.485 (0.011) & 0.507 (0.008) & 0.482 (0.004) & 0.484 (0.002) & 0.499 (0.006) & 0.485 (0.010) & 0.489 (0.006) & 0.506 (0.001) & \textbf{0.521 (0.001)} & 0.495 (0.009) & 0.501 (0.002) & 0.496 (0.001) \\
MNIST-C & 0.727 (0.007) & 0.752 (0.002) & 0.782 (0.002) & 0.743 (0.003) & 0.747 (0.003) & 0.772 (0.001) & 0.760 (0.004) & 0.730 (0.002) & 0.751 (0.001) & \textbf{0.787 (0.002)} & 0.739 (0.005) & 0.683 (0.004) & 0.684 (0.007) & 0.699 (0.002) & 0.787 (0.001) \\
MVTec & 0.667 (0.006) & 0.712 (0.006) & 0.678 (0.002) & 0.724 (0.006) & 0.613 (0.006) & 0.624 (0.007) & 0.585 (0.004) & 0.582 (0.009) & 0.735 (0.005) & 0.624 (0.002) & 0.697 (0.008) & 0.656 (0.005) & 0.693 (0.003) & \textbf{0.740 (0.002)} & 0.624 (0.004) \\
SVHN & 0.589 (0.004) & 0.596 (0.001) & 0.609 (0.001) & 0.605 (0.002) & 0.586 (0.003) & 0.606 (0.000) & 0.617 (0.001) & 0.600 (0.002) & 0.601 (0.003) & 0.613 (0.001) & 0.596 (0.001) & 0.567 (0.002) & 0.552 (0.002) & \textbf{0.628 (0.000)} & 0.611 (0.001) \\
yelp & 0.651 (0.004) & 0.661 (0.002) & 0.653 (0.004) & 0.573 (0.023) & 0.428 (0.022) & 0.643 (0.003) & 0.653 (0.004) & 0.392 (0.014) & 0.612 (0.010) & 0.650 (0.003) & 0.541 (0.001) & 0.533 (0.004) & 0.515 (0.018) & 0.659 (0.002) & \textbf{0.678 (0.002)} \\
\midrule
Overall & 0.701 (0.003) & 0.710 (0.003) & 0.725 (0.001) & 0.725 (0.003) & 0.641 (0.003) & 0.602 (0.002) & 0.604 (0.001) & 0.628 (0.005) & 0.726 (0.002) & 0.685 (0.005) & 0.673 (0.001) & 0.695 (0.001) & \textbf{0.730 (0.004)} & 0.608 (0.002) & 0.665 (0.001) \\
\bottomrule
\end{tabular}
}
\end{table}

\begin{table}[H]
    \centering
    \caption{Average AUC-F1 and standard deviations over five seeds for the unsupervised setting on ADBench.}
    \resizebox{\textwidth}{!}{
\begin{tabular}{lccccccccccccccc}
\toprule
 & K-DSM-EMA & DSM-EMA & DDAE-EMA & DTE-EMA & MSM-EMA & K-DSM & DSM & MSM & DTE & DDAE & SLAD & ICL & MCD & LOF & KNN \\
\midrule
ALOI & 0.049 (0.001) & 0.028 (0.002) & 0.092 (0.002) & 0.043 (0.001) & 0.098 (0.006) & 0.060 (0.001) & 0.053 (0.002) & \textbf{0.141 (0.003)} & 0.045 (0.001) & 0.091 (0.003) & 0.049 (0.001) & 0.070 (0.002) & 0.046 (0.001) & 0.140 (0.001) & 0.118 (0.001) \\
annthyroid & 0.252 (0.004) & 0.297 (0.018) & 0.294 (0.014) & \textbf{0.670 (0.005)} & 0.331 (0.018) & 0.163 (0.009) & 0.261 (0.009) & 0.390 (0.016) & 0.648 (0.007) & 0.291 (0.010) & 0.128 (0.017) & 0.313 (0.008) & 0.455 (0.002) & 0.200 (0.002) & 0.306 (0.004) \\
backdoor & 0.312 (0.138) & 0.001 (0.001) & 0.003 (0.000) & 0.011 (0.002) & 0.240 (0.017) & 0.477 (0.008) & 0.515 (0.004) & 0.300 (0.011) & 0.028 (0.006) & 0.127 (0.014) & 0.039 (0.010) & \textbf{0.716 (0.063)} & 0.119 (0.083) & 0.130 (0.015) & 0.058 (0.003) \\
breastw & 0.674 (0.013) & 0.603 (0.007) & 0.878 (0.004) & 0.751 (0.016) & 0.621 (0.030) & 0.614 (0.010) & 0.396 (0.021) & 0.687 (0.011) & 0.736 (0.019) & 0.874 (0.009) & 0.590 (0.012) & 0.815 (0.024) & \textbf{0.918 (0.011)} & 0.187 (0.031) & 0.837 (0.011) \\
campaign & 0.243 (0.010) & 0.335 (0.012) & \textbf{0.402 (0.002)} & 0.397 (0.007) & 0.281 (0.007) & 0.279 (0.011) & 0.240 (0.003) & 0.255 (0.005) & 0.398 (0.005) & 0.328 (0.006) & 0.298 (0.003) & 0.314 (0.003) & 0.399 (0.001) & 0.133 (0.001) & 0.385 (0.001) \\
cardio & 0.393 (0.027) & \textbf{0.527 (0.033)} & 0.362 (0.015) & 0.301 (0.010) & 0.107 (0.017) & 0.153 (0.014) & 0.106 (0.014) & 0.117 (0.009) & 0.275 (0.011) & 0.252 (0.014) & 0.184 (0.016) & 0.159 (0.006) & 0.460 (0.017) & 0.201 (0.017) & 0.345 (0.010) \\
Cardiotocography & 0.336 (0.016) & \textbf{0.415 (0.014)} & 0.307 (0.006) & 0.276 (0.009) & 0.182 (0.006) & 0.227 (0.005) & 0.218 (0.005) & 0.175 (0.006) & 0.296 (0.010) & 0.241 (0.011) & 0.208 (0.019) & 0.200 (0.007) & 0.282 (0.001) & 0.273 (0.003) & 0.276 (0.004) \\
celeba & 0.044 (0.008) & 0.013 (0.002) & 0.054 (0.003) & 0.042 (0.006) & 0.061 (0.005) & 0.022 (0.002) & 0.022 (0.002) & 0.040 (0.004) & 0.045 (0.004) & 0.042 (0.002) & 0.066 (0.004) & 0.054 (0.004) & \textbf{0.121 (0.003)} & 0.052 (0.005) & 0.062 (0.003) \\
census & 0.150 (0.010) & \textbf{0.264 (0.019)} & 0.179 (0.004) & 0.095 (0.004) & 0.065 (0.008) & 0.054 (0.002) & 0.061 (0.001) & 0.054 (0.010) & 0.097 (0.002) & 0.114 (0.003) & 0.103 (0.018) & 0.128 (0.002) & 0.021 (0.013) & 0.004 (0.001) & 0.139 (0.002) \\
cover & \textbf{0.414 (0.073)} & 0.134 (0.066) & 0.057 (0.005) & 0.044 (0.005) & 0.017 (0.004) & 0.031 (0.006) & 0.031 (0.002) & 0.050 (0.007) & 0.063 (0.015) & 0.067 (0.009) & 0.026 (0.008) & 0.048 (0.018) & 0.054 (0.005) & 0.048 (0.004) & 0.075 (0.003) \\
donors & 0.172 (0.033) & 0.188 (0.053) & 0.133 (0.013) & 0.115 (0.016) & 0.043 (0.009) & 0.145 (0.002) & 0.077 (0.009) & 0.088 (0.005) & 0.156 (0.010) & 0.150 (0.016) & 0.004 (0.001) & 0.137 (0.019) & 0.148 (0.003) & 0.100 (0.004) & \textbf{0.212 (0.004)} \\
fault & 0.411 (0.005) & 0.405 (0.011) & 0.514 (0.005) & 0.445 (0.018) & 0.434 (0.003) & 0.459 (0.006) & 0.456 (0.006) & 0.413 (0.007) & 0.453 (0.014) & 0.506 (0.005) & 0.514 (0.003) & 0.481 (0.005) & 0.444 (0.006) & 0.414 (0.006) & \textbf{0.521 (0.003)} \\
fraud & \textbf{0.784 (0.029)} & 0.635 (0.058) & 0.715 (0.015) & 0.471 (0.035) & 0.025 (0.009) & 0.349 (0.074) & 0.225 (0.055) & 0.073 (0.019) & 0.750 (0.025) & 0.232 (0.048) & 0.261 (0.036) & 0.203 (0.033) & 0.526 (0.021) & 0.000 (0.000) & 0.121 (0.019) \\
glass & 0.110 (0.012) & 0.139 (0.028) & 0.169 (0.021) & 0.145 (0.014) & 0.109 (0.029) & 0.019 (0.019) & 0.075 (0.028) & 0.095 (0.045) & 0.155 (0.012) & 0.158 (0.055) & 0.155 (0.041) & 0.105 (0.025) & 0.039 (0.024) & 0.189 (0.029) & \textbf{0.205 (0.030)} \\
Hepatitis & 0.299 (0.034) & \textbf{0.354 (0.071)} & 0.311 (0.027) & 0.270 (0.057) & 0.166 (0.026) & 0.072 (0.022) & 0.062 (0.017) & 0.182 (0.063) & 0.166 (0.057) & 0.152 (0.017) & 0.067 (0.020) & 0.135 (0.038) & 0.259 (0.029) & 0.135 (0.044) & 0.180 (0.030) \\
http & 0.828 (0.014) & 0.009 (0.003) & 0.201 (0.194) & 0.011 (0.005) & 0.044 (0.008) & 0.009 (0.005) & 0.106 (0.100) & 0.048 (0.009) & 0.149 (0.125) & 0.027 (0.004) & 0.025 (0.006) & 0.503 (0.021) & \textbf{0.861 (0.011)} & 0.049 (0.009) & 0.041 (0.006) \\
InternetAds & 0.470 (0.024) & \textbf{0.539 (0.008)} & 0.465 (0.009) & 0.226 (0.028) & 0.367 (0.021) & 0.446 (0.029) & 0.371 (0.006) & 0.298 (0.013) & 0.218 (0.023) & 0.359 (0.009) & 0.276 (0.007) & 0.191 (0.013) & 0.431 (0.041) & 0.359 (0.011) & 0.359 (0.012) \\
Ionosphere & 0.518 (0.017) & 0.604 (0.014) & 0.711 (0.009) & 0.719 (0.013) & 0.535 (0.028) & 0.429 (0.011) & 0.523 (0.013) & 0.525 (0.011) & 0.725 (0.021) & 0.625 (0.009) & 0.717 (0.016) & 0.491 (0.021) & \textbf{0.854 (0.010)} & 0.758 (0.014) & 0.577 (0.018) \\
landsat & 0.206 (0.012) & 0.230 (0.010) & 0.294 (0.004) & 0.174 (0.017) & 0.253 (0.004) & 0.262 (0.002) & 0.269 (0.004) & 0.257 (0.005) & 0.179 (0.004) & 0.303 (0.008) & 0.368 (0.003) & \textbf{0.414 (0.016)} & 0.200 (0.002) & 0.268 (0.005) & 0.289 (0.005) \\
letter & 0.300 (0.023) & 0.338 (0.022) & 0.324 (0.012) & 0.334 (0.016) & 0.022 (0.006) & 0.396 (0.022) & \textbf{0.452 (0.024)} & 0.268 (0.017) & 0.326 (0.016) & 0.306 (0.004) & 0.342 (0.013) & 0.272 (0.015) & 0.252 (0.011) & 0.396 (0.016) & 0.356 (0.005) \\
Lymphography & 0.720 (0.056) & \textbf{0.826 (0.033)} & 0.810 (0.024) & 0.769 (0.062) & 0.021 (0.021) & 0.058 (0.038) & 0.021 (0.021) & 0.034 (0.034) & 0.377 (0.038) & 0.217 (0.055) & 0.605 (0.036) & 0.483 (0.076) & 0.758 (0.055) & 0.112 (0.047) & 0.228 (0.067) \\
magic.gamma & 0.595 (0.002) & 0.594 (0.003) & 0.610 (0.002) & 0.617 (0.009) & 0.597 (0.008) & 0.521 (0.003) & 0.511 (0.003) & 0.617 (0.004) & 0.633 (0.013) & 0.611 (0.004) & 0.483 (0.009) & 0.531 (0.014) & 0.586 (0.000) & 0.513 (0.001) & \textbf{0.639 (0.001)} \\
mammography & 0.178 (0.003) & 0.079 (0.008) & 0.188 (0.009) & 0.221 (0.010) & 0.098 (0.013) & \textbf{0.258 (0.022)} & 0.082 (0.003) & 0.145 (0.016) & 0.208 (0.008) & 0.173 (0.012) & 0.125 (0.010) & 0.109 (0.009) & 0.012 (0.002) & 0.182 (0.005) & 0.219 (0.006) \\
mnist & \textbf{0.484 (0.009)} & 0.390 (0.010) & 0.385 (0.006) & 0.381 (0.007) & 0.335 (0.004) & 0.371 (0.006) & 0.376 (0.007) & 0.251 (0.009) & 0.393 (0.016) & 0.354 (0.006) & 0.087 (0.087) & 0.295 (0.011) & 0.397 (0.005) & 0.268 (0.006) & 0.405 (0.001) \\
musk & \textbf{1.000 (0.000)} & \textbf{1.000 (0.000)} & \textbf{1.000 (0.000)} & 0.740 (0.120) & 0.049 (0.021) & 0.303 (0.052) & 0.454 (0.018) & 0.099 (0.008) & 0.470 (0.065) & 0.334 (0.031) & 0.260 (0.026) & 0.540 (0.083) & 0.173 (0.008) & 0.144 (0.035) & 0.155 (0.023) \\
optdigits & 0.017 (0.003) & 0.000 (0.000) & 0.000 (0.000) & 0.000 (0.000) & 0.001 (0.001) & 0.025 (0.003) & 0.025 (0.003) & \textbf{0.047 (0.011)} & 0.000 (0.000) & 0.020 (0.006) & 0.000 (0.000) & 0.012 (0.004) & 0.000 (0.000) & 0.047 (0.005) & 0.000 (0.000) \\
PageBlocks & \textbf{0.633 (0.016)} & 0.362 (0.007) & 0.525 (0.012) & 0.573 (0.011) & 0.471 (0.010) & 0.335 (0.008) & 0.340 (0.010) & 0.469 (0.013) & 0.527 (0.011) & 0.483 (0.019) & 0.388 (0.009) & 0.504 (0.007) & 0.598 (0.005) & 0.320 (0.012) & 0.462 (0.005) \\
pendigits & \textbf{0.449 (0.156)} & 0.319 (0.129) & 0.069 (0.007) & 0.046 (0.006) & 0.042 (0.010) & 0.076 (0.007) & 0.067 (0.004) & 0.029 (0.006) & 0.049 (0.006) & 0.069 (0.004) & 0.051 (0.012) & 0.122 (0.016) & 0.044 (0.003) & 0.090 (0.008) & 0.096 (0.005) \\
Pima & 0.470 (0.021) & 0.444 (0.009) & 0.520 (0.005) & 0.425 (0.013) & 0.367 (0.017) & 0.323 (0.010) & 0.350 (0.010) & 0.391 (0.017) & 0.448 (0.001) & 0.514 (0.004) & 0.340 (0.022) & 0.427 (0.011) & \textbf{0.533 (0.014)} & 0.406 (0.012) & 0.525 (0.011) \\
satellite & 0.424 (0.007) & 0.384 (0.008) & 0.503 (0.005) & \textbf{0.602 (0.009)} & 0.376 (0.008) & 0.365 (0.003) & 0.355 (0.003) & 0.370 (0.004) & 0.573 (0.012) & 0.506 (0.006) & 0.534 (0.006) & 0.564 (0.029) & 0.468 (0.003) & 0.371 (0.005) & 0.478 (0.003) \\
satimage-2 & 0.744 (0.130) & \textbf{0.775 (0.031)} & 0.563 (0.059) & 0.206 (0.011) & 0.028 (0.009) & 0.059 (0.015) & 0.045 (0.014) & 0.037 (0.016) & 0.197 (0.018) & 0.217 (0.020) & 0.341 (0.026) & 0.400 (0.103) & 0.355 (0.010) & 0.099 (0.027) & 0.355 (0.022) \\
shuttle & \textbf{0.957 (0.001)} & 0.635 (0.132) & 0.091 (0.003) & 0.676 (0.066) & 0.227 (0.028) & 0.093 (0.002) & 0.141 (0.008) & 0.250 (0.010) & 0.744 (0.025) & 0.103 (0.007) & 0.611 (0.139) & 0.441 (0.023) & 0.875 (0.002) & 0.128 (0.004) & 0.211 (0.003) \\
skin & 0.408 (0.019) & 0.195 (0.033) & 0.421 (0.007) & 0.199 (0.006) & 0.164 (0.011) & 0.110 (0.008) & 0.115 (0.002) & 0.180 (0.006) & 0.194 (0.011) & 0.355 (0.007) & 0.415 (0.006) & 0.189 (0.036) & \textbf{0.588 (0.002)} & 0.210 (0.003) & 0.282 (0.004) \\
smtp & 0.656 (0.058) & 0.656 (0.058) & 0.546 (0.056) & 0.614 (0.053) & 0.000 (0.000) & 0.193 (0.083) & 0.686 (0.027) & 0.000 (0.000) & \textbf{0.696 (0.020)} & 0.495 (0.072) & \textbf{0.696 (0.020)} & 0.206 (0.079) & 0.000 (0.000) & 0.000 (0.000) & 0.324 (0.078) \\
SpamBase & 0.420 (0.001) & 0.405 (0.005) & 0.464 (0.006) & 0.380 (0.007) & 0.413 (0.005) & 0.471 (0.003) & 0.387 (0.002) & 0.395 (0.005) & 0.385 (0.010) & \textbf{0.479 (0.004)} & 0.388 (0.003) & 0.379 (0.012) & 0.410 (0.007) & 0.355 (0.003) & 0.417 (0.003) \\
speech & \textbf{0.043 (0.008)} & 0.039 (0.011) & 0.026 (0.011) & 0.023 (0.012) & 0.026 (0.008) & 0.033 (0.009) & 0.026 (0.008) & 0.007 (0.004) & 0.030 (0.010) & 0.036 (0.012) & 0.013 (0.006) & 0.026 (0.011) & 0.030 (0.006) & 0.030 (0.008) & 0.030 (0.003) \\
Stamps & 0.210 (0.035) & \textbf{0.239 (0.057)} & 0.229 (0.037) & 0.222 (0.022) & 0.104 (0.016) & 0.056 (0.015) & 0.063 (0.018) & 0.126 (0.014) & 0.213 (0.019) & 0.176 (0.013) & 0.164 (0.027) & 0.159 (0.035) & 0.124 (0.030) & 0.187 (0.033) & 0.206 (0.020) \\
thyroid & 0.383 (0.015) & 0.409 (0.043) & 0.417 (0.050) & \textbf{0.703 (0.014)} & 0.228 (0.036) & 0.148 (0.005) & 0.095 (0.020) & 0.178 (0.025) & 0.701 (0.004) & 0.202 (0.032) & 0.168 (0.021) & 0.194 (0.035) & 0.656 (0.000) & 0.097 (0.009) & 0.297 (0.012) \\
vertebral & 0.000 (0.000) & 0.000 (0.000) & 0.043 (0.007) & 0.084 (0.014) & 0.116 (0.020) & \textbf{0.155 (0.031)} & 0.133 (0.019) & 0.130 (0.023) & 0.078 (0.015) & 0.039 (0.010) & 0.065 (0.010) & 0.052 (0.011) & 0.068 (0.044) & 0.105 (0.034) & 0.046 (0.012) \\
vowels & 0.132 (0.019) & 0.252 (0.019) & 0.472 (0.016) & 0.392 (0.054) & 0.008 (0.005) & 0.212 (0.008) & 0.272 (0.026) & 0.208 (0.022) & 0.416 (0.047) & 0.424 (0.007) & 0.344 (0.025) & 0.244 (0.024) & 0.356 (0.028) & 0.324 (0.021) & \textbf{0.504 (0.013)} \\
Waveform & 0.052 (0.007) & 0.048 (0.004) & 0.064 (0.012) & 0.048 (0.002) & 0.012 (0.005) & 0.032 (0.012) & 0.042 (0.009) & 0.050 (0.011) & 0.064 (0.007) & 0.058 (0.010) & 0.008 (0.004) & \textbf{0.198 (0.022)} & 0.064 (0.006) & 0.102 (0.010) & 0.156 (0.009) \\
WBC & 0.751 (0.058) & 0.739 (0.059) & \textbf{0.782 (0.032)} & 0.313 (0.016) & 0.060 (0.030) & 0.040 (0.027) & 0.064 (0.027) & 0.049 (0.013) & 0.222 (0.018) & 0.244 (0.036) & 0.540 (0.088) & 0.329 (0.071) & 0.667 (0.056) & 0.082 (0.029) & 0.250 (0.025) \\
WDBC & \textbf{0.531 (0.027)} & 0.433 (0.036) & 0.417 (0.068) & 0.183 (0.047) & 0.015 (0.009) & 0.016 (0.010) & 0.016 (0.010) & 0.015 (0.009) & 0.208 (0.050) & 0.246 (0.081) & 0.303 (0.077) & 0.023 (0.010) & 0.438 (0.070) & 0.058 (0.027) & 0.279 (0.053) \\
Wilt & 0.004 (0.000) & 0.012 (0.001) & 0.003 (0.001) & 0.105 (0.031) & 0.079 (0.011) & 0.013 (0.002) & 0.074 (0.002) & 0.066 (0.006) & 0.055 (0.013) & 0.003 (0.001) & \textbf{0.107 (0.011)} & 0.075 (0.011) & 0.000 (0.000) & 0.060 (0.006) & 0.001 (0.001) \\
wine & 0.000 (0.000) & 0.035 (0.035) & 0.005 (0.005) & 0.043 (0.027) & 0.051 (0.028) & 0.066 (0.027) & 0.070 (0.020) & 0.060 (0.026) & 0.044 (0.027) & 0.029 (0.020) & 0.000 (0.000) & 0.021 (0.021) & \textbf{0.550 (0.128)} & 0.000 (0.000) & 0.076 (0.025) \\
WPBC & 0.182 (0.025) & 0.192 (0.013) & 0.203 (0.009) & 0.197 (0.027) & 0.246 (0.025) & 0.223 (0.022) & 0.254 (0.028) & \textbf{0.272 (0.016)} & 0.226 (0.027) & 0.254 (0.018) & 0.218 (0.016) & 0.178 (0.010) & 0.158 (0.016) & 0.187 (0.032) & 0.261 (0.015) \\
yeast & 0.293 (0.007) & 0.281 (0.013) & 0.282 (0.005) & 0.272 (0.005) & 0.303 (0.008) & 0.321 (0.003) & 0.331 (0.006) & 0.302 (0.007) & 0.295 (0.010) & 0.284 (0.008) & \textbf{0.336 (0.005)} & 0.277 (0.009) & 0.268 (0.007) & 0.306 (0.004) & 0.295 (0.005) \\
20news & 0.086 (0.003) & 0.081 (0.003) & 0.070 (0.004) & 0.064 (0.007) & 0.058 (0.002) & 0.087 (0.007) & 0.092 (0.004) & 0.092 (0.005) & 0.070 (0.005) & 0.090 (0.005) & 0.060 (0.004) & 0.066 (0.002) & 0.069 (0.004) & \textbf{0.103 (0.003)} & 0.081 (0.004) \\
agnews & 0.137 (0.000) & 0.105 (0.001) & 0.125 (0.001) & 0.085 (0.007) & 0.039 (0.004) & 0.114 (0.006) & 0.123 (0.003) & 0.030 (0.002) & 0.084 (0.003) & 0.115 (0.004) & 0.075 (0.001) & 0.074 (0.002) & 0.062 (0.004) & \textbf{0.169 (0.002)} & 0.110 (0.001) \\
amazon & 0.045 (0.000) & 0.050 (0.001) & 0.046 (0.001) & 0.056 (0.003) & 0.039 (0.008) & 0.057 (0.004) & 0.050 (0.001) & 0.027 (0.007) & 0.057 (0.003) & 0.051 (0.002) & 0.043 (0.001) & 0.044 (0.003) & \textbf{0.057 (0.004)} & 0.053 (0.002) & 0.046 (0.001) \\
CIFAR10 & 0.140 (0.004) & 0.123 (0.001) & 0.128 (0.001) & 0.130 (0.001) & 0.075 (0.003) & 0.127 (0.001) & 0.108 (0.002) & 0.095 (0.002) & 0.126 (0.002) & 0.125 (0.002) & 0.138 (0.001) & 0.105 (0.002) & 0.106 (0.004) & \textbf{0.159 (0.001)} & 0.130 (0.000) \\
FashionMNIST & \textbf{0.481 (0.008)} & 0.404 (0.002) & 0.394 (0.001) & 0.361 (0.006) & 0.125 (0.013) & 0.347 (0.002) & 0.267 (0.001) & 0.212 (0.004) & 0.321 (0.002) & 0.316 (0.002) & 0.360 (0.003) & 0.223 (0.003) & 0.327 (0.012) & 0.257 (0.002) & 0.354 (0.001) \\
imdb & 0.018 (0.001) & 0.022 (0.001) & 0.023 (0.001) & 0.040 (0.004) & 0.045 (0.003) & 0.029 (0.002) & 0.028 (0.002) & 0.044 (0.003) & 0.026 (0.003) & 0.025 (0.002) & \textbf{0.057 (0.001)} & 0.055 (0.002) & 0.034 (0.006) & 0.040 (0.002) & 0.019 (0.001) \\
MNIST-C & \textbf{0.264 (0.016)} & 0.255 (0.002) & 0.232 (0.001) & 0.202 (0.006) & 0.139 (0.004) & 0.212 (0.002) & 0.233 (0.003) & 0.168 (0.002) & 0.195 (0.003) & 0.217 (0.001) & 0.214 (0.004) & 0.139 (0.003) & 0.187 (0.006) & 0.159 (0.001) & 0.222 (0.001) \\
MVTec & 0.497 (0.005) & \textbf{0.514 (0.005)} & 0.454 (0.005) & 0.488 (0.007) & 0.371 (0.007) & 0.394 (0.007) & 0.364 (0.005) & 0.334 (0.008) & 0.493 (0.006) & 0.393 (0.002) & 0.469 (0.007) & 0.389 (0.005) & 0.462 (0.004) & 0.495 (0.005) & 0.422 (0.003) \\
SVHN & \textbf{0.124 (0.004)} & 0.115 (0.001) & 0.113 (0.000) & 0.113 (0.001) & 0.078 (0.000) & 0.108 (0.002) & 0.091 (0.002) & 0.094 (0.001) & 0.111 (0.001) & 0.105 (0.001) & 0.111 (0.001) & 0.093 (0.002) & 0.096 (0.001) & 0.110 (0.001) & 0.113 (0.001) \\
yelp & 0.114 (0.002) & 0.098 (0.002) & 0.109 (0.004) & 0.074 (0.014) & 0.030 (0.004) & 0.097 (0.008) & 0.109 (0.003) & 0.025 (0.004) & 0.076 (0.012) & 0.102 (0.005) & 0.023 (0.001) & 0.031 (0.002) & 0.072 (0.011) & 0.102 (0.004) & \textbf{0.115 (0.002)} \\
\midrule
Overall & \textbf{0.344 (0.005)} & 0.308 (0.005) & 0.312 (0.004) & 0.284 (0.003) & 0.166 (0.001) & 0.195 (0.003) & 0.200 (0.003) & 0.182 (0.002) & 0.279 (0.004) & 0.241 (0.003) & 0.238 (0.002) & 0.245 (0.006) & 0.308 (0.005) & 0.184 (0.002) & 0.250 (0.001) \\
\bottomrule
\end{tabular}
}
\end{table}

\clearpage
\section{Image Benchmark Results}
\label{sec:image_benchmarks_appendix}

To assess whether K-DSM generalises beyond tabular data, we evaluate on two standard industrial image anomaly detection benchmarks, MVTec-AD \citep{bergmann2019mvtec} and VisA \citep{zou2022visa}, using frozen DINOv3 \citep{siméoni2025dinov3} feature embeddings followed by standardisation. All methods share the same features and preprocessing.

\begin{table}[h]
\centering
\caption{Mean AUC-PR and AUC-ROC on MVTec-AD and VisA with standardised DINOv3 embeddings (5 seeds, standard error in parentheses). Best in bold.}
\label{tab:image_mean}
\resizebox{\textwidth}{!}{
\begin{tabular}{llccccccccc}
\toprule
 & & KNN & LOF & ICL & DTE & DSM & K-DSM & MSM & DDAE & SLAD \\
\midrule
\multirow{2}{*}{AUC-PR}  & VISA   & 0.904 (0.000) & 0.900 (0.000) & 0.785 (0.001) & 0.857 (0.001) & \textbf{0.939 (0.000)} & 0.918 (0.000) & 0.936 (0.000) & 0.926 (0.000) & 0.829 (0.000) \\
                         & MVTec  & 0.971 (0.000) & 0.970 (0.000) & 0.954 (0.000) & 0.963 (0.000) & 0.983 (0.000) & 0.962 (0.000) & 0.958 (0.000) & \textbf{0.983 (0.000)} & 0.970 (0.000) \\
\midrule
\multirow{2}{*}{AUC-ROC} & VISA   & 0.887 (0.000) & 0.878 (0.000) & 0.735 (0.001) & 0.815 (0.001) & \textbf{0.928 (0.000)} & 0.881 (0.000) & 0.926 (0.000) & 0.913 (0.000) & 0.784 (0.000) \\
                         & MVTec  & 0.943 (0.000) & 0.938 (0.000) & 0.911 (0.000) & 0.930 (0.000) & 0.960 (0.000) & 0.900 (0.000) & 0.919 (0.001) & \textbf{0.965 (0.000)} & 0.939 (0.000) \\
\bottomrule
\end{tabular}}
\end{table}

K-DSM is competitive with the strongest methods on both benchmarks but does not dominate as it does on ADBench. Single-scale DSM, DDAE, and MSM all perform comparably, and the margins among the top methods are small. This is consistent with the observation that standardised DINOv3 embeddings are high-dimensional and approximately Gaussian per coordinate after whitening: marginal kurtosis is close to 3 almost everywhere, so the CF multiplier $1 + c(\kappa-3)$ is close to 1 across features and K-DSM collapses toward single-scale DSM. Kurtosis-guided scaling delivers its largest gains precisely when marginals exhibit meaningful tail heterogeneity, as is typical of tabular ADBench features but not of Gaussianised deep embeddings. We view this as a clean statement of the regime of validity of our method rather than a failure mode: when per-feature tail shape is uninformative, there is simply little for kurtosis to exploit.

\subsection{Per-class Image Benchmark Results}
\label{sec:image_perclass}

We report per-class AUC-PR on MVTec-AD and VisA with standardised DINOv3 embeddings. All methods use identical features, preprocessing, and splits. Results are mean over 5 seeds with standard error in parentheses; best per row in bold.

\begin{table}[h]
\centering
\caption{Per-class AUC-PR on MVTec-AD (standardised DINOv3 embeddings).}
\label{tab:mvtec_perclass_pr}
\resizebox{\textwidth}{!}{
\begin{tabular}{lccccccccc}
\toprule
 & KNN & LOF & ICL & DTE & DSM & K-DSM & MSM & DDAE & SLAD \\
\midrule
bottle & 0.999 (0.000) & 0.999 (0.000) & 0.987 (0.000) & 0.999 (0.000) & \textbf{1.000 (0.000)} & 1.000 (0.000) & 1.000 (0.000) & \textbf{1.000 (0.000)} & 0.999 (0.000) \\
cable & 0.936 (0.000) & 0.946 (0.000) & 0.894 (0.000) & 0.896 (0.000) & 0.959 (0.000) & \textbf{0.974 (0.000)} & 0.956 (0.002) & 0.963 (0.001) & 0.935 (0.000) \\
capsule & 0.986 (0.000) & 0.988 (0.000) & 0.979 (0.000) & \textbf{0.988 (0.000)} & 0.987 (0.000) & 0.985 (0.000) & 0.987 (0.001) & 0.985 (0.000) & 0.978 (0.000) \\
carpet & 0.996 (0.000) & 0.996 (0.000) & 0.996 (0.000) & 0.995 (0.000) & \textbf{0.997 (0.000)} & 0.985 (0.000) & 0.996 (0.000) & 0.995 (0.000) & 0.995 (0.000) \\
grid & \textbf{0.999 (0.000)} & 0.998 (0.000) & 0.987 (0.000) & 0.996 (0.000) & 0.999 (0.000) & 0.999 (0.000) & 0.999 (0.000) & 0.998 (0.000) & 0.997 (0.000) \\
hazelnut & 0.945 (0.000) & 0.960 (0.000) & 0.937 (0.000) & 0.970 (0.000) & \textbf{0.995 (0.000)} & 0.975 (0.000) & 0.995 (0.000) & 0.983 (0.001) & 0.952 (0.000) \\
leather & \textbf{1.000 (0.000)} & 1.000 (0.000) & 0.997 (0.000) & 0.999 (0.000) & 0.999 (0.000) & 0.992 (0.000) & 1.000 (0.000) & 1.000 (0.000) & 1.000 (0.000) \\
metal\_nut & 0.986 (0.000) & 0.986 (0.000) & 0.970 (0.000) & 0.985 (0.000) & \textbf{0.998 (0.000)} & 0.978 (0.000) & 0.995 (0.001) & 0.996 (0.000) & 0.991 (0.000) \\
pill & 0.992 (0.000) & 0.992 (0.000) & 0.971 (0.000) & 0.976 (0.000) & 0.995 (0.000) & \textbf{0.996 (0.000)} & 0.996 (0.000) & 0.995 (0.000) & 0.989 (0.000) \\
screw & 0.916 (0.000) & 0.900 (0.000) & 0.908 (0.000) & 0.914 (0.000) & 0.951 (0.000) & \textbf{0.954 (0.000)} & 0.927 (0.002) & 0.934 (0.001) & 0.919 (0.000) \\
tile & 0.999 (0.000) & 1.000 (0.000) & 0.997 (0.000) & 0.997 (0.000) & \textbf{1.000 (0.000)} & 0.714 (0.000) & \textbf{1.000 (0.000)} & 1.000 (0.000) & 0.999 (0.000) \\
toothbrush & 0.959 (0.000) & 0.925 (0.000) & 0.960 (0.000) & 0.922 (0.000) & 0.941 (0.000) & 0.965 (0.000) & 0.579 (0.001) & \textbf{0.970 (0.001)} & 0.954 (0.000) \\
transistor & 0.908 (0.000) & 0.914 (0.000) & 0.776 (0.000) & 0.829 (0.000) & 0.936 (0.000) & 0.934 (0.000) & \textbf{0.959 (0.005)} & 0.956 (0.002) & 0.908 (0.000) \\
wood & 0.944 (0.000) & 0.953 (0.000) & 0.979 (0.000) & 0.979 (0.000) & 0.985 (0.000) & \textbf{0.988 (0.000)} & 0.983 (0.000) & 0.971 (0.001) & 0.946 (0.000) \\
zipper & 0.995 (0.000) & 0.997 (0.000) & 0.979 (0.000) & 0.997 (0.000) & 0.998 (0.000) & 0.992 (0.000) & \textbf{0.999 (0.000)} & 0.998 (0.000) & 0.992 (0.000) \\
\midrule
Mean & 0.971 (0.000) & 0.970 (0.000) & 0.954 (0.000) & 0.963 (0.000) & 0.983 (0.000) & 0.962 (0.000) & 0.958 (0.000) & \textbf{0.983 (0.000)} & 0.970 (0.000) \\
\bottomrule
\end{tabular}}
\end{table}

\begin{table}[h]
\centering
\caption{Per-class AUC-PR on VisA (standardised DINOv3 embeddings).}
\label{tab:visa_perclass_pr}
\resizebox{\textwidth}{!}{
\begin{tabular}{lccccccccc}
\toprule
 & KNN & LOF & ICL & DTE & DSM & K-DSM & MSM & DDAE & SLAD \\
\midrule
candle & 0.878 (0.000) & 0.887 (0.000) & 0.785 (0.010) & 0.851 (0.017) & 0.912 (0.001) & \textbf{0.924 (0.004)} & 0.909 (0.001) & 0.896 (0.001) & 0.797 (0.000) \\
capsules & 0.910 (0.000) & 0.913 (0.000) & 0.811 (0.000) & 0.891 (0.000) & 0.931 (0.000) & 0.931 (0.000) & \textbf{0.936 (0.002)} & 0.921 (0.000) & 0.878 (0.000) \\
cashew & 0.967 (0.000) & 0.967 (0.000) & 0.882 (0.000) & 0.867 (0.000) & 0.977 (0.000) & 0.875 (0.000) & \textbf{0.979 (0.001)} & 0.969 (0.002) & 0.919 (0.000) \\
chewinggum & 0.991 (0.000) & 0.990 (0.000) & 0.969 (0.000) & 0.988 (0.000) & \textbf{0.995 (0.000)} & 0.988 (0.000) & 0.994 (0.000) & 0.994 (0.000) & 0.987 (0.000) \\
fryum & 0.950 (0.000) & 0.956 (0.000) & 0.897 (0.000) & 0.924 (0.000) & 0.980 (0.000) & \textbf{0.983 (0.000)} & 0.981 (0.000) & 0.969 (0.001) & 0.938 (0.000) \\
macaroni1 & 0.846 (0.000) & 0.879 (0.000) & 0.715 (0.000) & 0.811 (0.000) & \textbf{0.939 (0.000)} & 0.918 (0.000) & 0.933 (0.002) & 0.897 (0.004) & 0.739 (0.000) \\
macaroni2 & 0.708 (0.000) & 0.716 (0.000) & 0.588 (0.000) & 0.673 (0.000) & \textbf{0.801 (0.000)} & 0.791 (0.000) & 0.783 (0.005) & 0.769 (0.006) & 0.712 (0.000) \\
pcb1 & 0.967 (0.000) & 0.907 (0.000) & 0.562 (0.000) & 0.926 (0.000) & \textbf{0.983 (0.000)} & 0.977 (0.000) & 0.971 (0.001) & 0.980 (0.001) & 0.552 (0.000) \\
pcb2 & 0.867 (0.000) & 0.873 (0.000) & 0.721 (0.000) & 0.815 (0.000) & 0.939 (0.000) & \textbf{0.940 (0.000)} & 0.937 (0.002) & 0.909 (0.003) & 0.750 (0.000) \\
pcb3 & 0.803 (0.000) & 0.812 (0.000) & 0.721 (0.000) & 0.694 (0.000) & 0.829 (0.000) & 0.819 (0.000) & 0.825 (0.001) & \textbf{0.830 (0.002)} & 0.765 (0.000) \\
pcb4 & 0.987 (0.000) & 0.921 (0.000) & 0.827 (0.000) & 0.899 (0.000) & 0.990 (0.000) & 0.922 (0.000) & 0.992 (0.001) & \textbf{0.992 (0.001)} & 0.951 (0.000) \\
pipe\_fryum & 0.976 (0.000) & 0.979 (0.000) & 0.937 (0.000) & 0.946 (0.000) & \textbf{0.992 (0.000)} & 0.952 (0.000) & 0.988 (0.001) & 0.985 (0.001) & 0.959 (0.000) \\
\midrule
Mean & 0.904 (0.000) & 0.900 (0.000) & 0.785 (0.001) & 0.857 (0.001) & \textbf{0.939 (0.000)} & 0.918 (0.000) & 0.936 (0.000) & 0.926 (0.000) & 0.829 (0.000) \\
\bottomrule
\end{tabular}}
\end{table}

\clearpage
\section{Full Semi-supervised Results}
\label{app:semi_sup_full}

We report per-dataset AUC-PR, AUC-ROC, and F1 for the semi-supervised setting summarised in Section~\ref{sec:results}. Best per row in bold; standard errors over 5 seeds in parentheses.

\begin{table}[H]
    \centering
    \caption{Average AUC-PR and standard deviations over five seeds for the semi-supervised setting on ADBench.}
    \resizebox{\textwidth}{!}{
\begin{tabular}{lccccccccc}
\toprule
 & K-DSM & DSM & MSM & DTE & DDAE & SLAD & ICL & LOF & KNN \\
\midrule
ALOI & \textbf{0.071 (0.001)} & 0.062 (0.000) & 0.061 (0.001) & 0.058 (0.000) & 0.061 (0.000) & 0.060 (0.000) & 0.055 (0.001) & 0.065 (0.000) & 0.060 (0.000) \\
annthyroid & 0.517 (0.018) & 0.406 (0.033) & 0.354 (0.007) & \textbf{0.828 (0.001)} & 0.686 (0.008) & 0.707 (0.000) & 0.527 (0.013) & 0.535 (0.000) & 0.681 (0.000) \\
backdoor & 0.887 (0.005) & 0.874 (0.004) & \textbf{0.889 (0.005)} & 0.606 (0.013) & 0.749 (0.012) & 0.048 (0.000) & 0.880 (0.005) & 0.537 (0.012) & 0.285 (0.004) \\
breastw & 0.991 (0.001) & 0.987 (0.002) & 0.824 (0.028) & 0.895 (0.008) & \textbf{0.995 (0.001)} & 0.994 (0.001) & 0.983 (0.002) & 0.828 (0.032) & 0.990 (0.001) \\
campaign & 0.357 (0.007) & 0.371 (0.008) & 0.432 (0.006) & 0.486 (0.008) & 0.517 (0.006) & 0.484 (0.000) & 0.466 (0.004) & 0.402 (0.000) & \textbf{0.587 (0.000)} \\
cardio & 0.705 (0.008) & 0.724 (0.003) & 0.259 (0.005) & 0.704 (0.003) & 0.726 (0.004) & 0.692 (0.000) & 0.631 (0.021) & 0.702 (0.000) & \textbf{0.772 (0.000)} \\
Cardiotocography & 0.544 (0.002) & 0.553 (0.007) & 0.338 (0.008) & 0.547 (0.009) & 0.530 (0.005) & 0.494 (0.000) & 0.489 (0.011) & \textbf{0.573 (0.000)} & 0.561 (0.000) \\
celeba & 0.085 (0.024) & 0.119 (0.007) & 0.083 (0.002) & \textbf{0.139 (0.006)} & 0.048 (0.003) & 0.094 (0.003) & 0.101 (0.002) & 0.036 (0.001) & 0.059 (0.002) \\
census & 0.197 (0.003) & 0.189 (0.003) & 0.185 (0.008) & 0.184 (0.003) & 0.204 (0.005) & 0.132 (0.014) & 0.258 (0.005) & 0.136 (0.002) & \textbf{0.334 (0.008)} \\
cover & 0.882 (0.020) & \textbf{0.911 (0.008)} & 0.849 (0.018) & 0.676 (0.020) & 0.705 (0.025) & 0.062 (0.021) & 0.418 (0.072) & 0.829 (0.010) & 0.552 (0.017) \\
donors & 0.991 (0.003) & 0.959 (0.012) & \textbf{0.995 (0.001)} & 0.733 (0.018) & 0.948 (0.007) & 0.460 (0.064) & 0.793 (0.034) & 0.636 (0.007) & 0.799 (0.008) \\
fault & 0.598 (0.002) & 0.609 (0.003) & 0.581 (0.008) & 0.637 (0.003) & 0.609 (0.003) & \textbf{0.668 (0.000)} & 0.644 (0.002) & 0.504 (0.000) & 0.620 (0.000) \\
fraud & 0.624 (0.060) & 0.561 (0.040) & 0.515 (0.056) & \textbf{0.625 (0.020)} & 0.532 (0.032) & 0.463 (0.026) & 0.585 (0.050) & 0.598 (0.017) & 0.423 (0.028) \\
glass & 0.780 (0.058) & 0.706 (0.032) & 0.849 (0.041) & 0.395 (0.025) & 0.628 (0.045) & 0.415 (0.023) & \textbf{0.857 (0.047)} & 0.423 (0.036) & 0.441 (0.032) \\
Hepatitis & \textbf{1.000 (0.000)} & 0.997 (0.003) & 0.997 (0.003) & 0.970 (0.023) & 0.997 (0.003) & 0.998 (0.002) & 0.990 (0.010) & 0.458 (0.048) & 0.911 (0.022) \\
http & 1.000 (0.000) & 0.895 (0.012) & 0.966 (0.022) & 0.578 (0.050) & 0.951 (0.042) & 0.886 (0.034) & 0.962 (0.009) & 0.971 (0.017) & \textbf{1.000 (0.000)} \\
InternetAds & \textbf{0.767 (0.002)} & 0.585 (0.005) & 0.691 (0.013) & 0.557 (0.027) & 0.731 (0.006) & 0.593 (0.000) & 0.596 (0.005) & 0.504 (0.000) & 0.705 (0.000) \\
Ionosphere & \textbf{0.991 (0.001)} & 0.986 (0.001) & 0.975 (0.005) & 0.966 (0.006) & 0.985 (0.002) & 0.985 (0.003) & 0.989 (0.002) & 0.945 (0.007) & 0.979 (0.003) \\
landsat & 0.497 (0.003) & 0.484 (0.007) & 0.422 (0.006) & 0.374 (0.011) & 0.519 (0.008) & 0.446 (0.000) & 0.535 (0.005) & \textbf{0.614 (0.000)} & 0.548 (0.000) \\
letter & 0.093 (0.000) & 0.095 (0.000) & \textbf{0.119 (0.009)} & 0.090 (0.001) & 0.095 (0.001) & 0.091 (0.000) & 0.099 (0.003) & 0.113 (0.000) & 0.087 (0.000) \\
Lymphography & 0.996 (0.003) & \textbf{1.000 (0.000)} & 0.900 (0.062) & 0.890 (0.041) & \textbf{1.000 (0.000)} & 0.999 (0.001) & \textbf{1.000 (0.000)} & 0.865 (0.021) & 0.986 (0.005) \\
magic.gamma & \textbf{0.904 (0.001)} & 0.875 (0.002) & 0.893 (0.002) & 0.895 (0.002) & 0.871 (0.002) & 0.774 (0.000) & 0.841 (0.002) & 0.864 (0.000) & 0.859 (0.000) \\
mammography & \textbf{0.496 (0.010)} & 0.349 (0.019) & 0.360 (0.016) & 0.399 (0.024) & 0.341 (0.017) & 0.177 (0.000) & 0.233 (0.005) & 0.341 (0.000) & 0.413 (0.000) \\
mnist & \textbf{0.823 (0.003)} & 0.771 (0.008) & 0.654 (0.004) & 0.587 (0.005) & 0.749 (0.006) & 0.678 (0.000) & 0.555 (0.006) & 0.710 (0.000) & 0.727 (0.000) \\
musk & \textbf{1.000 (0.000)} & \textbf{1.000 (0.000)} & \textbf{1.000 (0.000)} & \textbf{1.000 (0.000)} & \textbf{1.000 (0.000)} & \textbf{1.000 (0.000)} & \textbf{1.000 (0.000)} & \textbf{1.000 (0.000)} & \textbf{1.000 (0.000)} \\
optdigits & \textbf{0.648 (0.018)} & 0.475 (0.019) & 0.537 (0.081) & 0.179 (0.012) & 0.397 (0.017) & 0.334 (0.000) & 0.286 (0.021) & 0.436 (0.000) & 0.300 (0.000) \\
PageBlocks & 0.618 (0.007) & 0.551 (0.002) & 0.594 (0.011) & 0.675 (0.006) & 0.678 (0.006) & 0.640 (0.000) & 0.642 (0.017) & \textbf{0.711 (0.000)} & 0.676 (0.000) \\
pendigits & 0.966 (0.008) & 0.924 (0.014) & 0.848 (0.028) & 0.518 (0.032) & 0.943 (0.003) & 0.353 (0.000) & 0.724 (0.057) & 0.786 (0.000) & \textbf{0.970 (0.000)} \\
Pima & \textbf{0.795 (0.008)} & 0.759 (0.010) & 0.668 (0.014) & 0.675 (0.011) & 0.768 (0.014) & 0.624 (0.023) & 0.765 (0.015) & 0.674 (0.016) & 0.757 (0.015) \\
satellite & 0.849 (0.002) & 0.844 (0.002) & 0.813 (0.008) & 0.846 (0.002) & 0.858 (0.002) & \textbf{0.887 (0.000)} & 0.886 (0.003) & 0.859 (0.000) & 0.860 (0.000) \\
satimage-2 & 0.931 (0.002) & 0.908 (0.004) & 0.162 (0.041) & 0.716 (0.024) & 0.917 (0.007) & 0.956 (0.000) & 0.966 (0.004) & 0.885 (0.000) & \textbf{0.967 (0.000)} \\
shuttle & 0.989 (0.002) & 0.993 (0.002) & \textbf{1.000 (0.000)} & 0.942 (0.000) & 0.995 (0.001) & 0.981 (0.000) & 0.999 (0.000) & 0.997 (0.000) & 0.979 (0.000) \\
skin & 0.854 (0.013) & 0.697 (0.016) & 0.959 (0.007) & 0.693 (0.003) & 0.792 (0.009) & 0.789 (0.037) & 0.657 (0.047) & 0.618 (0.009) & \textbf{0.983 (0.002)} \\
smtp & \textbf{0.622 (0.048)} & 0.360 (0.032) & 0.255 (0.083) & 0.499 (0.028) & 0.525 (0.042) & 0.500 (0.028) & 0.440 (0.076) & 0.492 (0.023) & 0.503 (0.029) \\
SpamBase & 0.836 (0.001) & 0.848 (0.001) & \textbf{0.860 (0.003)} & 0.838 (0.002) & 0.851 (0.001) & 0.856 (0.000) & 0.856 (0.004) & 0.727 (0.000) & 0.831 (0.000) \\
speech & 0.030 (0.001) & 0.032 (0.002) & \textbf{0.041 (0.003)} & 0.029 (0.001) & 0.030 (0.001) & 0.033 (0.000) & 0.037 (0.003) & 0.032 (0.000) & 0.028 (0.000) \\
Stamps & \textbf{0.889 (0.031)} & 0.854 (0.010) & 0.827 (0.036) & 0.598 (0.040) & 0.828 (0.036) & 0.496 (0.047) & 0.819 (0.016) & 0.665 (0.030) & 0.733 (0.039) \\
thyroid & 0.727 (0.015) & 0.484 (0.030) & 0.126 (0.014) & \textbf{0.832 (0.004)} & 0.756 (0.034) & 0.740 (0.000) & 0.535 (0.021) & 0.606 (0.000) & 0.809 (0.000) \\
vertebral & 0.718 (0.032) & 0.520 (0.013) & \textbf{0.763 (0.013)} & 0.344 (0.023) & 0.263 (0.016) & 0.198 (0.020) & 0.480 (0.043) & 0.327 (0.018) & 0.251 (0.016) \\
vowels & 0.418 (0.008) & \textbf{0.436 (0.023)} & 0.337 (0.049) & 0.383 (0.026) & 0.349 (0.011) & 0.387 (0.000) & 0.281 (0.058) & 0.331 (0.000) & 0.302 (0.000) \\
Waveform & 0.080 (0.002) & 0.109 (0.003) & 0.093 (0.007) & 0.103 (0.005) & 0.292 (0.015) & 0.053 (0.000) & 0.255 (0.017) & \textbf{0.307 (0.000)} & 0.270 (0.000) \\
WBC & 0.986 (0.008) & 0.941 (0.033) & 0.803 (0.044) & 0.319 (0.014) & 0.970 (0.017) & 0.978 (0.009) & \textbf{0.992 (0.007)} & 0.251 (0.017) & 0.895 (0.028) \\
WDBC & \textbf{0.980 (0.018)} & 0.953 (0.023) & 0.662 (0.088) & 0.747 (0.060) & 0.925 (0.043) & 0.882 (0.030) & 0.881 (0.038) & 0.955 (0.016) & 0.853 (0.027) \\
Wilt & 0.090 (0.001) & 0.140 (0.003) & \textbf{0.446 (0.028)} & 0.256 (0.006) & 0.104 (0.001) & 0.121 (0.000) & 0.344 (0.040) & 0.157 (0.000) & 0.122 (0.000) \\
wine & 0.996 (0.004) & 0.994 (0.006) & 0.998 (0.002) & 0.995 (0.002) & 0.999 (0.001) & \textbf{1.000 (0.000)} & 0.980 (0.019) & 0.910 (0.009) & 0.952 (0.008) \\
WPBC & 0.880 (0.024) & 0.891 (0.026) & \textbf{0.899 (0.021)} & 0.603 (0.020) & 0.805 (0.020) & 0.871 (0.029) & 0.846 (0.025) & 0.420 (0.010) & 0.471 (0.013) \\
yeast & 0.488 (0.003) & 0.508 (0.003) & \textbf{0.546 (0.007)} & 0.499 (0.004) & 0.487 (0.003) & 0.507 (0.000) & 0.486 (0.003) & 0.489 (0.000) & 0.483 (0.000) \\
20news & \textbf{0.217 (0.007)} & 0.207 (0.006) & 0.208 (0.011) & 0.163 (0.006) & 0.192 (0.008) & 0.137 (0.002) & 0.136 (0.001) & 0.150 (0.003) & 0.138 (0.005) \\
agnews & 0.278 (0.003) & 0.233 (0.001) & \textbf{0.279 (0.010)} & 0.173 (0.005) & 0.228 (0.002) & 0.122 (0.000) & 0.130 (0.001) & 0.259 (0.000) & 0.167 (0.000) \\
amazon & 0.119 (0.001) & 0.118 (0.000) & 0.107 (0.003) & 0.105 (0.002) & \textbf{0.121 (0.000)} & 0.098 (0.000) & 0.100 (0.000) & 0.110 (0.000) & 0.117 (0.000) \\
CIFAR10 & \textbf{0.232 (0.000)} & 0.195 (0.001) & 0.144 (0.001) & 0.198 (0.002) & 0.208 (0.000) & 0.200 (0.000) & 0.191 (0.002) & 0.222 (0.000) & 0.196 (0.000) \\
FashionMNIST & \textbf{0.702 (0.000)} & 0.437 (0.009) & 0.448 (0.003) & 0.557 (0.001) & 0.619 (0.000) & 0.594 (0.000) & 0.647 (0.001) & 0.636 (0.000) & 0.591 (0.000) \\
imdb & 0.088 (0.000) & 0.089 (0.000) & 0.097 (0.002) & 0.087 (0.003) & 0.089 (0.000) & 0.098 (0.000) & \textbf{0.101 (0.000)} & 0.090 (0.000) & 0.089 (0.000) \\
MNIST-C & \textbf{0.656 (0.001)} & 0.527 (0.009) & 0.471 (0.004) & 0.472 (0.001) & 0.507 (0.000) & 0.469 (0.000) & 0.501 (0.001) & 0.519 (0.000) & 0.462 (0.000) \\
MVTec & 0.875 (0.004) & 0.872 (0.005) & 0.884 (0.005) & 0.838 (0.006) & 0.882 (0.005) & 0.879 (0.004) & \textbf{0.889 (0.004)} & 0.754 (0.005) & 0.754 (0.004) \\
SVHN & \textbf{0.185 (0.000)} & 0.149 (0.000) & 0.144 (0.001) & 0.156 (0.001) & 0.160 (0.000) & 0.154 (0.000) & 0.157 (0.001) & 0.160 (0.000) & 0.153 (0.000) \\
yelp & \textbf{0.173 (0.002)} & 0.162 (0.000) & 0.139 (0.003) & 0.123 (0.006) & 0.162 (0.002) & 0.100 (0.000) & 0.098 (0.000) & 0.161 (0.000) & 0.160 (0.000) \\
\midrule
Overall & \textbf{0.627 (0.003)} & 0.584 (0.001) & 0.550 (0.004) & 0.526 (0.003) & 0.594 (0.001) & 0.516 (0.002) & 0.572 (0.004) & 0.524 (0.001) & 0.565 (0.001) \\
\bottomrule
\end{tabular}
}
\end{table}

\begin{table}[H]
    \centering
    \caption{Average AUC-ROC and standard deviations over five seeds for the semi-supervised setting on ADBench.}
    \resizebox{\textwidth}{!}{
\begin{tabular}{lccccccccc}
\toprule
 & K-DSM & DSM & MSM & DTE & DDAE & SLAD & ICL & LOF & KNN \\
\midrule
ALOI & \textbf{0.541 (0.000)} & 0.494 (0.001) & 0.488 (0.001) & 0.506 (0.002) & 0.518 (0.000) & 0.512 (0.000) & 0.474 (0.004) & 0.488 (0.000) & 0.510 (0.000) \\
annthyroid & 0.816 (0.012) & 0.748 (0.026) & 0.824 (0.004) & \textbf{0.975 (0.000)} & 0.931 (0.004) & 0.934 (0.000) & 0.843 (0.013) & 0.886 (0.000) & 0.928 (0.000) \\
backdoor & 0.949 (0.002) & \textbf{0.972 (0.002)} & 0.967 (0.001) & 0.902 (0.006) & 0.955 (0.002) & 0.500 (0.000) & 0.921 (0.003) & 0.954 (0.001) & 0.907 (0.002) \\
breastw & 0.992 (0.001) & 0.990 (0.001) & 0.900 (0.018) & 0.938 (0.007) & \textbf{0.995 (0.001)} & 0.995 (0.000) & 0.989 (0.001) & 0.916 (0.016) & 0.991 (0.001) \\
campaign & 0.706 (0.006) & 0.703 (0.009) & 0.748 (0.006) & 0.791 (0.002) & 0.810 (0.004) & 0.762 (0.000) & 0.770 (0.003) & 0.705 (0.000) & \textbf{0.848 (0.000)} \\
cardio & 0.874 (0.006) & 0.895 (0.004) & 0.623 (0.006) & 0.878 (0.003) & 0.871 (0.006) & 0.819 (0.000) & 0.820 (0.015) & \textbf{0.922 (0.000)} & 0.920 (0.000) \\
Cardiotocography & 0.595 (0.004) & 0.611 (0.009) & 0.448 (0.007) & 0.620 (0.015) & 0.562 (0.010) & 0.477 (0.000) & 0.523 (0.011) & \textbf{0.645 (0.000)} & 0.619 (0.000) \\
celeba & 0.617 (0.098) & 0.743 (0.018) & 0.714 (0.004) & \textbf{0.823 (0.007)} & 0.528 (0.014) & 0.675 (0.006) & 0.761 (0.005) & 0.437 (0.003) & 0.570 (0.007) \\
census & 0.700 (0.003) & 0.686 (0.003) & 0.663 (0.007) & 0.701 (0.003) & 0.716 (0.005) & 0.537 (0.037) & 0.712 (0.003) & 0.585 (0.005) & \textbf{0.759 (0.002)} \\
cover & 0.992 (0.002) & \textbf{0.997 (0.000)} & 0.985 (0.002) & 0.982 (0.001) & 0.972 (0.004) & 0.737 (0.048) & 0.937 (0.021) & 0.992 (0.000) & 0.975 (0.001) \\
donors & 1.000 (0.000) & 0.994 (0.002) & \textbf{1.000 (0.000)} & 0.983 (0.002) & 0.996 (0.001) & 0.871 (0.031) & 0.983 (0.003) & 0.970 (0.001) & 0.988 (0.001) \\
fault & 0.569 (0.002) & 0.576 (0.003) & 0.553 (0.002) & 0.588 (0.005) & 0.579 (0.003) & \textbf{0.640 (0.000)} & 0.619 (0.006) & 0.474 (0.000) & 0.587 (0.000) \\
fraud & 0.945 (0.007) & 0.942 (0.006) & 0.949 (0.007) & 0.944 (0.007) & 0.955 (0.007) & 0.944 (0.006) & 0.950 (0.003) & 0.945 (0.007) & \textbf{0.956 (0.005)} \\
glass & 0.969 (0.013) & 0.938 (0.017) & 0.991 (0.003) & 0.918 (0.008) & 0.964 (0.003) & 0.859 (0.027) & \textbf{0.992 (0.002)} & 0.898 (0.010) & 0.921 (0.005) \\
Hepatitis & \textbf{1.000 (0.000)} & 0.999 (0.001) & 0.999 (0.001) & 0.993 (0.005) & 0.999 (0.001) & 0.999 (0.001) & 0.999 (0.001) & 0.693 (0.035) & 0.968 (0.007) \\
http & 1.000 (0.000) & 0.999 (0.000) & 1.000 (0.000) & 0.996 (0.001) & 1.000 (0.000) & 0.999 (0.000) & 1.000 (0.000) & 1.000 (0.000) & \textbf{1.000 (0.000)} \\
InternetAds & \textbf{0.826 (0.001)} & 0.692 (0.011) & 0.801 (0.004) & 0.771 (0.010) & 0.816 (0.002) & 0.760 (0.000) & 0.785 (0.001) & 0.717 (0.000) & 0.737 (0.000) \\
Ionosphere & \textbf{0.989 (0.001)} & 0.983 (0.002) & 0.977 (0.004) & 0.946 (0.013) & 0.981 (0.002) & 0.982 (0.004) & 0.989 (0.002) & 0.943 (0.010) & 0.973 (0.004) \\
landsat & 0.605 (0.002) & 0.610 (0.004) & 0.617 (0.005) & 0.521 (0.008) & 0.667 (0.002) & 0.647 (0.000) & 0.643 (0.001) & 0.666 (0.000) & \textbf{0.683 (0.000)} \\
letter & 0.413 (0.002) & 0.412 (0.001) & \textbf{0.481 (0.015)} & 0.371 (0.007) & 0.406 (0.001) & 0.382 (0.000) & 0.422 (0.013) & 0.448 (0.000) & 0.354 (0.000) \\
Lymphography & 1.000 (0.000) & \textbf{1.000 (0.000)} & 0.993 (0.004) & 0.992 (0.001) & \textbf{1.000 (0.000)} & 1.000 (0.000) & \textbf{1.000 (0.000)} & 0.985 (0.004) & 0.999 (0.000) \\
magic.gamma & \textbf{0.879 (0.002)} & 0.849 (0.003) & 0.870 (0.002) & 0.873 (0.003) & 0.839 (0.003) & 0.720 (0.000) & 0.793 (0.003) & 0.834 (0.000) & 0.833 (0.000) \\
mammography & \textbf{0.883 (0.002)} & 0.859 (0.008) & 0.857 (0.004) & 0.865 (0.007) & 0.844 (0.006) & 0.736 (0.000) & 0.779 (0.005) & 0.855 (0.000) & 0.876 (0.000) \\
mnist & \textbf{0.956 (0.002)} & 0.936 (0.004) & 0.880 (0.004) & 0.892 (0.002) & 0.943 (0.002) & 0.896 (0.000) & 0.800 (0.005) & 0.929 (0.000) & 0.938 (0.000) \\
musk & \textbf{1.000 (0.000)} & \textbf{1.000 (0.000)} & \textbf{1.000 (0.000)} & \textbf{1.000 (0.000)} & \textbf{1.000 (0.000)} & \textbf{1.000 (0.000)} & \textbf{1.000 (0.000)} & \textbf{1.000 (0.000)} & \textbf{1.000 (0.000)} \\
optdigits & \textbf{0.979 (0.002)} & 0.965 (0.003) & 0.947 (0.014) & 0.857 (0.011) & 0.960 (0.003) & 0.946 (0.000) & 0.913 (0.012) & 0.967 (0.000) & 0.940 (0.000) \\
PageBlocks & 0.831 (0.004) & 0.832 (0.003) & 0.885 (0.001) & 0.901 (0.004) & 0.900 (0.002) & 0.869 (0.000) & 0.890 (0.005) & \textbf{0.913 (0.000)} & 0.897 (0.000) \\
pendigits & \textbf{0.999 (0.000)} & 0.997 (0.001) & 0.994 (0.001) & 0.981 (0.003) & 0.998 (0.000) & 0.946 (0.000) & 0.980 (0.006) & 0.991 (0.000) & 0.999 (0.000) \\
Pima & \textbf{0.816 (0.003)} & 0.753 (0.005) & 0.693 (0.014) & 0.699 (0.009) & 0.772 (0.011) & 0.602 (0.020) & 0.774 (0.020) & 0.701 (0.010) & 0.774 (0.010) \\
satellite & 0.798 (0.002) & 0.799 (0.003) & 0.784 (0.009) & 0.786 (0.004) & 0.818 (0.003) & \textbf{0.875 (0.000)} & 0.860 (0.006) & 0.803 (0.000) & 0.822 (0.000) \\
satimage-2 & 0.993 (0.001) & 0.990 (0.001) & 0.883 (0.021) & 0.994 (0.000) & 0.994 (0.001) & \textbf{0.998 (0.000)} & 0.997 (0.001) & 0.994 (0.000) & 0.997 (0.000) \\
shuttle & 1.000 (0.000) & 0.999 (0.000) & \textbf{1.000 (0.000)} & 0.998 (0.000) & 1.000 (0.000) & 0.999 (0.000) & 1.000 (0.000) & 1.000 (0.000) & 0.999 (0.000) \\
skin & 0.954 (0.005) & 0.889 (0.011) & 0.984 (0.003) & 0.918 (0.001) & 0.946 (0.002) & 0.911 (0.012) & 0.863 (0.029) & 0.865 (0.008) & \textbf{0.995 (0.000)} \\
smtp & 0.848 (0.018) & 0.923 (0.009) & 0.943 (0.010) & \textbf{0.954 (0.007)} & 0.917 (0.013) & 0.919 (0.008) & 0.919 (0.007) & 0.929 (0.015) & 0.929 (0.012) \\
SpamBase & 0.833 (0.001) & 0.841 (0.001) & \textbf{0.853 (0.002)} & 0.831 (0.002) & 0.848 (0.001) & 0.848 (0.000) & 0.838 (0.002) & 0.732 (0.000) & 0.832 (0.000) \\
speech & 0.442 (0.006) & 0.435 (0.004) & \textbf{0.560 (0.020)} & 0.381 (0.010) & 0.388 (0.002) & 0.440 (0.000) & 0.512 (0.014) & 0.375 (0.000) & 0.364 (0.000) \\
Stamps & \textbf{0.985 (0.004)} & 0.971 (0.006) & 0.969 (0.007) & 0.922 (0.008) & 0.977 (0.004) & 0.818 (0.029) & 0.974 (0.002) & 0.942 (0.007) & 0.962 (0.006) \\
thyroid & 0.963 (0.003) & 0.884 (0.015) & 0.746 (0.037) & \textbf{0.989 (0.000)} & 0.976 (0.005) & 0.952 (0.000) & 0.955 (0.008) & 0.927 (0.000) & 0.987 (0.000) \\
vertebral & 0.893 (0.016) & 0.776 (0.005) & \textbf{0.916 (0.004)} & 0.677 (0.012) & 0.550 (0.019) & 0.449 (0.041) & 0.743 (0.029) & 0.638 (0.008) & 0.573 (0.019) \\
vowels & 0.855 (0.003) & \textbf{0.871 (0.008)} & 0.864 (0.008) & 0.869 (0.009) & 0.854 (0.002) & 0.855 (0.000) & 0.829 (0.022) & 0.863 (0.000) & 0.822 (0.000) \\
Waveform & 0.596 (0.007) & 0.693 (0.006) & 0.628 (0.016) & 0.650 (0.005) & \textbf{0.781 (0.007)} & 0.489 (0.000) & 0.725 (0.006) & 0.760 (0.000) & 0.752 (0.000) \\
WBC & 0.998 (0.001) & 0.994 (0.003) & 0.983 (0.004) & 0.844 (0.027) & 0.997 (0.002) & 0.997 (0.001) & \textbf{0.999 (0.001)} & 0.793 (0.019) & 0.989 (0.002) \\
WDBC & \textbf{0.999 (0.001)} & 0.998 (0.001) & 0.977 (0.007) & 0.989 (0.003) & 0.997 (0.001) & 0.994 (0.001) & 0.993 (0.002) & 0.998 (0.001) & 0.992 (0.001) \\
Wilt & 0.483 (0.003) & 0.650 (0.007) & \textbf{0.919 (0.006)} & 0.849 (0.004) & 0.557 (0.004) & 0.616 (0.000) & 0.846 (0.019) & 0.688 (0.000) & 0.637 (0.000) \\
wine & 0.999 (0.001) & 0.999 (0.001) & 1.000 (0.000) & 0.999 (0.000) & 1.000 (0.000) & \textbf{1.000 (0.000)} & 0.997 (0.002) & 0.985 (0.002) & 0.992 (0.001) \\
WPBC & 0.961 (0.009) & \textbf{0.966 (0.008)} & 0.964 (0.008) & 0.699 (0.020) & 0.910 (0.008) & 0.955 (0.010) & 0.943 (0.011) & 0.581 (0.009) & 0.643 (0.012) \\
yeast & 0.461 (0.005) & 0.476 (0.007) & \textbf{0.526 (0.009)} & 0.471 (0.007) & 0.449 (0.004) & 0.488 (0.000) & 0.463 (0.008) & 0.458 (0.000) & 0.447 (0.000) \\
20news & 0.688 (0.003) & \textbf{0.688 (0.002)} & 0.682 (0.007) & 0.640 (0.007) & 0.669 (0.004) & 0.595 (0.003) & 0.596 (0.004) & 0.603 (0.004) & 0.575 (0.004) \\
agnews & \textbf{0.763 (0.004)} & 0.760 (0.001) & 0.712 (0.010) & 0.676 (0.006) & 0.758 (0.002) & 0.585 (0.000) & 0.577 (0.001) & 0.746 (0.000) & 0.670 (0.000) \\
amazon & 0.604 (0.006) & 0.607 (0.001) & 0.536 (0.012) & 0.540 (0.009) & \textbf{0.618 (0.001)} & 0.521 (0.000) & 0.534 (0.001) & 0.579 (0.000) & 0.606 (0.000) \\
CIFAR10 & \textbf{0.721 (0.000)} & 0.683 (0.001) & 0.609 (0.002) & 0.689 (0.003) & 0.692 (0.000) & 0.666 (0.000) & 0.658 (0.001) & 0.703 (0.000) & 0.675 (0.000) \\
FashionMNIST & \textbf{0.937 (0.001)} & 0.711 (0.006) & 0.864 (0.001) & 0.904 (0.001) & 0.913 (0.000) & 0.900 (0.000) & 0.908 (0.001) & 0.916 (0.000) & 0.899 (0.000) \\
imdb & 0.492 (0.001) & 0.494 (0.001) & 0.508 (0.007) & 0.473 (0.015) & 0.497 (0.003) & 0.513 (0.000) & \textbf{0.522 (0.001)} & 0.496 (0.000) & 0.501 (0.000) \\
MNIST-C & \textbf{0.914 (0.000)} & 0.837 (0.009) & 0.868 (0.001) & 0.856 (0.000) & 0.867 (0.000) & 0.836 (0.000) & 0.850 (0.001) & 0.872 (0.000) & 0.841 (0.000) \\
MVTec & 0.913 (0.003) & 0.922 (0.002) & 0.941 (0.001) & 0.881 (0.004) & 0.938 (0.002) & 0.932 (0.002) & \textbf{0.943 (0.002)} & 0.802 (0.003) & 0.813 (0.002) \\
SVHN & \textbf{0.685 (0.000)} & 0.609 (0.001) & 0.626 (0.001) & 0.630 (0.001) & 0.630 (0.000) & 0.609 (0.000) & 0.617 (0.001) & 0.638 (0.000) & 0.617 (0.000) \\
yelp & \textbf{0.685 (0.003)} & 0.678 (0.001) & 0.629 (0.004) & 0.580 (0.019) & 0.682 (0.002) & 0.548 (0.000) & 0.541 (0.001) & 0.672 (0.000) & 0.681 (0.000) \\
\midrule
Overall & \textbf{0.823 (0.002)} & 0.813 (0.001) & 0.813 (0.001) & 0.805 (0.001) & 0.819 (0.001) & 0.773 (0.001) & 0.812 (0.001) & 0.786 (0.001) & 0.808 (0.000) \\
\bottomrule
\end{tabular}
}
\end{table}

\begin{table}[H]
    \centering
    \caption{Average AUC-F1 and standard deviations over five seeds for the semi-supervised setting on ADBench.}
    \resizebox{\textwidth}{!}{
\begin{tabular}{lccccccccc}
\toprule
 & K-DSM & DSM & MSM & DTE & DDAE & SLAD & ICL & LOF & KNN \\
\midrule
ALOI & 0.071 (0.002) & 0.065 (0.001) & 0.061 (0.001) & 0.043 (0.000) & 0.064 (0.002) & 0.051 (0.000) & 0.051 (0.002) & \textbf{0.082 (0.000)} & 0.059 (0.000) \\
annthyroid & 0.529 (0.020) & 0.436 (0.025) & 0.424 (0.010) & \textbf{0.773 (0.002)} & 0.628 (0.015) & 0.661 (0.000) & 0.560 (0.013) & 0.496 (0.000) & 0.620 (0.000) \\
backdoor & 0.859 (0.006) & \textbf{0.886 (0.004)} & 0.884 (0.006) & 0.788 (0.020) & 0.865 (0.005) & 0.000 (0.000) & 0.861 (0.006) & 0.721 (0.011) & 0.322 (0.006) \\
breastw & 0.962 (0.005) & 0.959 (0.004) & 0.855 (0.021) & 0.884 (0.010) & 0.969 (0.002) & 0.969 (0.003) & \textbf{0.972 (0.003)} & 0.878 (0.012) & 0.960 (0.002) \\
campaign & 0.362 (0.008) & 0.397 (0.006) & 0.478 (0.005) & 0.522 (0.004) & 0.529 (0.005) & 0.497 (0.000) & 0.487 (0.003) & 0.422 (0.000) & \textbf{0.578 (0.000)} \\
cardio & 0.645 (0.004) & \textbf{0.650 (0.004)} & 0.226 (0.013) & 0.580 (0.002) & 0.620 (0.001) & 0.614 (0.000) & 0.581 (0.006) & 0.625 (0.000) & 0.619 (0.000) \\
Cardiotocography & 0.461 (0.006) & 0.453 (0.011) & 0.304 (0.009) & 0.407 (0.015) & 0.368 (0.009) & 0.352 (0.000) & 0.378 (0.012) & \textbf{0.483 (0.000)} & 0.468 (0.000) \\
celeba & 0.103 (0.038) & \textbf{0.183 (0.012)} & 0.091 (0.007) & 0.170 (0.006) & 0.049 (0.007) & 0.141 (0.005) & 0.130 (0.004) & 0.019 (0.002) & 0.080 (0.006) \\
census & 0.219 (0.006) & 0.205 (0.004) & 0.206 (0.017) & 0.196 (0.004) & 0.213 (0.009) & 0.042 (0.042) & 0.249 (0.007) & 0.127 (0.004) & \textbf{0.313 (0.003)} \\
cover & 0.863 (0.016) & \textbf{0.874 (0.005)} & 0.826 (0.011) & 0.734 (0.014) & 0.728 (0.015) & 0.075 (0.038) & 0.476 (0.059) & 0.821 (0.010) & 0.650 (0.010) \\
donors & \textbf{0.984 (0.004)} & 0.904 (0.021) & 0.980 (0.005) & 0.852 (0.015) & 0.920 (0.008) & 0.539 (0.063) & 0.808 (0.027) & 0.748 (0.008) & 0.876 (0.006) \\
fault & 0.559 (0.002) & 0.554 (0.001) & 0.547 (0.001) & 0.551 (0.005) & 0.555 (0.005) & \textbf{0.599 (0.000)} & 0.586 (0.007) & 0.507 (0.000) & 0.556 (0.000) \\
fraud & 0.659 (0.056) & 0.595 (0.051) & 0.561 (0.052) & \textbf{0.684 (0.035)} & 0.572 (0.019) & 0.490 (0.018) & 0.614 (0.046) & 0.615 (0.021) & 0.478 (0.015) \\
glass & 0.758 (0.055) & 0.675 (0.030) & 0.838 (0.048) & 0.350 (0.027) & 0.554 (0.036) & 0.334 (0.027) & \textbf{0.860 (0.016)} & 0.250 (0.037) & 0.299 (0.048) \\
Hepatitis & \textbf{1.000 (0.000)} & 0.996 (0.004) & 0.996 (0.004) & 0.957 (0.018) & 0.995 (0.003) & 0.996 (0.004) & 0.996 (0.004) & 0.400 (0.047) & 0.820 (0.023) \\
http & 0.999 (0.001) & 0.925 (0.005) & 0.974 (0.019) & 0.392 (0.082) & 0.980 (0.017) & 0.886 (0.041) & 0.968 (0.008) & 0.966 (0.013) & \textbf{1.000 (0.000)} \\
InternetAds & 0.671 (0.003) & 0.549 (0.003) & 0.661 (0.008) & 0.646 (0.017) & \textbf{0.695 (0.003)} & 0.582 (0.000) & 0.624 (0.002) & 0.546 (0.000) & 0.598 (0.000) \\
Ionosphere & 0.936 (0.006) & 0.929 (0.006) & 0.921 (0.005) & 0.899 (0.007) & 0.924 (0.007) & 0.923 (0.005) & \textbf{0.936 (0.007)} & 0.873 (0.011) & 0.902 (0.009) \\
landsat & 0.486 (0.003) & 0.492 (0.007) & 0.450 (0.008) & 0.377 (0.012) & 0.519 (0.004) & 0.467 (0.000) & \textbf{0.542 (0.003)} & 0.536 (0.000) & 0.515 (0.000) \\
letter & 0.018 (0.004) & 0.042 (0.002) & 0.056 (0.006) & 0.026 (0.007) & 0.036 (0.002) & 0.020 (0.000) & 0.046 (0.007) & \textbf{0.100 (0.000)} & 0.010 (0.000) \\
Lymphography & 0.969 (0.021) & \textbf{1.000 (0.000)} & 0.893 (0.070) & 0.798 (0.008) & \textbf{1.000 (0.000)} & 0.995 (0.005) & \textbf{1.000 (0.000)} & 0.767 (0.030) & 0.916 (0.030) \\
magic.gamma & \textbf{0.810 (0.002)} & 0.780 (0.003) & 0.802 (0.001) & 0.807 (0.003) & 0.764 (0.002) & 0.660 (0.000) & 0.716 (0.003) & 0.761 (0.000) & 0.762 (0.000) \\
mammography & \textbf{0.499 (0.007)} & 0.359 (0.016) & 0.405 (0.006) & 0.372 (0.024) & 0.351 (0.007) & 0.212 (0.000) & 0.252 (0.009) & 0.385 (0.000) & 0.404 (0.000) \\
mnist & \textbf{0.785 (0.003)} & 0.720 (0.011) & 0.614 (0.002) & 0.600 (0.007) & 0.725 (0.009) & 0.664 (0.000) & 0.510 (0.005) & 0.714 (0.000) & 0.719 (0.000) \\
musk & \textbf{1.000 (0.000)} & \textbf{1.000 (0.000)} & \textbf{1.000 (0.000)} & \textbf{1.000 (0.000)} & \textbf{1.000 (0.000)} & \textbf{1.000 (0.000)} & \textbf{1.000 (0.000)} & \textbf{1.000 (0.000)} & \textbf{1.000 (0.000)} \\
optdigits & \textbf{0.652 (0.017)} & 0.551 (0.020) & 0.573 (0.058) & 0.135 (0.014) & 0.441 (0.027) & 0.360 (0.000) & 0.339 (0.023) & 0.533 (0.000) & 0.240 (0.000) \\
PageBlocks & 0.585 (0.013) & 0.513 (0.010) & 0.602 (0.008) & 0.630 (0.005) & 0.610 (0.002) & 0.586 (0.000) & 0.639 (0.012) & \textbf{0.659 (0.000)} & 0.590 (0.000) \\
pendigits & \textbf{0.936 (0.006)} & 0.850 (0.020) & 0.829 (0.018) & 0.613 (0.043) & 0.910 (0.008) & 0.442 (0.000) & 0.692 (0.053) & 0.763 (0.000) & 0.904 (0.000) \\
Pima & \textbf{0.744 (0.008)} & 0.707 (0.009) & 0.670 (0.012) & 0.656 (0.012) & 0.713 (0.014) & 0.582 (0.013) & 0.714 (0.013) & 0.667 (0.014) & 0.710 (0.013) \\
satellite & 0.722 (0.001) & 0.726 (0.002) & 0.693 (0.009) & 0.722 (0.001) & 0.732 (0.002) & \textbf{0.785 (0.000)} & 0.763 (0.007) & 0.726 (0.000) & 0.718 (0.000) \\
satimage-2 & 0.907 (0.003) & 0.870 (0.009) & 0.189 (0.034) & 0.699 (0.021) & 0.831 (0.015) & 0.887 (0.000) & \textbf{0.921 (0.008)} & 0.817 (0.000) & 0.901 (0.000) \\
shuttle & 0.986 (0.001) & 0.989 (0.000) & \textbf{0.989 (0.001)} & 0.980 (0.000) & 0.985 (0.000) & 0.985 (0.000) & 0.988 (0.000) & 0.984 (0.000) & 0.982 (0.000) \\
skin & 0.840 (0.012) & 0.762 (0.011) & 0.910 (0.010) & 0.821 (0.002) & 0.846 (0.005) & 0.744 (0.012) & 0.691 (0.042) & 0.707 (0.009) & \textbf{0.967 (0.002)} \\
smtp & 0.637 (0.056) & 0.512 (0.045) & 0.297 (0.104) & \textbf{0.696 (0.020)} & 0.672 (0.012) & \textbf{0.696 (0.020)} & 0.562 (0.065) & 0.684 (0.011) & \textbf{0.696 (0.020)} \\
SpamBase & 0.805 (0.002) & 0.806 (0.001) & 0.820 (0.002) & 0.800 (0.001) & \textbf{0.821 (0.001)} & 0.818 (0.000) & 0.810 (0.002) & 0.740 (0.000) & 0.805 (0.000) \\
speech & 0.033 (0.000) & 0.036 (0.006) & 0.039 (0.013) & 0.036 (0.012) & 0.033 (0.009) & \textbf{0.066 (0.000)} & 0.033 (0.009) & 0.033 (0.000) & 0.033 (0.000) \\
Stamps & \textbf{0.865 (0.023)} & 0.845 (0.022) & 0.784 (0.031) & 0.620 (0.038) & 0.823 (0.028) & 0.467 (0.042) & 0.840 (0.011) & 0.659 (0.044) & 0.786 (0.042) \\
thyroid & 0.667 (0.020) & 0.477 (0.027) & 0.168 (0.009) & \textbf{0.766 (0.006)} & 0.690 (0.032) & 0.720 (0.000) & 0.587 (0.017) & 0.527 (0.000) & 0.753 (0.000) \\
vertebral & 0.728 (0.029) & 0.512 (0.011) & \textbf{0.741 (0.020)} & 0.403 (0.051) & 0.273 (0.032) & 0.143 (0.049) & 0.531 (0.045) & 0.332 (0.025) & 0.230 (0.026) \\
vowels & 0.384 (0.010) & 0.392 (0.027) & 0.348 (0.048) & 0.388 (0.026) & 0.360 (0.013) & \textbf{0.400 (0.000)} & 0.292 (0.063) & 0.340 (0.000) & 0.260 (0.000) \\
Waveform & 0.100 (0.003) & 0.136 (0.004) & 0.128 (0.012) & 0.126 (0.009) & 0.326 (0.009) & 0.030 (0.000) & \textbf{0.352 (0.014)} & 0.280 (0.000) & 0.270 (0.000) \\
WBC & 0.948 (0.024) & 0.867 (0.038) & 0.723 (0.050) & 0.340 (0.023) & 0.904 (0.032) & 0.914 (0.032) & \textbf{0.973 (0.022)} & 0.212 (0.049) & 0.836 (0.023) \\
WDBC & \textbf{0.953 (0.030)} & 0.908 (0.031) & 0.619 (0.066) & 0.731 (0.052) & 0.899 (0.032) & 0.797 (0.047) & 0.831 (0.049) & 0.901 (0.033) & 0.809 (0.022) \\
Wilt & 0.022 (0.002) & 0.142 (0.008) & \textbf{0.541 (0.033)} & 0.175 (0.013) & 0.030 (0.002) & 0.070 (0.000) & 0.401 (0.047) & 0.167 (0.000) & 0.023 (0.000) \\
wine & 0.993 (0.007) & 0.982 (0.018) & 0.993 (0.007) & 0.979 (0.010) & 0.993 (0.007) & \textbf{1.000 (0.000)} & 0.971 (0.019) & 0.816 (0.021) & 0.859 (0.017) \\
WPBC & 0.887 (0.016) & \textbf{0.899 (0.015)} & 0.887 (0.019) & 0.577 (0.023) & 0.818 (0.011) & 0.887 (0.015) & 0.863 (0.023) & 0.417 (0.019) & 0.502 (0.014) \\
yeast & 0.499 (0.006) & 0.498 (0.007) & \textbf{0.523 (0.008)} & 0.497 (0.006) & 0.470 (0.003) & 0.497 (0.000) & 0.486 (0.008) & 0.477 (0.000) & 0.467 (0.000) \\
20news & 0.238 (0.006) & 0.231 (0.006) & \textbf{0.243 (0.014)} & 0.181 (0.010) & 0.229 (0.008) & 0.136 (0.006) & 0.155 (0.007) & 0.169 (0.007) & 0.151 (0.010) \\
agnews & 0.309 (0.004) & 0.288 (0.002) & \textbf{0.318 (0.010)} & 0.210 (0.008) & 0.288 (0.005) & 0.129 (0.000) & 0.144 (0.001) & 0.306 (0.000) & 0.201 (0.000) \\
amazon & 0.107 (0.002) & 0.106 (0.003) & 0.112 (0.007) & 0.104 (0.004) & 0.103 (0.001) & 0.102 (0.000) & 0.104 (0.002) & 0.100 (0.000) & \textbf{0.114 (0.000)} \\
CIFAR10 & \textbf{0.272 (0.001)} & 0.232 (0.002) & 0.166 (0.002) & 0.238 (0.002) & 0.245 (0.001) & 0.244 (0.000) & 0.225 (0.001) & 0.271 (0.000) & 0.229 (0.000) \\
FashionMNIST & \textbf{0.688 (0.001)} & 0.468 (0.007) & 0.471 (0.004) & 0.594 (0.002) & 0.620 (0.001) & 0.597 (0.000) & 0.641 (0.001) & 0.633 (0.000) & 0.590 (0.000) \\
imdb & 0.058 (0.001) & 0.063 (0.000) & 0.101 (0.004) & 0.072 (0.008) & 0.057 (0.001) & \textbf{0.102 (0.000)} & 0.100 (0.001) & 0.064 (0.000) & 0.054 (0.000) \\
MNIST-C & \textbf{0.667 (0.001)} & 0.532 (0.007) & 0.506 (0.003) & 0.492 (0.001) & 0.509 (0.001) & 0.478 (0.000) & 0.514 (0.001) & 0.530 (0.000) & 0.463 (0.000) \\
MVTec & 0.807 (0.004) & 0.816 (0.003) & \textbf{0.825 (0.004)} & 0.776 (0.006) & 0.822 (0.005) & 0.819 (0.004) & 0.821 (0.003) & 0.671 (0.003) & 0.671 (0.005) \\
SVHN & \textbf{0.225 (0.001)} & 0.183 (0.001) & 0.173 (0.002) & 0.195 (0.000) & 0.201 (0.000) & 0.192 (0.000) & 0.197 (0.001) & 0.192 (0.000) & 0.190 (0.000) \\
yelp & 0.199 (0.001) & 0.189 (0.002) & 0.169 (0.006) & 0.140 (0.009) & 0.195 (0.003) & 0.070 (0.000) & 0.078 (0.001) & \textbf{0.206 (0.000)} & 0.188 (0.000) \\
\midrule
Overall & \textbf{0.608 (0.003)} & 0.574 (0.003) & 0.547 (0.004) & 0.523 (0.003) & 0.580 (0.001) & 0.500 (0.002) & 0.570 (0.003) & 0.516 (0.002) & 0.539 (0.002) \\
\bottomrule
\end{tabular}
}
\end{table}



\end{document}